\documentclass[12pt,twoside]{mitthesis}
 
\usepackage{lgrind}
\usepackage{cmap}
\usepackage[T1]{fontenc}
\usepackage{csquotes}

\usepackage{amssymb}
\usepackage{amsmath}
\usepackage{amsthm}
\usepackage{graphicx}
\usepackage{esint}
\usepackage{enumitem}
\usepackage{algorithm}
\usepackage{hyperref}
\usepackage{marvosym}
\usepackage{mathtools}
\usepackage{xspace} 
\usepackage{color}
\usepackage{comment}
\usepackage{wrapfig}
\usepackage{tabularx}
\usepackage{float}
\usepackage[]{units}

\usepackage{makeidx}

\usepackage{caption}
\numberwithin{equation}{section}

\newtheorem{problem}{Problem}
\newtheorem{theorem}{Theorem}
\newtheorem{lemma}{Lemma}
\newtheorem{cor}{Corollary}

\newtheorem{remark}{Remark}

\numberwithin{lemma}{section}
\numberwithin{theorem}{section}
\numberwithin{problem}{section}
\numberwithin{example}{section}
\numberwithin{cor}{section}

\usepackage[numbers,sectionbib]{natbib}
\usepackage[sectionbib]{chapterbib}

\usepackage[noend]{algpseudocode}
\usepackage[section]{placeins}
\usepackage{hyperref}

\algrenewcommand\algorithmicindent{1.0em}

\newcommand{\GLC}{\ensuremath{\mathrm{GLC}}\xspace}
\newcommand{\PRM}{\ensuremath{\mathrm{PRM}}\xspace}
\newcommand{\PRMs}{\ensuremath{\mathrm{PRM}^*}\xspace}
\newcommand{\RRT}{\ensuremath{\mathrm{RRT}}\xspace}
\newcommand{\RRTs}{\ensuremath{\mathrm{RRT}^*}\xspace}
\newcommand{\SST}{\ensuremath{\mathrm{SST}}\xspace}
\newcommand{\EST}{\ensuremath{\mathrm{EST}}\xspace}
\newcommand{\NULL}{\ensuremath{\mathtt{NULL}}\xspace}

\newcommand{\Stretch}{\arraycolsep=1.4pt\def\arraystretch{1.5} }

\newcommand{\xfree}{\ensuremath{X_{\mathrm{free}}}\xspace}
\newcommand{\xgoal}{\ensuremath{X_{\mathrm{goal}}}\xspace}
\newcommand{\Xfree}{\ensuremath{\mathcal{X}_{\mathrm{free}}}\xspace}
\newcommand{\Xgoal}{\ensuremath{\mathcal{X}_{\mathrm{goal}}}\xspace}
\newcommand{\U}{\ensuremath{\mathcal{U}}\xspace}
\newcommand{\X}{\ensuremath{\mathcal{X}}\xspace}
\newcommand{\Ufree}{\ensuremath{\mathcal{U}_{\mathrm{free}}}\xspace}
\newcommand{\Ugoal}{\ensuremath{\mathcal{U}_{\mathrm{goal}}}\xspace}
\newcommand{\UR}{\ensuremath{\mathcal{U}_R}\xspace}
\newcommand{\Xic}{\ensuremath{\mathcal{X}_{x_{ic}}}\xspace}
\newcommand{\Xxo}{\ensuremath{\mathcal{X}_{x_{0}}}\xspace}

\usepackage{mathrsfs}

\makeindex
\begin{document}

%
%
%
%
%
%
%
%
%
%
%

\title{A Generalized Label Correcting Method for Optimal Kinodynamic Motion Planning}

\author{Brian A. Paden}
\department{Department of Mechanical Engineering}

\degree{Doctor of Philosophy}

\degreemonth{June}
\degreeyear{2017}
\thesisdate{June 5, 2017}


\supervisor{Emilio Frazzoli}{Professor of Aeronautics and Astronautics}
\committee{Domitilla Del Vecchio}{Associate Professor of Mechanical Engineering}{Chair, Thesis Committee}
\committee{Karl Iagnemma}{Principal Research Scientist}{Member, Thesis Committee}

\chairman{Rohan Abeyaratne}{Chair, Department Committee on Graduate Theses}

\maketitle



\cleardoublepage
\setcounter{savepage}{\thepage}
\begin{abstractpage}
%
%
%

%
Nearly all autonomous robotic systems use some form of motion planning to compute reference motions through their environment.
An increasing use of autonomous robots in a broad range of applications creates a need for efficient, general purpose motion planning algorithms that are applicable in any of these new application domains.
This thesis presents a resolution complete optimal kinodynamic motion planning algorithm based on a direct forward search of the set of admissible input signals to a dynamical model. 
The advantage of this generalized label correcting method is that it does not require a local planning subroutine as in the case of related methods.

Preliminary material focuses on new topological properties of the canonical problem formulation that are used to show continuity of the performance objective.
These observations are used to derive a generalization of Bellman's principle of optimality in the context of kinodynamic motion planning. 
A generalized label correcting algorithm is then proposed which leverages these results to prune candidate input signals from the search when their cost is greater than related signals. 

The second part of this thesis addresses admissible heuristics for kinodynamic motion planning.
An admissibility condition is derived that can be used to verify the admissibility of candidate heuristics for a particular problem. 
This condition also characterizes a convex set of admissible heuristics.
A linear program is formulated to obtain a heuristic which is as close to  the optimal cost-to-go as possible while remaining admissible.
This optimization is justified by showing its solution coincides with the solution to the Hamilton-Jacobi-Bellman equation.
Lastly, a sum-of-squares relaxation of this infinite-dimensional linear program is proposed for obtaining provably admissible approximate solutions.

\end{abstractpage}


\cleardoublepage

\section*{Acknowledgments}
I would like to express my gratitude to my advisor Professor Emilio Frazzoli for his input and support through my graduate studies.
Emilio has given me tremendous freedom in pursuing the research questions I found most interesting and provided skillful guidance to keep me moving in the right direction.
I have also enjoyed numerous opportunities that were a direct consequence of my affiliation with Emilio's research group. 
Among these were several exciting internships working on autonomous systems and the opportunity to spend a year doing research at ETH.
I am also grateful to my committee members, Professor Domitilla Del Vecchio and Karl Iagnemma for their valuable feedback and advice in out meetings over the past several years. 

Additionally, I would like to thank my colleagues and friends with whom I have worked over the past four years. 
Sze Zheng Yong was a great mentor in my first year at MIT and we were able to collaborate on a number of interesting projects.
The afternoon chalkboard sessions with Dmitry Yershov were stimulating and memorable. 
They contributed significantly to the ideas in the first part of this thesis. 
The second part of this thesis related to heuristics for kinodynamic planning problems came from a collaboration with Valerio Varricchio on a course project in Professor Pablo Parrilo's algebraic techniques class.
%

\pagestyle{plain}
\chapter*{}
\vspace{2.5in}
\begin{center}
\emph{For Stephanie}
\end{center}

\tableofcontents

\chapter{Introduction}
%
%
Robotics and automation have been reshaping the world economy with widespread use in industries such as manufacturing, agriculture, transportation, defense, and medicine.
%
%
Advances to the theoretical foundations of the subject together with a seemingly endless stream of new sensing and computing technology is making autonomous robots capable of increasingly complex tasks.
One of the emerging new applications is driverless vehicles~\cite{padenSurvey} and transportation systems which have the potential to eliminate urban congestion~\cite{spieser2014toward} and dramatically reduce the roughly 1.24 million lives lost each year in road traffic accidents~\cite{world2013global}. 
The sophistication and complexity of modern robotic systems requires entire fields of research devoted to individual subsystems.
These can be divided into two major categories. 
The first, sensing and perception which inform the system about its state and the surrounding environment; the second, planning and control which makes use of sensor data, processed by the perception system, and selects appropriate actions to accomplish task specifications. 

While both of these are subjects of active research, the scope of this thesis falls into the latter category of planning and control.
Specifically, it focuses on the planning of motions through an environment that satisfy a dynamical model of the system while optimizing a measure of performance.

\section{Background and Motivation}

Research on planning and control for robotic systems is addressed with a combination of techniques from theoretical computer science and control theory. 
High level decision making and task planning problems generally involve making selections from a discrete set of alternatives.
The techniques from theoretical computer science lend themselves well to these problems given their discrete nature. 
Central research questions involve: determining expressive problem formulations that can be utilized in a large variety of applications, determining the computational complexity of a particular problem formulation, and then designing sound and correct algorithms which realize this complexity.
At the other end of the spectrum, a feedback control system faithfully executes reference motions given a potentially complex dynamical model of the system.
Central research questions in control theory are the stability of the system in the presence of feedback control, and robustness to disturbances and modeling errors.  
While computer science provides the tools for high level planning and control theory provides the tools to execute a reference motion, the actual planning of a robot's reference motion falls in the middle ground between these two disciplines and has been treated almost independently by the two.
The main specifications for the planned motion are: (i) the motion must originate from a prespecified initial state and terminate in a set of goal states; (ii) it must satisfy constraints on the state, such as obstacle avoidance, along the planned motion; (iii) there must exist a control signal which, together with the planned motion, satisfies the differential equation modeling the system; (iv) lastly, a measure of cost must be minimized.

\subsection{Methods from Optimal Control}
A number of techniques for solving this problem have been developed within the optimal control literature where greater emphasis has been placed on dealing with differential constraints and less on state constraints.  
The classical approach is to use extensions of the calculus of variations to express necessary first order optimality conditions of a solution.
One celebrated approach is known as Pontryagin's minimum principle after Lev Pontryagin who pioneered the extension to optimal control~\cite{pontryagin}.
The first order optimality conditions are a number of differential equations that must satisfy initial and terminal state constraints, also known as a two-point boundary value problem.
The two-point boundary value problem can be solved in closed form in a number of examples, but in general it requires a numerical method to determine the appropriate boundary values. 
The shooting method is the standard numerical method for solving the two-point boundary value problem. 
However, extreme sensitivity of the terminal boundary value with respect to the initial boundary value makes it difficult to solve in most examples~\cite[pg. 214]{brysonapplied}.
Direct methods have become a favored approach.
These methods approximate the trajectory and control as a finite-dimensional vector space and then directly optimize the performance objective with a nonlinear programming algorithm. 
Some of the popular approaches are the direct shooting method, reviewed in~\cite{betts1998survey}, which is particularly easy to implement;  orthogonal collocation methods~\cite{benson2006direct} which offer a good numerical approximation with a relatively small number of basis vectors; and direct collocation methods~\cite{hargraves1987direct} which provide an approximation resulting in a sparse Hessian matrix for the performance objective.
The direct optimization methods rapidly converge to a locally optimal solution when they are given a suitable initial guess.
However, constructing an initial guess is more of an art than a science and if the initial guess does not satisfy the constraints, many solvers may fail to find a feasible solution at all. 
Additionally, direct optimization methods will converge to a locally optimal solution which may be unacceptably different than any globally optimal solutions.
These issues are exacerbated in complex environments that can introduce numerous local minima in the performance objective.
Thus, direct optimization methods are not sound in the sense that they may fail to find a solution when one exists.
In real-time motion planning applications where the output of the motion planner can be safety critical, these methods need to be used with a great deal of caution.

\subsection{Methods from Computational Geometry}
A classic problem in computational geometry is the mover's problem (also called the piano or couch mover's problem) which had received attention from the computational geometry community during the 1970s.
A generalization of the mover's problem was gaining interest with the increasing use of industrial manipulators in manufacturing in the late 1970s.
The generalized mover's problem consider's planning a motion for multiple polyhedra freely linked at distinguished vertices (to model an industrial manipulator arm).
In a famous paper by Reif~\cite{reif1979complexity}, an algorithm is provided for the mover's problem whose complexity is polynomial in the number of constraints defined by the obstacles.
Further, Reif shows that the generalized mover's problem is PSPACE-hard with respect to the number of degrees of freedom meaning that the difficulty of solving motion planning problems grows rapidly as the degrees of freedom increase.
Several years later Schwartz and Sharir presented a cell-decomposition algorithm~\cite{schwartz1983piano} which solved the generalized mover's problem for algebraic (instead of polyhedral) bodies with complexity in $O((2n)^{3^{r+1}}\cdot m^{2^r})$ where $r$ is proportional to the number of degrees of freedom of the robot, $n$ is proportional to the algebraic degree of the constraints, and $m$ is the number of polynomials describing the constraints.
The cell decomposition algorithm is most often cited for the doubly exponential complexity in the degrees of freedom of the robot which is not so disappointing considering Reif's results several years earlier. 
However, the important observation is that the complexity is polynomial with respect to the number of obstacles in the environment.
Canny later provided an algorithm that solved the generalized mover's problem with complexity which is only exponential in the degrees of freedom~\cite{canny1987new}.

In contrast to the trajectory optimization problems addressed by the optimal control community, the classical kinematic motion planning problems are only concerned with obstacle avoidance.
For robot manipulators there is a straight-forward justification for neglecting the robot dynamics.
The dynamics of robot manipulators determined by the principles of classical mechanics~\cite{abraham1978foundations} can almost universally be written as
\begin{equation}\label{eq:manipulator_equation}
M(q(t),\dot{q}(t))\ddot{q}(t)+C(q(t),\dot{q}(t))\dot{q}(t)+G(q(t))=B(q(t))\tau(t),
\end{equation}
where $q(t)$ is the vector of generalized coordinates for the robot's configuration and $\tau(t)$ is the vector of generalized control forces applied by the actuators.
When there is at least one actuator for each generalized coordinate and the matrix $B$ has rank equal to the number of generalized coordinates for every configuration, any twice-differentiable time parameterization of a planned motion $q$ can be executed with the control forces $\tau(t)$ solving the manipulator equation \eqref{eq:manipulator_equation}. 
When these conditions are met, the robot is said to be fully actuated which is typically the case for robot manipulators.

\subsection{Approximate Methods for Motion Planning}
The growing robotics industry in the late 1980s and 1990s called for practical solutions to the motion planning problem. 
With the disappointing complexity of available complete algorithms, researchers began developing practical techniques without theoretical guarantees, that worked well in practice~\cite{khatib1986real,hwang1992potential,rimon1992exact}.
Another practical approach with some theoretical justification was to seek methods with \emph{resolution completeness}\index{resolution completeness}~\cite{brooks1985subdivision}, meaning that a the output of the algorithm is correct for some (not known a-priori) sufficiently high resolution. 
The analogue for randomized algorithms is \emph{probabilistic completeness}\index{probabilistic completeness}~\cite{barraquand1996random}.
The concept of probabilistic completeness was introduced  at around the same time as the probabilistic roadmap (\PRM) algorithm~\cite{kavraki1996probabilistic}, whose effectiveness triggered a paradigm shift in motion planning research towards sampling-based approximate methods.
One of the more attractive features of the algorithm is that the probability of failing to correctly determine the feasibility of a motion planning problem converges to $0$ at an exponential rate in the number of random samples. 
The basic principle of the \PRM algorithm is to randomly sample a large number admissible robot configurations, and then construct a graph by connecting nearby configurations with an edge if the line between them does not contain inadmissible configurations. 

With the \PRM algorithm effectively solving classical motion planning problems, attention turned to planning for systems where dynamical constraints cannot be neglected or \emph{kinodynamic motion planning}\index{kinodynamic motion planning}. 
The difficulty in applying the \PRM to kinodynamic motion planning is that the construction of edges by linear interpolation between configurations may not be a dynamically feasible motion.
A simple adaptation is to replace linear interpolation by a local planning or steering subroutine, but this complicates individual implementations and places a burden on the user of the algorithm to provide this subroutine which itself must solve a kinodynamic motion planning problem.
The first algorithms addressing the kinodynamic motion planning problem without a steering subroutine were the expansive space trees (\EST) algorithm~\cite{EST_Journal} and the rapidly exploring random tree (\RRT) algorithm~\cite{RRT_Journal};
both of which relied on forward integration of the dynamics with random control inputs instead of a point-to-point local planning subroutine.
Further, both methods boasted probabilistic completeness like the \PRM algorithm.
The next major development in sampling-based motion planning was optimal variations of the \PRM and \RRT, denoted \PRMs and \RRTs~\cite{karaman2011sampling} developed by Karaman and Frazzoli.
The innovation was in the selection of edges in the \PRM graph.
Karaman and Frazzoli described how to construct the graph to be as sparse as possible while ensuring asymptotic optimality with respect to a performance objective in addition to probabilistic completeness.
The \RRTs algorithm is in fact more closely related to the \PRM than the \RRT. 
The \RRTs algorithm is essentially an incremental version of the \PRMs algorithm which simultaneously constructs a minimum spanning tree in the graph from an initial configuration.
While the \RRTs and \PRMs algorithms were highly impactful, a local planning subroutine was once again required.
Nonetheless, the utility of the \PRMs and \RRTs stimulated significant research efforts towards finding general methods for solving the local steering problem.
For systems with linear dynamics, classical solutions from linear systems theory can be applied to the local planning subroutine.
The drawback to this approach is that the time spent on local planning is prohibitive, taking several minutes for \RRTs to produce satisfactory solutions in the examples presented in~\cite{webb2013kinodynamic} and~\cite{perez2012lqr}.
Additionally, many systems of interest are non-linear making this approach limited in scope.

More recently, the \SST algorithm~\cite{li2015sparse} was proposed which provided an algorithm converging asymptotically to an approximately optimal solution without the use of a local planning subroutine.
This offers an advantage over \RRTs in problems where the local planning subroutine is not available in closed form.
However, \SST still requires significant running time making it difficult to apply to real-time planning.

\section{Statement of Contributions}
This thesis presents a number of theoretical contributions related to optimal kinodynamic motion planning.
The principal contribution is a resolution complete optimal kinodynamic motion planning algorithm based on a direct forward search of the set of input control signals.
Chapters \ref{chap:problem} and \ref{chap:topo} review the problem formulation addressed, as well as a careful investigation into topological properties of the set of solutions.
The key observation in these chapters is continuity of the performance objective with respect to the input signal,  presented in Lemma \ref{lem:cost_cont}, and a bound on the sensitivity of the cost function with respect to initial conditions, presented in Lemma \ref{lem:cost_sensitivity}. 
The results of Chapter \ref{chap:topo} are the basis of a generalization of classical label correcting algorithms where the comparison of relative cost can be made, not only between trajectories terminating at the same state, but between all trajectories terminating in a particular region of the state space.
The generalized label correcting conditions, described in Chapter \ref{chap:glc}, specify which segments of trajectories can be discarded from the search without compromising convergence to the optimal solution with increasing resolution. 
The generalized label correcting conditions presented in this thesis are a sharper version of the conditions presented in~\cite{paden2016generalized} resulting in faster algorithm run-times.
Theorem \ref{thm:pruning} is a restatement of Bellman's principle of optimality in the context of kinodynamic motion planning and taking into account the topological properties established in Chapter \ref{chap:topo}.
From this result, resolution completeness of the generalized label correcting method follows in Corollary \ref{cor:main} in the same way that completeness of a label correcting algorithm follows from the principle of optimality in graph search problems.  

A wide range of numerical experiments are presented in Chapter \ref{chap:glc} which confirm the theoretical results as well as suggest that this method is suitable for real-time planning applications.
To further improve the running time of the algorithm, Chapter \ref{chap:admissible_heuristics} addresses admissible heuristics for kinodynamic motion planning problems.
An admissibility condition for candidate heuristics is presented in Theorem \ref{thm:admissibility} that provides a tool for verifying the admissibility of candidate heuristics. 
Further, this condition characterizes a convex set of admissible heuristics which contains the optimal cost-to-go for a particular problem.
An infinite-dimensional linear program is formulated to optimize over the set of admissible heuristics, and it is shown in Theorem \ref{thm:HJBequivalence} that this linear program is equivalent to the Hamilton-Jacobi-Bellman equation.
Lastly, a relaxation of this linear program as a sum-of-squares program is proposed which provides provably admissible heuristics which are as close as possible to the optimal cost-to-go within a finite-dimensional subspace of polynomials.

To create a self-contained thesis, appendices containing a concise review of mathematical analysis, dynamical systems theory, and graph search algorithms are provided at the end of this document.

\renewcommand{\bibname}{Chapter References}
\bibliographystyle{amsalpha}
\bibliography{chap1}
\chapter{Problem Formulation}\label{chap:problem}
Consider a controlled dynamical system whose state space is $\mathbb{R}^{n}$, and whose input space is $\mathbb{R}^m$. 
Canonical state variables for a robot include generalized coordinates and momenta, but can also include relevant quantities such as currents and voltages in electronics.
Similarly, canonical input variables for a robot are the generalized forces applied to the robot, but higher fidelity models might include the torque as a state responding to inputs from a drive-train.
To reflect constraints such as obstacle avoidance and velocity limits, the state is restricted to remain in an open\footnote{with respect to the standard topology on $\mathbb{R}^n$.} subset of \emph{admissible states}\index{admissible states} $\xfree$ of $\mathbb{R}^n$ at each instant in time.
Similarly, to reflect the design limitations of the actuators, the control is restricted to a bounded subset of \emph{admissible control inputs}\index{admissible control inputs} $\Omega$ of   $\mathbb{R}^m$ at each instant in time.
Additionally, the motion planning objective is encoded with a terminal constraint that the executed motion terminates in an open set of \emph{goal states}\index{goal states} $\xgoal$ in $\mathbb{R}^n$.

A continuous function $x$ from a closed interval into $\mathbb{R}^n$ is called a \emph{trajectory}\index{trajectory}. 
A trajectory is a time history of states over some time interval. 
If $x$ is a trajectory, then $x(t)$ is the state on that trajectory at time $t$.
A measurable function $u$ from a closed time interval into $\Omega$ is called an \emph{input signal}\index{input signal}.
Like a trajectory, an input signal is a time history, now of control inputs.
Instantaneous changes in input are permitted as long as the control signal is mathematically well behaved (measurable). 
On the other hand, a model of the behavior of a system would not be particularly useful if it permitted instantaneous changes in state, hence the continuity requirement on trajectories (e.g. consider a robot moving instantaneously from one position to another).

\section{Decision Variables}
The input signal space \U of all input signals and trajectory space \Xxo of all trajectories starting from the state $x_{0}$ are defined
\begin{equation}\label{eq:signal_space}
\U\coloneqq\bigcup_{\tau>0}\left\{ u\in L_{1}([0,\tau]):\: u(t)\in\Omega\:\,\forall t\in[0,\tau]\right\}, 
\end{equation}
\begin{equation}\label{eq:traj_space}
\mathcal{X}_{x_0} \coloneqq \bigcup_{\tau>0}\left\{ x:[0,\tau]\rightarrow \mathbb{R}^n: \: x(0)=x_{0},\quad {\rm Lip}(x) \leq M \right\},
\end{equation}
where ${\rm Lip}(x)$ denotes the greatest lower bound\footnote{ Appendix \ref{app:orders}  defines  \emph{greatest lower bound}.} on the set of Lipschitz constants\footnote{Appendix \ref{app:continuity}  defines  \emph{Lipschitz continuity}.} of $x$.
It is important to note that the subscript $x_0$ in \eqref{eq:traj_space} is a parameter that can be varied to denote the space of trajectories originating from the initial condition $x_0$.
Figure \ref{fig:lipschitz} illustrates a trajectory with Lipschitz constant $M$.
%
%
Since the domains of functions in the sets \eqref{eq:signal_space} and \eqref{eq:traj_space} is variable, it will be useful to denote the \emph{terminal time}\index{terminal time} of a trajectory's or control signal's domain by $\tau(x)$ or $\tau(u)$ so that the \emph{terminal state}\index{terminal state} or \emph{terminal control input}\index{terminal control input} is given by $x(\tau(x))$ or $u(\tau(u))$ respectively.
\begin{wrapfigure}{R}{0.4\textwidth}
	\frame{\includegraphics[width=0.4\textwidth]{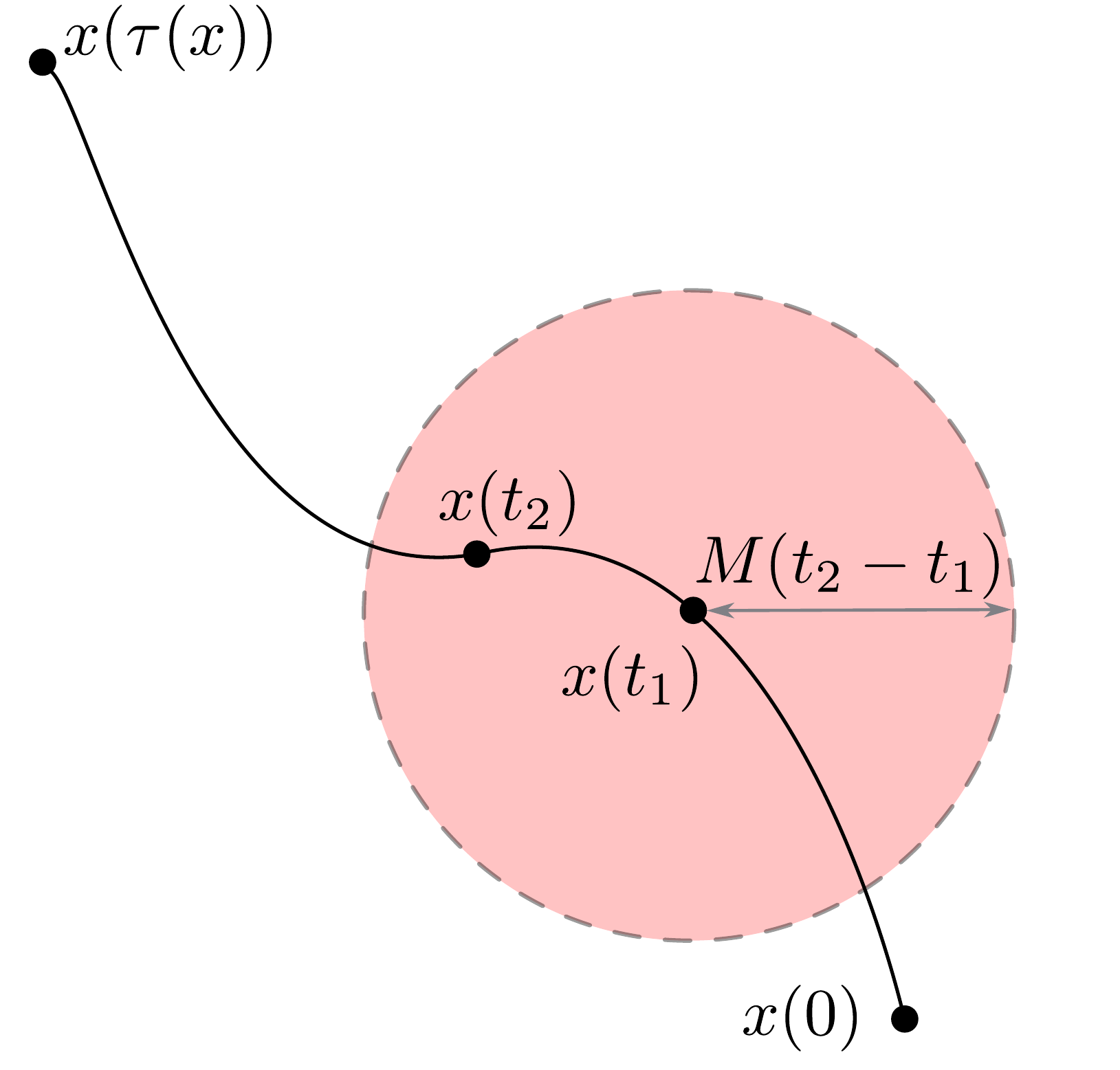}}
	\caption{The black line illustrates a trajectory $x$ with $\Vert x \Vert_{\rm Lip}\leq M$. The Lipschitz constant bounds the distance between the states at $t_1$ and $t_2$ are. }\label{fig:lipschitz}
\end{wrapfigure}

The motivation for defining these two sets comes from the dynamical model of the system\footnote{Taking the derivative of \eqref{eq:int_dynamics} results in the more familiar diffrential form which is equivalent to the integral form.} given in \eqref{eq:int_dynamics}.
The distinguishing feature between various system models is the function $f$ taking a state-control pair into $\mathbb{R}^n$.
Throughout this thesis, the function $f$ is assumed to have a known global Lipschitz constant $L_f$ in its first argument, is measurable in its second, and is bounded by a constant $M$.  
\begin{equation}\label{eq:int_dynamics}
x(t)=x_0+\int_{[0,t]}f(x(t),u(t)) \mu(dt).
\end{equation}
A trajectory $x$ is called a \emph{solution}\index{solution (integral equation)} to  \eqref{eq:int_dynamics} for a given input signal $u\in \mathcal{U}$ and initial condition $x_0\in\mathbb{R}^n$ if it has the same time domain as $u$ and satisfies the equation for each $t$ in the domain of $u$. 
The measure $\mu$ will refer to the usual Lebesgue measure\footnote{A review of Lebesgue integration is presented in Appendix \ref{app:functional_analysis}.} on $\mathbb{R}$.
Let $\varphi_{x_0}\subset \U \times \Xxo$ be a binary relation between the input signal space and trajectory space defined by $(u,x)$-pairs where $x$ is a solution to \eqref{eq:int_dynamics} with input signal $u$.
Theorem \ref{thm:existence-uniqueness} below states that the relation $\varphi_{x_0}$ is a function mapping \U into \Xxo. 
This will be important to the analysis presented in later chapters. 
\begin{theorem}\label{thm:existence-uniqueness}
	For each input signal in \U,  there is a unique solution to \eqref{eq:int_dynamics} in \Xxo.    
\end{theorem}
The function $\varphi_{x_0}$ will be called the \emph{system map}\index{system map} in light of this observation.
This is a straight forward corollary to the standard existence-uniqueness theorem for integral equations provided in  Appendix \ref{sec:ode_models}.

\section{Problem Specifications}
We will denote the subset of trajectories which originate from a given \emph{initial condition}\index{initial condition} $x_{ic}$ and remain within the set $\xfree$ by  $\Xfree$.
\begin{equation}
\Xfree \coloneqq \left\{ x\in \mathcal{X}_{x_\mathrm{ic}}:\,\, x(t)\in X_\mathrm{free}\,\, \forall t\in [0,\tau (x)] \right\}.
\end{equation}
\begin{wrapfigure}{R}{0.4\textwidth}
	\frame{\includegraphics[width=0.4\textwidth]{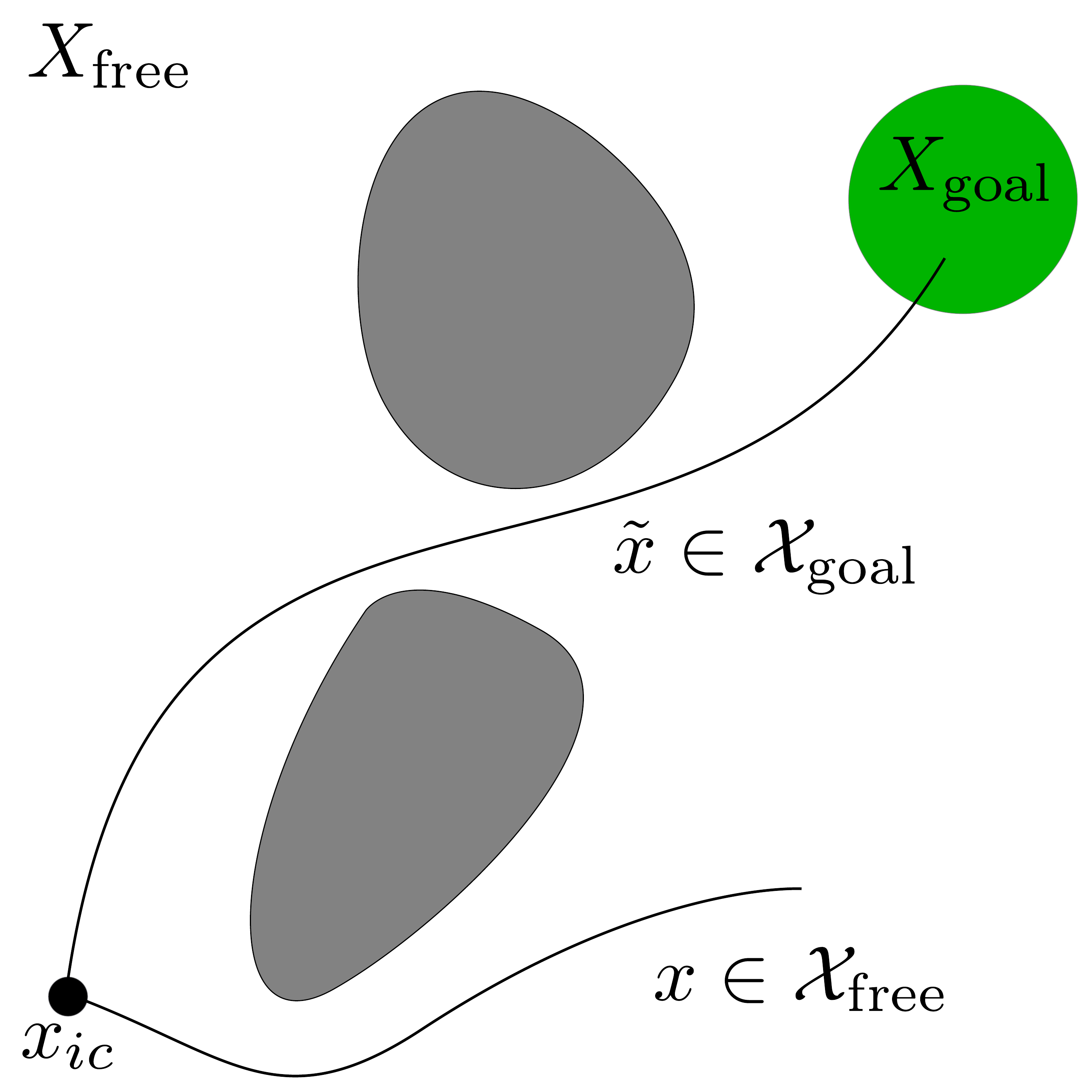}}
	\caption{The admissible states \xfree consist of the rectangular region excluding the gray obstacles. The goal region \xgoal is the green circle. $\tilde{x}$ depicts a trajectory in \Xgoal while $x$ shows a trajectory in \Xfree.}\label{fig:traj_samples}
\end{wrapfigure}
%
%
The subset of trajectories in \Xfree which additionally terminate in \xgoal is denoted \Xgoal.
\begin{equation}
\Xgoal\coloneqq \left\{x\in \Xfree:\, x(\tau(x))\in \xgoal\right\}.
\end{equation}
Figure \ref{fig:traj_samples} illustrates example trajectories in \Xgoal and \Xfree.

The subset of input signals in \U such that $\varphi_{x_{ic}}(u)\in\Xfree$ is denoted \Ufree.
Similarly, the subset of input signals \U such that $\varphi_{x_{ic}}(u)\in \Xgoal$ is denoted \Ugoal.
It is important to remark that \Ufree and \Ugoal are the inverse images of \Xfree and \Xgoal under the system map $\varphi_{x_{ic}}$,
\begin{equation}\label{eq:solution_sets}
\Ufree=\varphi^{-1}_{x_{ic}}(\Xfree),\qquad \Ugoal=\varphi^{-1}_{x_{ic}}(\Xgoal).
\end{equation}
The first part of the problem specification is to find a signal $u\in \Ugoal$. 

Since there are often many input signals in \Ugoal, we additionally would like to minimize a performance objective.
Performance objectives of the form 
\begin{equation}\label{eq:cost}
J_{x_{0}}(u)=\int_{[0,\tau(u)]}g([\varphi_{x_{0}}(u)](t),u(t))\, \mu(dt),
\end{equation}
are considered.
The function $g$ from $\xfree \times \Omega$ into $(0,\infty)$ is a \emph{running cost}\index{running cost} which gives a rate of accumulation of cost for each state-control pair.
It is assumed that a global Lipschitz constant $L_g$ for $g$ on $\xfree \times \Omega$ is known. 

Since the state trajectory is completely determined by the input signal and initial condition, the cost functional $J$ is considered a mapping from \U into $(0,\infty)$ parameterized by the initial condition.
It will be convenient to extend the domain of $J_{x_0}$ with a symbol \NULL where $J_{x_0}(\NULL)=\infty$ for every $x_0$.
We will also adopt the convention that $\inf_{u\in\emptyset}J_{x_\mathrm{ic}}(u)=\infty$.

Intuitively, we would like to minimize the cost $J$ over the set \Ugoal. 
However, a particular problem instance may not admit a minimizer or even a solution at all.
Therefore, we will seek a solution to the following \emph{relaxed optimal kinodynamic motion planning problem}\index{relaxed optimal kinodynamic motion planning problem}:
\begin{problem}\label{Problem}
	Given an initial condition $x_{ic}$, a set of admissible states $\xfree$, a set of goal states $\xgoal$, a set of admissible control inputs $\Omega$, a dynamic model $f$, and a running cost $g$; find a sequence $\{u_{R}\}\subset\mathcal{U}_\mathrm{goal}\cup \NULL$ such that 
	\begin{equation}
	\lim_{R\rightarrow\infty}J_{x_\mathrm{ic}}(u_R)=\inf_{u\in\mathcal{U}_\mathrm{goal}}J_{x_\mathrm{ic}}(u)\coloneqq c^{*}.\label{eq:meaningful_problem}
	\end{equation}
\end{problem}
An algorithm parameterized by a resolution $R \in \mathbb{N}$ whose output for each $R$ forms a sequence solving this problem will be called \emph{resolution complete}\index{resolution complete}.

\section{Review of Assumptions}\label{sec:assumptions}
To summarize the problem formulation, the \emph{problem data}\index{problem data} for the optimal kinodynamic motion planning problem is $(x_{ic},\xfree,\xgoal,\Omega,f,g)$, where $x_{ic}$ is an initial condition, $X_{\rm free}$ is the set of admissible states, \xgoal is the set of goal states, $f$ is the model of the system dynamics, and $g$ is the running cost.
The assumptions on the problem data described in this section are:
\begin{enumerate}[label=A-\arabic*,itemindent=0.25cm]
	\item The sets \xfree and \xgoal are open with respect to the topology induced by the Euclidean distance. 
	\item The set of admissible inputs $\Omega$ is bounded, and a bound $u_{\rm max}$ is known. 
	\item The running cost $g$ is strictly positive and Lipschitz continuous in both arguments with a known Lipschitz constant $L_g$.
	\item The dynamic model $f$ is Lipschitz continuous in its first argument with a known Lipschitz constant $L_f$, Lebesgue measurable in its second argument, and bounded in both arguments by $M$.
\end{enumerate}

There are two important things to consider regarding the assumptions of the input data to an algorithm or method.
First, the applicability of a method requiring the assumptions to be met depends on how general the assumptions are.
Secondly, given an instance of problem data, discerning whether it meets the assumptions should be an easy decision problem, decidable with an algorithm of lesser complexity than the algorithm the problem data is being given to. 
Otherwise, there is little value to the algorithm.

The formulation and assumptions of this chapter are quite general and can be discerned by inspection from typical problem data. 
For example, if \xfree is described by the union of a finite set of closed polyhedra and \xgoal is the union of a set of open polyhedra, then A-1 is satisfied.
Similarly, robot actuators have clear design limitations bounding the set of admissible inputs so that A-2 is satisfied in all practical instances. 
Assumption A-3 is often verified by inspection of the running cost $g$.
For example, minimum time problems with $g(z,w)=1$ are among the most frequently discussed objectives. 
The Lipschitz constant in this case is $0$.
Lipschitz continuity of the dynamics is generally required for dynamical models to be well defined. 
Therefore, most models derived from physical laws will satisfy A-4.  
As a final remark on the generality of this formulation, it does not require any controllability properties of the dynamic model $f$ as is the case in~\cite{Li2016Asymptotically-,karaman2010optimal}.
Although this makes the formulation more general it does not necessarily add to the practical value since there are few applications where motion planning is used on an uncontrollable system.

\subsection{Verifiability of Assumptions}
Most well known sampling-based motion planning algorithms impose an abstract assumption on the problem data for the algorithm to have the desired theoretical guarantees.
In~\cite{kavraki1998analysis} the \emph{$\varepsilon$-goodness}\index{$\varepsilon$-goodness} property was the basis of an analysis of the \PRM algorithm. 
In~\cite{hsu1997path} the \emph{expansiveness}\index{expansiveness} property was required of the problem data to prove the probabilistic completeness of the \EST algorithm.
Similarly, in~\cite{RRT_Journal} the existence of an \emph{attraction sequence}\index{attraction sequence} is required to prove the probabilistic completeness of the \RRT algorithm.
Lastly, in~\cite{karaman2011sampling}, probabilistic completeness of the \RRTs algorithm relies on the problem data satisfying a  \emph{$\delta$-robustness}\index{$\delta$-robustness} property. 

While these assumptions are precisely defined, it is not clear how to verify if these properties are satisfied, whether this subset of problem instances are of practical value, or if there even exists a problem instance satisfying a particular assumption.
These are important open problems in motion planning that have received little attention.

The following chapter uses techniques from the subject of topology to develop a  foundation of tools for analyzing the optimal kinodynamic motion planning problem.
A brief introduction to topology is presented in Appendix \ref{app:topo}.

To motivate the use of these techniques, we apply them to formally prove the intuitive fact that an open set of admissible states is sufficient for the  $\delta$-robustness assumption of the \RRTs algorithm.
The \RRTs algorithm requires that feasible problem instances be \emph{$\delta$-robustly feasible}\index{$\delta$-robust feasibility}.
That is, feasible problem instances must have a trajectory $x\in\Xgoal$, such that for some $\delta>0$ the trajectory satisfies 
\begin{equation}
B_\delta(x(t))\subset \xfree,\quad \forall t\in [0,\tau(x_1)].
\end{equation}  

\begin{lemma}\label{lem:rrt_glc}
	If \xfree is open, then every feasible problem instance ($\Xgoal\neq \emptyset$) is $\delta$-robustly feasible.
\end{lemma}
\begin{proof}
	Let $x$ be a trajectory in \Xgoal. 
	Since $x$ is continuous, $x([0,\tau(x)])$ is a compact subset of \xfree for all $t$ in the interval $[0,\tau(x)]$ (cf. Lemma \ref{lem:compact_image}). By assumption, \xfree is open so there exists a $\delta>0$ such that $B_{\delta}(x(t))\subset \xfree$ for all $t$ in $[0,\tau(x)]$ (cf. Corollary \ref{cor:compact_closed}) which is the definition of $\delta$-robust feasibility.

\end{proof}

\bibliographystyle{amsalpha}
\bibliography{chap2}
\chapter{Topological Properties of the Input Signal and Trajectory Spaces}\label{chap:topo}

The problem formulation will use concepts from general topology to establish basic properties used by the \GLC method.
The function spaces \Xic and \U are equipped with metrics in order to perform an analysis on the subsets \Xfree, \Xgoal, \Ufree, and \Ugoal in their respective metric topologies. 
It will follow from the assumption that \xfree and \xgoal are open in the standard topology on $\mathbb{R}^n$ that \Xfree and \Xgoal are open in the metric topology on \Xic.
We then review a result based on Gronwall's inequality\footnote{Gronwall's inequality is discussed in Appendix \ref{app:functional_analysis}.} which shows $\varphi_{x_{ic}}$ is a continuous mapping from \U into \Xic. 
A direct consequence of these observations is \Ufree and \Ugoal are open subsets of \U since the inverse image of an open set under a continuous function is open\footnote{This is the notion of continuity for functions on topological spaces. The relation to the definition for metric spaces is discussed in Lemma \ref{app:ep_delt_cont}.}.      
The final important observation is the continuity of the cost function from \U into $\mathbb{R}$. 

These results will be used in later chapters as follows: A dense subset of \U will be constructed in Chapter \ref{chap:approximation} as an approximation of \U. 
Since \Ufree and \Ugoal are open in \U, the approximation will also be dense in \Ufree and \Ugoal. 
Next, the image of the approximation, intersected with \Ufree and \Ugoal, under the cost function $J_{x_{ic}}$ will be dense in $J_{x_{ic}}(\Ufree)$ and $J_{x_{ic}}(\Ugoal)$ so that a signal, as will be seen, in the approximation of \Ugoal achieves the optimal cost with arbitrarily high accuracy.
Lastly, in Chapter \ref{chap:glc} the continuity of the cost function is used once again to identify a large subset of the approximation which can be discarded without compromising the accuracy with which it approximates a signal with the optimal cost.

\section{Metrics on the Signal and Trajectory Spaces}

The input signal and trajectory spaces become metric spaces when equipped with the metrics $d_{\U}$ and $d_{\Xxo}$  adapted from~\cite{yershov2011sufficient}: 
\begin{equation}\label{eq:du}
d_{\mathcal{U}}(u_{1},u_{2})\coloneqq\int_{[0,\min\{\tau(u_{1}),\tau(u_{2})\}]}\Vert u_{1}(t)-u_{2}(t)\Vert_{2} \, \mu(dt)+u_{max}|\tau(u_{1})-\tau(u_{2})|.
\end{equation}
\begin{equation}
d_{\mathcal{X}}(x_{1},x_{2})\coloneqq\max_{t\in\left[0,\min\{\tau(x_{1}),\tau(x_{2})\}\right]}\left\{ \left\Vert x_{1}(t)-x_{2}(t)\right\Vert \right\} +M|\tau(x_{1})-\tau(x_{2})|.\label{eq:dx}
\end{equation}
Figures \ref{fig:du} and \ref{fig:dx} illustrate these two metrics. 
Several results presented in this chapter regarding these metrics are simple variations of results in~\cite{yershov2011sufficient}. However, appropriately modified proofs are provided for the version of these results needed in later chapters.
\begin{figure}
	\centering
	\includegraphics[width=0.8\textwidth]{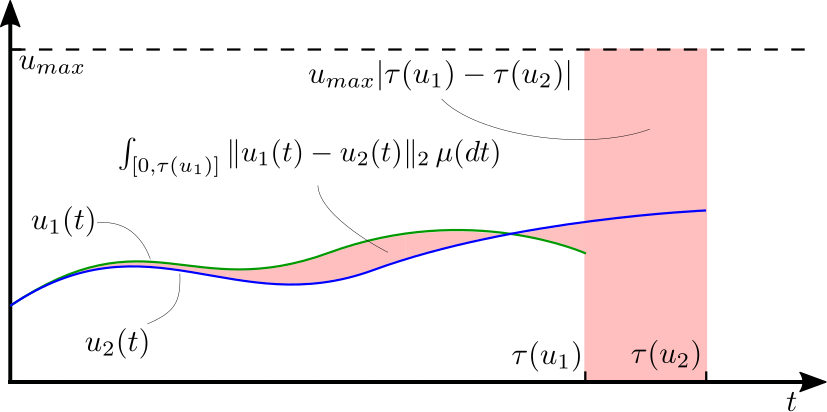}
	\caption{The distance $d_\U(u_1,u_2)$ between two signals is illustrated by the region shaded in red. Note that the signals have different time domains, and that the second term in \eqref{eq:du} describes the worst-case difference on the interval $[\tau(u_1),\tau(u_2)]$ between the two signals if $u_1$ were extended to have the same domain as $u_2$.}\label{fig:du}
\end{figure}
\begin{figure}
	\centering
	\includegraphics[width=0.8\textwidth]{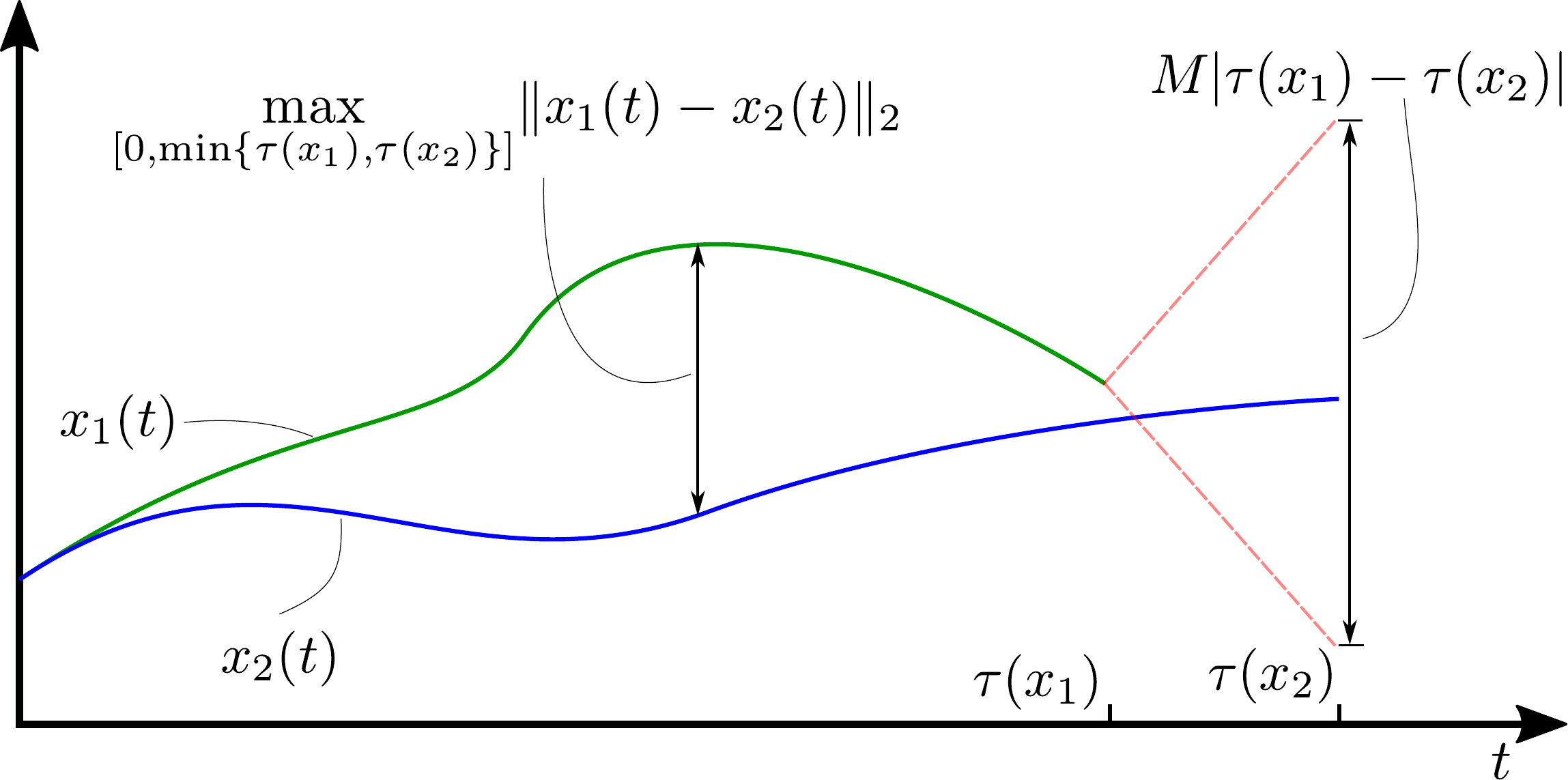}
	\caption{The two terms making up the expression for $d_\X(x_1,x_2)$  are illustrated by two annotated distances. Observe that the second term in \eqref{eq:dx} describes the worst case distance between the two if $x_1$ were extended to have the same domain as $x_2.$}\label{fig:dx}
\end{figure}

While these distance functions have been utilized as metrics in the literature, verification that they satisfy the definition of a metric have not been published. 
These proofs are provided in Appendix \ref{app:pfs}.

\section{Continuity of the System Map $\varphi_{x_0}$ and Cost Functional $J_{x_0}$}\label{sec:continuity}
\begin{wrapfigure}{l}{0.4\textwidth}
	\frame{\includegraphics[width=0.4\textwidth]{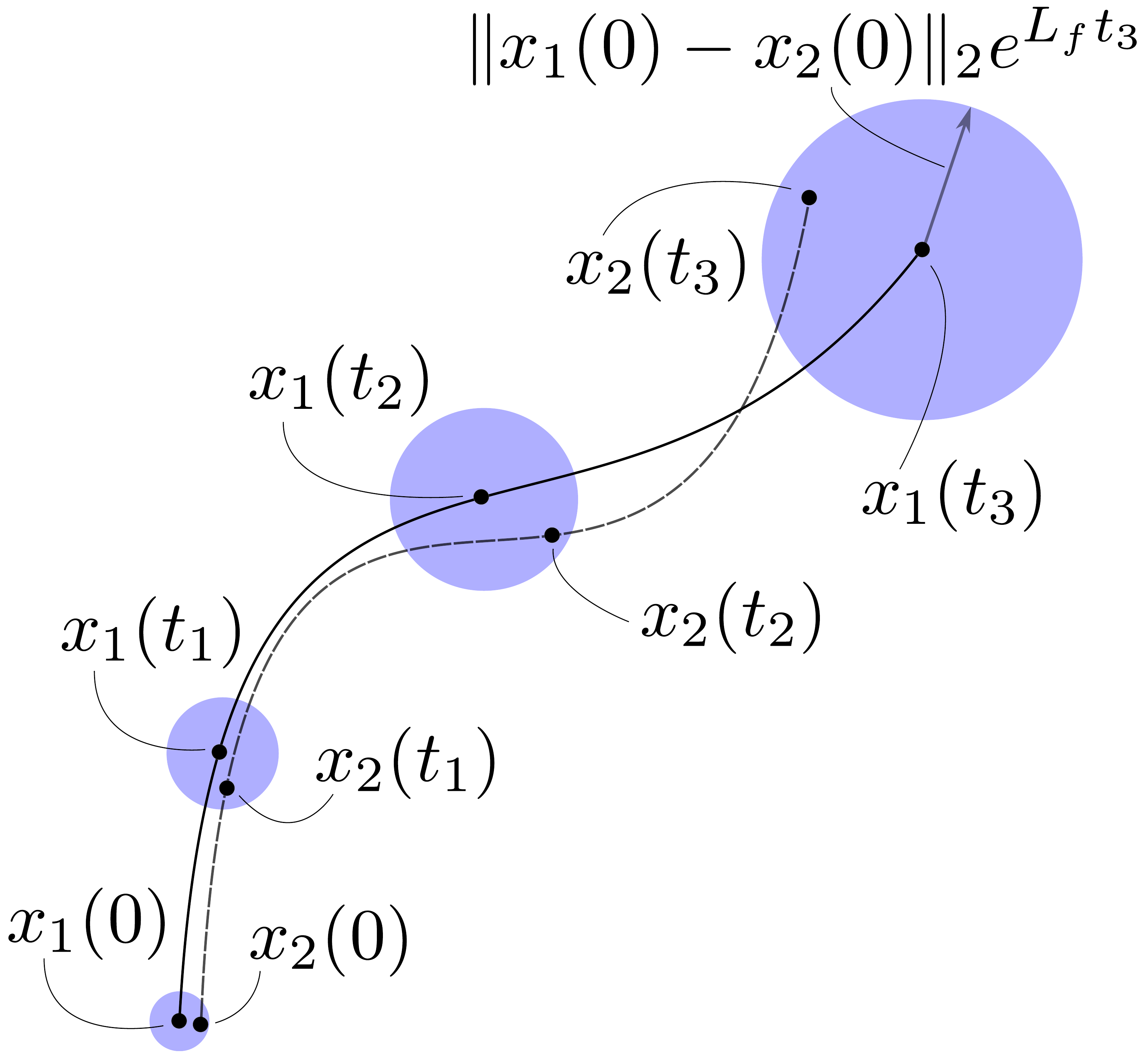}}
	\caption{Illustration of Lemma \ref{lem:cont_ic}. The distance of $x_2(t)$ from $x_1(t)$ is bounded by the initial distance multiplied by $e^{L_f t}$. The circles illustrate possible locations of each $x_2(t_i)$ relative to $x_1(t_i)$ predicted by the inequality.}\label{fig:cont_ic}
\end{wrapfigure}
A classical result in dynamical systems theory is the continuity of trajectories with respect to initial conditions.
Lemma \ref{lem:cont_ic} is a restatement of this result with the appropriate notation.
\begin{lemma}\label{lem:cont_ic}
	Let $u$ be an input signals in \U; $x_0,z_0$ initial conditions in $\mathbb{R}^n$; and $x_1\coloneqq \varphi_{x_0}(u)$, $x_2\coloneqq \varphi_{z_0}(u)$ the corresponding trajectories. Then the difference between these trajectories satisfies the inequality
	\begin{equation}\label{eq:cont_ic}
		\left\Vert x_1(t)-x_2(t)\right\Vert_2 \leq \Vert x_0-z_0 \Vert_2 e^{L_f t}.
	\end{equation} 
\end{lemma}
The standard proof can be found in Appendix \ref{sec:ode_models}.
Figure \ref{fig:cont_ic} illustrates the inequality in \eqref{eq:cont_ic}.

\begin{lemma}[\cite{yershov2011sufficient}]
	The system map $\varphi_{x_0}$ is continuous from $(\U,d_\U)$ into $(\Xxo,d_\X)$.
\end{lemma}
\begin{proof}
	Let $u_1,u_2$ be input signals in \U, ordered so that $\tau(u_1)\leq \tau(u_2)$. 
	We will show that the map $\varphi_{x_0}$ is continuous at $u_1$.
	Denote trajectories $\varphi_{x_{0}}(u_{1})$ and $\varphi_{x_{0}}(u_{2})$ by $x_{1}$ and $x_{2}$ respectively.
	For $t$ in the time interval $[0,\tau(u_1)]$, it follows from equation \eqref{eq:int_dynamics} that
	\begin{equation}\Stretch
	\begin{array}{rcl}
		\Vert x_1(t)-x_2(t) \Vert_2 &=& \left\Vert \int_{[0,t]} f(x_1(\tau),u_1(\tau))\, \mu(d\tau) - \int_{[0,t]} f(x_2(\tau),u_2(\tau))\, \mu(d\tau)\right\Vert_2\\
		&\leq&  \int_{[0,t]}\left\Vert  f(x_1(\tau),u_1(\tau)) - f(x_2(\tau),u_2(\tau))\, \right\Vert_2 \mu(d\tau).
	\end{array}
	\end{equation}
	Then using the Lipschitz continuity of the dynamic model in A-4,
	\begin{equation}\Stretch
		\begin{array}{rcl}
		\Vert x_1(t)-x_2(t) \Vert_2 &\leq& \int_{[0,t]} L_f\Vert x_1(\tau)-x_2(\tau)\Vert_2+L_f\Vert u_1(\tau)-u_2(\tau)\Vert_2\, \mu(d\tau)\\
		\Vert x_1(t)-x_2(t) \Vert_2 &\leq& \int_{[0,t]} L_f\Vert x_1(\tau)-x_2(\tau)\Vert_2\,\mu(d\tau)+L_f\int_{[0,t]}\Vert u_1(\tau)-u_2(\tau)\Vert_2\, \mu(d\tau)\\
		
		&\leq& L_fd_\U(u_1,u_2)+\int_{[0,t]} L_f\Vert x_1(\tau)-x_2(\tau)\Vert_2\, \mu(d\tau).
		\end{array}
	\end{equation}
	Then by Gronwall's inequality (cf. Lemma \ref{lem:gw_ineq}),
	\begin{equation}\label{eq:cont_pf_1}
				\Vert x_1(t)-x_2(t) \Vert_2 \leq L_fd_\U(u_1,u_2) e^{L_ft}.
	\end{equation}
	On the remaining time domain $[\tau(u_1),\tau(u_2)]$, observe that $d_\U(u_1,u_2)<\delta$ implies $|\tau(u_1)-\tau(u_2)|<d_\U(u_1,u_2)/u_{max}$.  Therefore,
	\begin{equation}\label{eq:cont_pf_2}
		\begin{array}{rcl}
		M|\tau(x_1)-\tau(x_2)|&=& M|\tau(u_1)-\tau(u_2)|\\
		&<&Md_\U(u_1,u_2)/u_{max}.
		\end{array}
	\end{equation}
	Then $d_\X(x_1,x_2)$ can be bounded as a function of $\delta$ by combining \eqref{eq:cont_pf_1} and \eqref{eq:cont_pf_2},
	\begin{equation}\Stretch\label{eq:cont_mod}
	\begin{array}{rcl}
		d_\X(x_1,x_2)&=&\underset{{t\in[0,\min\{\tau(x_1),\tau(x_2)\}]}}{\max}\{ \Vert x_1(t)-x_2(t) \Vert_2\} + M|\tau(x_1)-\tau(x_2)|\\
		&<& L_f d_\U(u_1,u_2) e^{L_f \tau(u_1)}+ Md_\U(u_1,u_2)/u_{max}\\
		&=& (L_fe^{L_f\tau(u_1)}+M/u_{max})d_\U(u_1,u_2).
		\end{array}
	\end{equation}
	Thus, the map $\varphi_{x_0}$ continuous at $u_1$ since for any $\varepsilon>0$, the signal $u_2$ can be chosen sufficiently close to $u_1$ so that $d_\X(x_1,x_2)<\varepsilon$, which is equivalent to the definition of continuity    
	(cf. Lemma \ref{app:ep_delt_cont}).
\end{proof}
Note that the continuity is not uniform because of the appearance of $\tau(u_1)$ in \eqref{eq:cont_mod}.

\section{Properties of the Set of Solutions}\label{sec:open_sets}
\begin{figure}
	\centering
	\frame{\includegraphics[width=0.6\textwidth]{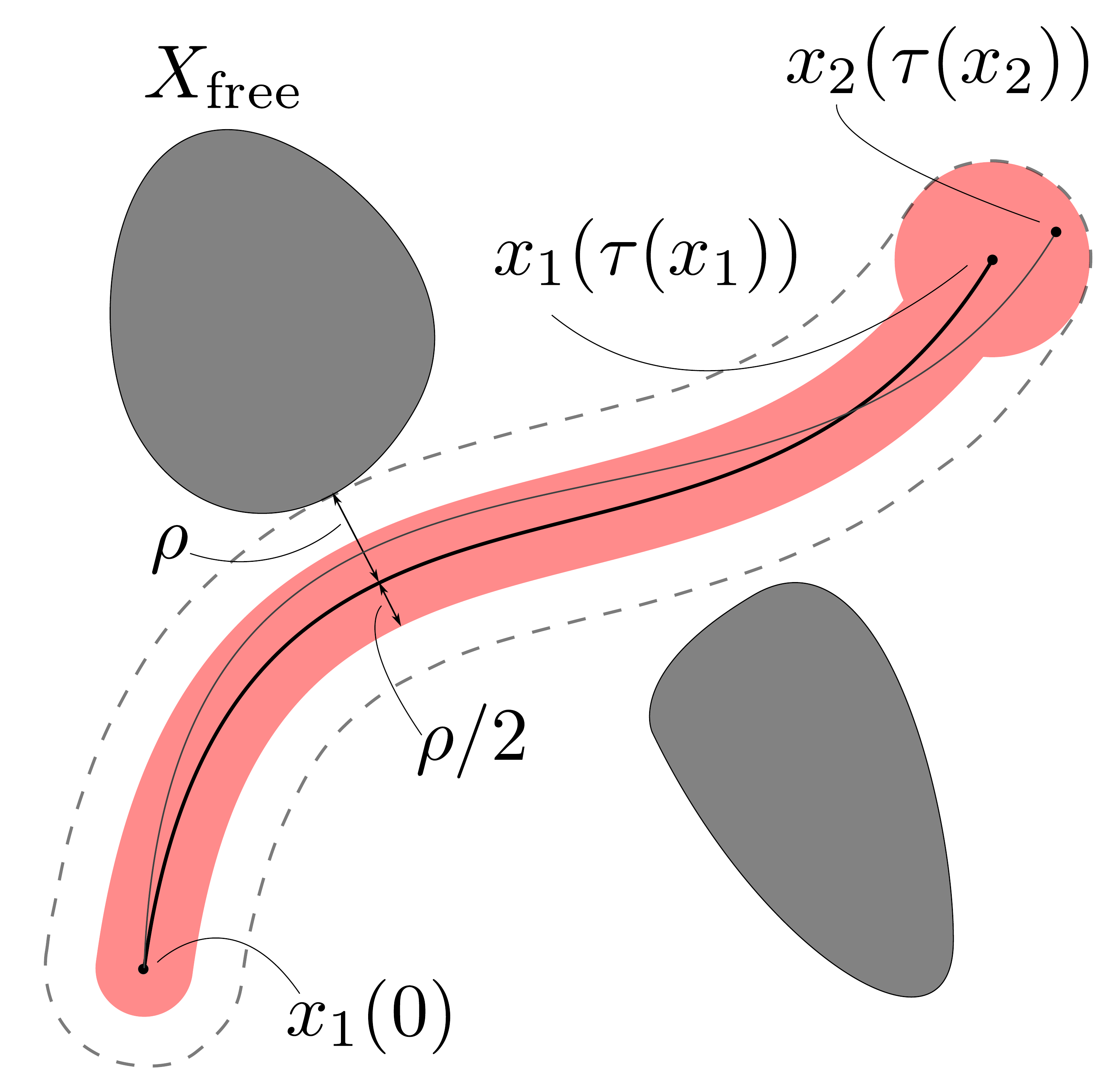}}
	\caption{Illustration of the trajectories $x_1$ and $x_2$ discussed in Lemma \ref{lem:xfree_open}. Any trajectory $x_2$ that is within a distance (in the sense of $d_\X$) of $\rho/2$ of $x_1$, by definition, remains in the red region around $x_1$. }\label{fig:X_open}
\end{figure} 
The first important observation is that \Xfree and \Xgoal are open in the metric topology defined by $d_\X$ when \xfree and \xgoal are open in the standard topology. 
\begin{lemma}[\cite{yershov2011sufficient}]\label{lem:xfree_open}
	 \Xfree is an open subset of \Xic in the topology induced by $d_\X$.
\end{lemma}
\begin{proof}
	Either \Xfree is empty, in which case it is open (cf. Appendix \ref{app:topo}), or it is non-empty. Assume the latter case and let $x_1$ be an element of \Xfree. Since $x_1$ is a Lipschitz continuous function from the compact interval $[0,\tau(x)]$ into $\mathbb{R}^n$, its image, $x_1([0,\tau(x)])$, is also compact (cf. Lemma \ref{lem:compact_image}). Next, since $\xfree$ is open, and $x([0,\tau(x_1)])$ is compact, there exists a $\rho>0$ such that $B_{\rho}(x_1(t))\subset \xfree$ for all $t\in[0,\tau(x_1)]$ (cf. Corollary \ref{cor:compact_closed}). 
	Now consider a trajectory $x_2$ in the ball of radius $\rho/2$ centered at $x_1$ in \Xic (note that this is in the function space \Xic). Then, from the definition of $d_\X$,
\begin{equation}\label{eq:pf_def_ref}
\max_{t\in\left[0,\min\{\tau(x_{1}),\tau(x_{2})\}\right]}\left\{ \left\Vert x_{1}(t)-x_{2}(t)\right\Vert_2 \right\} +M|\tau(x_{1})-\tau(x_{2})|<\rho/2.
\end{equation}
The first term in \eqref{eq:pf_def_ref} implies, that on the interval $[0,\min\{\tau(x_{1}),\tau(x_{2})\}]$, the distance between $x_1(t)$ and $x_2(t)$ is within $\rho/2$.
If $\tau(x_2)\leq \tau(x_1)$, then $x_2$ is necessarily in \Xfree. 
Now suppose that $\tau(x_1)<\tau(x_2)$. Then for $t>\tau(x_1)$,
\begin{equation}\label{eq:triangle_pf}
\Vert x_2(t)-x_1(\tau(x_1))\Vert_2 \leq \Vert x_1(\tau(x_1))-x_2(\tau(x_1)) \Vert_2 + \Vert x_2(\tau(x_1))- x_2(t) \Vert_2.
\end{equation} 
The last term in \eqref{eq:triangle_pf} is further bounded by the Lipschitz constant for $x_2$,
\begin{equation}\Stretch\label{eq:triangle_pf_ctd}
\begin{array}{rcl}
\Vert x_2(t)-x_1(\tau(x_1))\Vert_2 &\leq & \Vert x_1(\tau(x_1))-x_2(\tau(x_1)) \Vert_2 + M|t-\tau(x_1)|\\
&\leq & \Vert x_1(\tau(x_1))-x_2(\tau(x_1)) \Vert_2 + M|\tau(x_2)-\tau(x_1)|\\
&\leq& \rho.
\end{array}
\end{equation} 
Thus, at any time $t$, $x_2(t)$ is within a distance of $\rho$ from some point in $x_1([0,\tau(x_1)])$. Therefore, $x_2$ is in \Xfree. Since the choice of $x_2$ was arbitrary, $x_1$ is an interior point of \Xfree; since the choice of $x_1$ was arbitrary, \Xfree is open.
    
\end{proof}
The trajectories discussed in the proof are illustrated in Figure \ref{fig:X_open}.
An analogous observation holds for the set \Xgoal.
\begin{lemma}[\cite{yershov2011sufficient}]\label{lem:xgoal_open}
	\Xgoal is an open subset of \Xic with respect to the topology induced by $d_\X$.
\end{lemma} 
\begin{proof}
	Either \Xgoal is empty, in which case it is open, or it is non-empty. 
	Assume the latter case and let $x_1$ be a trajectory in \Xgoal in the metric subspace $(\Xfree,d_\X)$. 
	Since \xgoal is open, the terminal point $x_1(\tau(x_1))$ is an interior point of \xgoal, so there exists a $\rho>0$ such that $B_{\rho/2}(x_1(\tau(x_1)))$ is a subset of \xgoal. 
	Now consider a trajectory $x_2$ such that $d_\X(x_1,x_2)<\rho/2$. 
	From the definition of $d_\X$,
	\begin{equation}\label{eq:xopen_pf}
	\max_{t\in\left[0,\min\{\tau(x_{1}),\tau(x_{2})\}\right]}\left\{ \left\Vert x_{1}(t)-x_{2}(t)\right\Vert_2 \right\} +M|\tau(x_{1})-\tau(x_{2})|<\rho/2.
	\end{equation}
	From \eqref{eq:xopen_pf}, the difference in terminal times of $x_1$ and $x_2$ satisfies the bound 
	\begin{equation}
	|\tau(x_1)-\tau(x_2)|<\rho/(2M).
	\end{equation} 
	Similarly, the distance between the states of the two trajectories at the shorter of the terminal times is bounded by
	\begin{equation}
	\Vert x_1(\min\{\tau(x_1),\tau(x_2)\}) - x_1(\min\{\tau(x_1),\tau(x_2)\}) \Vert_2 < \rho/2.
	\end{equation}	
	These observations are then used to bound the distance between the terminal states of the two trajectories. 
	The next step requires considering two cases: First, suppose $\tau(x_2)\leq \tau(x_1)$. 
	Then 
	\begin{equation}\Stretch
	\begin{array}{rcl}
	\Vert x_2(\tau(x_2)) - x_1(\tau(x_1)) \Vert_2 &\leq& \Vert x_2(\tau(x_2)) - x_1(\tau(x_2)) \Vert_2 + \Vert x_1(\tau(x_1))-x_1(\tau(x_2)) \Vert_2\\
	&\leq& \rho/2 + M|\tau(x_2)-\tau(x_1)|\\
	&\leq& \rho/2 + \frac{M\rho}{2M} \\
	&\leq& \rho.
	\end{array}
	\end{equation}
	Second, suppose $\tau(x_1)<\tau(x_2)$. 
	The analogous application of the triangle inequality yields,
	\begin{equation}\Stretch
	\begin{array}{rcl}
	\Vert x_2(\tau(x_2)) - x_1(\tau(x_1)) \Vert_2 &\leq& \Vert x_2(\tau(x_1)) - x_1(\tau(x_1)) \Vert_2 + \Vert x_2(\tau(x_1))-x_2(\tau(x_2)) \Vert_2\\
	&\leq& \rho/2 + M|\tau(x_2)-\tau(x_1)|\\
	&\leq& \rho/2 + \frac{M\rho}{2M} \\
	&\leq& \rho.
	\end{array}
	\end{equation}
	Thus, the terminal state $x_2(\tau(x_2))$ is an element of \xgoal. 
	Since the choice of $x_2$ was arbitrary, $x_1$ is an interior point of \Xgoal; since the choice of $x_1$ was arbitrary, every point in \Xgoal is an interior point. Thus, \Xgoal is open in the metric subspace $(\Xfree,d_\X)$. Since \Xfree is open \Xic, \Xgoal is open in \Xic (cf. Lemma \ref{lem:subspace}).
\end{proof}
Now that we have established that the system map is continuous and \Xfree and \Xgoal are open subsets of its codomain, by definition of continuity, the inverse images of \Xfree and \Xgoal under the system map are open.
\begin{theorem}[\cite{yershov2011sufficient}]\label{thm:uopen}
	\Ufree and \Ugoal are open in the topology induced by $d_\U$. 
\end{theorem}
\begin{proof}
	By Lemma \ref{lem:cont_ic} the map $\varphi_{x_{ic}}$ is continuous, and by Lemmas \ref{lem:xfree_open} and \ref{lem:xgoal_open}, \Xfree and \Xgoal are open. The sets \Ufree and \Ugoal are defined as 
	\begin{equation}
		\Ufree = \varphi^{-1}_{x_{ic}}(\Xfree),\qquad \Ugoal = \varphi^{-1}_{x_{ic}}(\Xgoal).
	\end{equation}
	 By the definition of continuity\footnote{Continuity is defined in Appendix \ref{app:topo}}, \Ufree and \Ugoal are open.
\end{proof} 
\par{\noindent}This is the essential property needed to prove the resolution completeness of the algorithm presented in later chapters.

Next, similar continuity properties for the cost function are derived. 
This continuity will be used to argue that the image of a dense subset of \U will be dense in the cost space.
\begin{lemma} \label{lem:cost_cont}
	$J_{x_0}:\mathcal{U}\rightarrow\mathbb{R}$ is continuous for any $x_0\in \mathbb{R}^n$.
\end{lemma}
\begin{proof}
	Let $u_{1},u_{2}\in\mbox{\ensuremath{\mathcal{U}}}$ and, without loss of generality, assume $\tau(u_{1})\leq\tau(u_{2})$. 
	Denote trajectories $\varphi_{x_{0}}(u_{1})$ and $\varphi_{x_{0}}(u_{2})$ by $x_{1}$ and $x_{2}$ respectively. 
	The associated difference in cost is 
	\begin{equation}\Stretch
	\begin{array}{rcl}
	\left|J_{x_0}(u_{1})-J_{x_0}(u_{2})\right|&=& \left|\int_{\left[0,\tau(u_{1})\right]} g\left(x_{1}(t),u_{1}(t)\right) \,\mu(dt) -\int_{\left[0,\tau(u_{2})\right]}g\left(x_{2}(t),u_{2}(t)\right)\, \mu(dt) \right| \\
	&\leq & \left|\int_{\left[0,\tau(u_{1})\right]}g\left(x_{1}(t),u_{1}(t)\right)\right.-g\left(x_{2}(t),u_{2}(t)\right)\, d\mu(t)\\
	&& -\left.\int_{\left[\tau(u_{1}),\tau(u_{2})\right]}g\left(x_{2}(t),u_{2}(t)\right)\, d\mu(t)\right|.\\
	&\leq&\int_{\left[0,\tau(u_{1})\right]} \left| g\left(x_{1}(t),u_{1}(t)\right)-g\left(x_{2}(t),u_{2}(t)\right)\,\right| \mu(dt)\\
	&& \left. +\int_{\left[\tau(u_{1}),\tau(u_{2})\right]} \left| g\left(x_{2}(t),u_{2}(t)\right)\, \right| \mu(dt).\right.
	\end{array}
	\end{equation}
	This is further bounded using the the Lipschitz constant of $g$ (cf. A-3)
	\begin{equation}\Stretch
	\begin{array}{rcl}
	\left|J_{x_0}(u_{1})-J_{x_0}(u_{2})\right|
	& \leq & \int_{\left[0,\tau(u_{1})\right]} L_{g}\left\Vert x_{1}(t)-x_{2}(t)\right\Vert _{2}+L_{g}\left\Vert u_{1}(t)-u_{2}(t)\right\Vert _{2}\, \mu(dt) \\
	&  & +\int_{\left[\tau(u_{1}),\tau(u_{2})\right]}g\left(x_{2}(t),u_{2}(t)\right)\, \mu(dt)\\
	& \leq & L_{g}\tau(u_{1}) d_{\mathcal{X}}(x_1,x_2)+ L_{g}d_{\mathcal{U}}(u_1,u_2)\\ 
	& &  +\int_{\left[\tau(u_{1}),\tau(u_{2})\right]}g\left(x_{2}(t),u_{2}(t)\right)\, \mu(dt). 
	\end{array}\label{eq:cost_diff}
	\end{equation}
	Since $x_2$ is continuous, $u_2$ is bounded, and $g$ is continuous, there exists a bound $G$ on $g(x_2(t),u_2(t))$ for $t\in [\tau(u_1),\tau(u_2)]$. Thus, the difference in cost is further bounded by  
	
	\begin{equation}
	\begin{array}{rcl}
	\left|J_{x_0}(u_{1})-J_{x_0}(u_{2})\right| & \leq & L_{g}\tau(u_{1}) d_{\mathcal{X}}(x_1,x_2) + L_{g}d_{\mathcal{U}}(u_1,u_2) +G|\tau(u_2)-\tau(u_1)|. 
	\end{array}
	\label{eq:cost_diff2}
	\end{equation}
	Observe that $|\tau(u_2)-\tau(u_1)|<d_\U(u_1,u_2)/u_{max}$ so the term $G|\tau(u_1)-\tau(u_2)|$ can be made arbitrarily small with $d_\U(u_1,u_2)$ sufficiently small. 
	Similarly, since $\varphi_{x_0}$ is continuous, $d_\X(x_1,x_2)$ is made arbitrarily small with $d_\U(u_1,u_2)$ sufficiently small. 
	Thus, for any $\varepsilon>0$, there exists a $\delta>0$ such that $d_\U(u_1,u_2)<\delta$ implies $|J_{x_0}(u_1)-J_{x_0}(u_2)<\varepsilon|$ which implies continuity by Lemma \ref{app:ep_delt_cont}. 
\end{proof} 
The cost functional is continuous with respect to the initial condition parameter as well. 
In fact, it is Lipschitz continuous  which is used when comparing the cost of two trajectories to check if one of the trajectories is provably suboptimal.
\begin{lemma}\label{lem:cost_sensitivity}
	For any $u\in \U$ and
	$x_{0},z_{0}\in\mathbb{R}^{n}$, 
	\begin{equation}
	|J_{x_{0}}(u)-J_{z_{0}}(u)|\leq\Vert x_{0}-z_{0}\Vert_{2}\cdot\frac{L_{g}}{L_{f}}\left(e^{L_{f}\tau(u)}-1\right)\label{eq:cost_sensitivity}
	\end{equation}
\end{lemma}
\begin{proof}
	The absolute difference is bounded using the Lipschitz continuity of $g$. This is further bounded using (\ref{lem:cont_ic}). Denoting $x(t)=[\varphi_{x_{0}}(u)](t)$ and $z(t)=[\varphi_{z_{0}}(u)](t)$,
	\begin{equation}\Stretch
	\begin{array}{rcl}
	|J_{x_{0}}(u)-J_{z_{0}}(u)| & = & \left|\int_{[0,\tau(u)]}g(x(t),u(t))\right. \left.-g(z(t),u(t))\, d\mu(t) \right|\\
	& \leq & \int_{[0,\tau(u)]}\left|g(x(t),u(t))\right. \left.-g(z(t),u(t))\right|\, \mu(dt)\\
	& \leq & \int_{[0,\tau(u)]}L_{g}\left\Vert x(t)-z(t)\right\Vert _{2}\, \mu(dt)\\
	& \leq & \int_{[0,\tau(u)]}\Vert x_{0}-z_{0}\Vert_{2} L_{g}e^{L_{f}t}\, \mu(dt)\\
	& = & \Vert x_{0}-z_{0}\Vert_{2}\frac{L_{g}}{Lf}\left(e^{L_{f}\tau(u)}-1\right).
	\end{array}
	\end{equation}
\end{proof}

\bibliographystyle{amsalpha}
\bibliography{chap3}
\chapter{Approximation of the Space of Input Signals}\label{chap:approximation}
\begin{wrapfigure}{R}{0.5\textwidth}
	\frame{\includegraphics[width=0.5\textwidth]{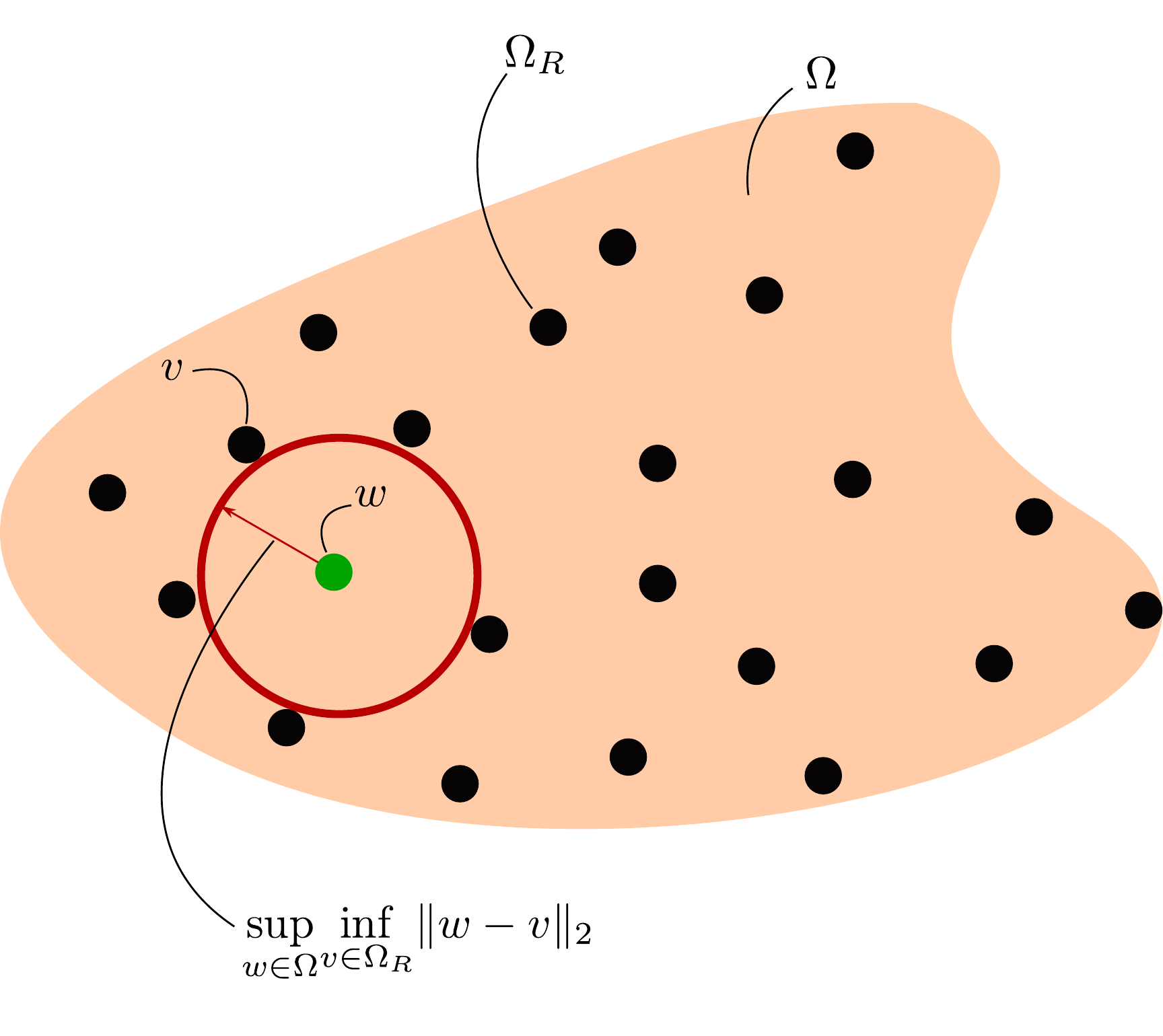}}
	\caption{The set $\Omega$ is illustrated in light red with the subset $\Omega_R$ indicated by small black markers. The dispersion of $\Omega_R$ in $\Omega$ is illustrated by the radius of the dark red circle. It is the radius of the largest ball centered at a point in $\Omega$ which does not intersect $\Omega_R$.}\label{fig:dispersion}
\end{wrapfigure}
The goal of this thesis is to present an algorithm which computes an approximately optimal solution to Problem \ref{Problem}, and, since the decision variable is selected from \U, an uncountably infinite space, we will use a countably infinite approximation which can be systematically searched to provide an approximate solution with arbitrary accuracy in finite time.

The natural number $R$ will denote the resolution of the approximation, and the signal space $\mathcal{U}$ is approximated by a subset $\mathcal{U}_{R}$ indexed by the resolution. 
The approximation $\UR$ of \U is constructed from strings of a finite collection of primitive input signals.
The primitive input signals each take a constant value from a finite subset $\Omega_{R}$ of $\Omega$ for a fixed duration.
The approximation of the control inputs $\Omega$ can be any sequence of subsets whose \emph{dispersion}\index{dispersion} converges to zero in $\Omega$. That is,
\begin{equation}\label{eq:disp_conv}
\lim_{R\rightarrow\infty} \left(\sup_{w\in \Omega} \inf_{v\in \Omega_R} \Vert w-v\Vert_2 \right)=0.
\end{equation}
The dispersion of a subset within a set is a measure of how well the subset approximates its containing set. Figure \ref{fig:dispersion} illustrates the definition of dispersion appearing in equation \eqref{eq:disp_conv}.
A family of subsets $\Omega_R$ exists and is often easily obtained with regular grids or random sampling for a given $\Omega$.   

The primitive input signal $u:[0,1/R]\rightarrow\Omega$ associated to each $w$ in $\Omega_R$ is the following:
\begin{equation}
u(t)=w,\qquad \forall t\in[0,1/R].
\end{equation}
The duration $1/R$ of the primitive input signals is selected for convenience. 
In general, all that is required is that the duration converges to zero with increasing $R$.
For example, the duration of primitive input signals could alternatively be $1/R^2$ or $1/\log(R)$ with the appropriate adjustments made to the remaining analysis.
\begin{wrapfigure}{L}{0.5\textwidth}
	\frame{\includegraphics[width=0.5\textwidth]{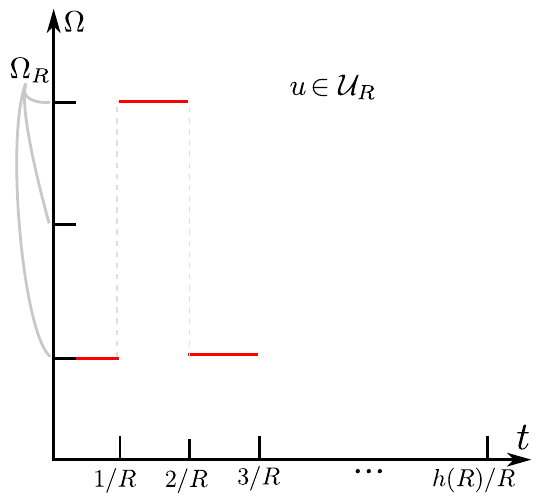}}
	\caption{Signals from \UR are piecewise constant taking values in $\Omega_R$ on intervals of duration $1/R$. The duration is variable as long as the number of constant segments is less than $h(R)$. The duration of the longest input signal in \UR is $h(R)/R$ which grows unbounded when $h(R)$ satisfies \eqref{eq:horizon}.}\label{fig:UR}
\end{wrapfigure}

The approximation \UR is then all possible strings of input signals with length $h(R)$ or less where $h:\mathbb{N}\rightarrow\mathbb{N}$ is a horizon or depth limit. 
That is, an input signal in \UR will have a time domain $[0,d/R]$ for $d\leq h(R)$, and on each interval $\left[\frac{(i-1)}{R},\frac{i}{R}\right)$ with $i\in{1,...,d}$, the input signal takes one value from $\Omega_R$. 
Figure \ref{fig:UR} illustrates the construction of signals in \UR.
The horizon limit can be any function satisfying 
\begin{equation}\label{eq:horizon}
\lim_{R\rightarrow\infty} \frac{R}{h(R)}=0. 
\end{equation}  
This ensures that with sufficiently high resolution, the time domain of any input signal in \U can be approximated by the time domain of a signal in \UR. 
Figure \ref{fig:approx_cartoon} shows how increasing the resolution improves the available approximations in \UR to a particular input signal.
\begin{figure}
	\centering
	\includegraphics[width=1.0\textwidth]{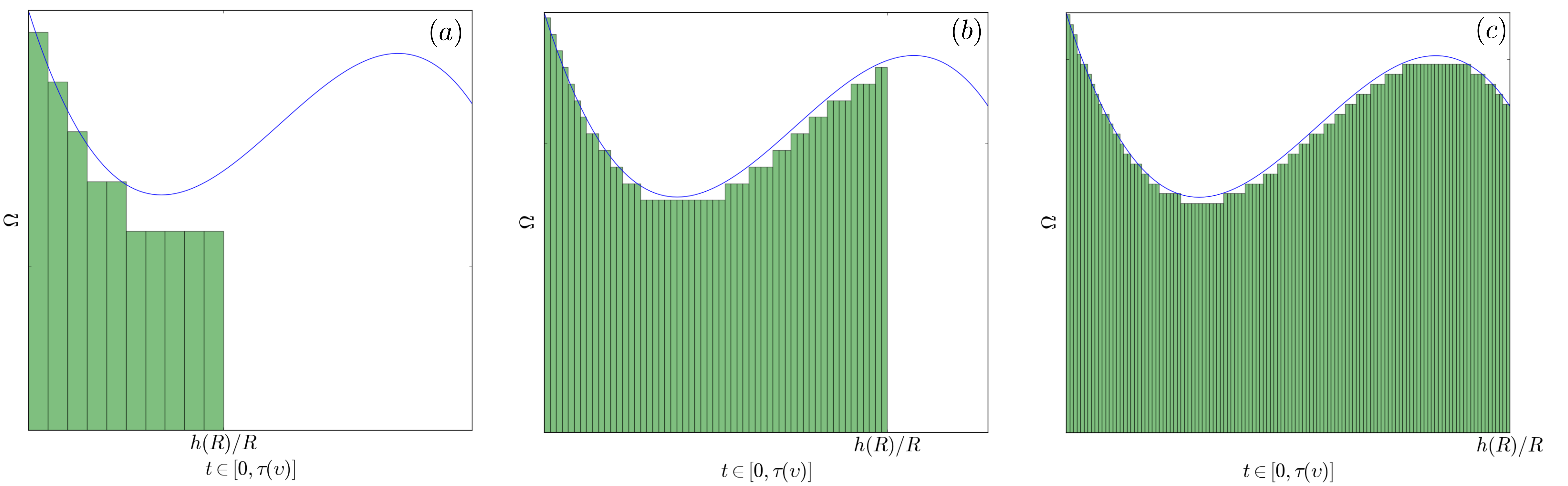}
	\caption{The three tiles illustrate how, with increasing resolution (left to right), there is a signal in \UR that approximates any $\upsilon$ (blue) in \U with increasing accuracy. The vertical axis represents the, generally multi-dimensional, control input space $\Omega$ and the horizontal axis is time.}\label{fig:approx_cartoon}
\end{figure}
\paragraph{Terminology related to \UR:}
To simplify the discussion of how different input signals relate to one another, it will be useful to introduce some terminology.
A \textit{parent} of an input signal $w\in\mathcal{U}_{R}$ with domain $\left[0,i/R\right)$ is defined as the input signal $u\in\mathcal{U}_{R}$ with domain $\left[0,(i-1)/R\right)$ such that $w(t)=u(t)$ for all $t\in\left[0,\frac{c\cdot(i-1)}{R}\right)$.
In this case, $w$ is a \textit{child} of $u$. 
Two signals are \emph{siblings} if they have the same parent.
A tree (graph) is defined for each natural number $R$, with $\mathcal{U}_{R}$ as the vertex set, and edges defined by ordered pairs of signals $(u,w)$ such that $u$ is the parent of $w$.
To serve as the root of the tree, $\mathcal{U}_R$ is augmented with the special input signal $Id_{\mathcal{U}}$ defined such that $J_{x_0}(Id_{\mathcal{U}})=0$ and $ \tau(Id_\U)=0$. $Id_{\mathcal{U}}$ has no parent, but is the parent of signals with domain $[0,1/R]$. 

The signal $w$ is an \textit{ancestor} of $u$ if $\tau(w)\leq\tau(u)$ and $w(t)=u(t)$ for all $t\in[0,\tau(w))$.
In this case $u$ is a \textit{descendant} of $w$. The \textit{depth} of an input signal in $\mathcal{U}_{R}$ is the number of ancestors of that input signal.

\begin{remark}
	With these conventions, each signal is an ancestor, descendant, and sibling of itself.
\end{remark}
\section{Consistency of the Approximation}\label{sec:consistency} 
Since there is a metric on \U, a natural requirement of any approximation of this set by some subset is that for any fixed $u\in\U$, 
\begin{equation}
\lim_{R\rightarrow \infty} \left(\min_{w\in\UR} d_\U (u,w)\right)=0.
\end{equation}
The interpretation of this condition is as follows: given any fixed\footnote{Note that this is a weaker condition than requiring the dispersion of \UR in \U to converge to zero.} input signal $u$ and $\varepsilon>0$, there is an input signal in \UR whose distance from $u$ is less than $\varepsilon$ with a sufficiently high resolution (i.e. sufficiently large $R$). 

The next result establishes this property for \UR.
\begin{lemma}
	\label{lem:density}For each $u\in\mathcal{U}$ and $\varepsilon>0$,
	there exists $R^{*}>0$ such that for any $R>R^{*}$ there exists
	$w\in\mathcal{U}_{R}$ satisfying $d_{\mathcal{U}}(u,w)<\varepsilon$.\end{lemma}
\begin{proof}
	The proof will first rely on Lusin's theorem\footnote{Lusin's Theorem is discussed in Appendix \ref{sec:measure}.} for the existence of a continuous  input signal approximating $u$.
	We then approximate the continuous input  signal $u$ with an input signal $w$ in \UR.
	If follows directly from Lusin's Theorem~\cite{lusin1912proprietes,feldman1981proof} that there exists a continuous input signal $\upsilon:[0,\tau(u)]\rightarrow\Omega$ such that 
	\begin{equation}\label{eq:lusin}
	\mu(\{t\in[0,\tau(u)]:\:\upsilon(t)\neq u(t)\})<\frac{\varepsilon}{4u_{max}}.
	\end{equation}
	Since $u$ and $\upsilon$ have the same time domain, the distance between the two is given by,
	\begin{equation}
		d_\U(u,\upsilon)=\int_{[0,\tau(u)]}\Vert u(t)-\upsilon(t) \Vert_2 \,\mu(dt).
	\end{equation}
	Then from \eqref{eq:lusin} this is equal to 
	\begin{equation}
	d_\U(u,\upsilon)=\int_{\{t\in[0,\tau(u)]:\:\upsilon(t)\neq u(t)\}}\Vert u(t)-\upsilon(t) \Vert_2 \,\mu(dt).
	\end{equation}
	Since $\Omega$ is bounded, $\Vert u(t)-\upsilon(t)\Vert_2<2u_{max}$ resulting in the bound
	\begin{equation}
	d_\U(u,\upsilon)<\int_{\{t\in[0,\tau(u)]:\:\upsilon(t)\neq u(t)\}}2 u_{max} \,\mu(dt)=\frac{\varepsilon}{2}.
	\end{equation} 
	
	Next, we will construct a $w$ in \UR such that $d_\U(\upsilon,w)$ is small.
	The domain of $\upsilon$ is compact so the continuity is also uniform. 
	Let $\delta(\epsilon)$ denote the modulus of continuity. That is,
	\begin{equation}
	|\sigma-\gamma|<\delta(\epsilon)\Rightarrow\Vert\upsilon(\sigma)-\upsilon(\gamma)\Vert_{2}<\epsilon.
	\end{equation}
	To construct an approximation of $\upsilon$ by $w\in\mathcal{U}_{R}$
	choose $R$ sufficiently large so that
	\begin{enumerate}
		\item $\tau(\upsilon)<h(R)/R$,
		\item $1/R<\min\left\{ \frac{\varepsilon}{6u_{max}}, \delta\left(\frac{\varepsilon}{6\tau(\upsilon)}\right)\right\}  $,
		\item there exists an integer $r<h(R)$ such that $0<\tau(\upsilon)-r/R<1/R$,
		\item the dispersion of $\Omega_{R}$ in $\Omega$ is less than $\frac{\varepsilon}{6\tau(\upsilon)}$.
	\end{enumerate}
	Note that if these conditions hold for $R^*$, they hold for all $R>R^*$.
	
	It follows from (2.) and (4.) above, and  the uniform continuity of $\upsilon$ that for each $i\in\{1,2,...,r\}$ and $t\in[(i-1)/R,i/R)$, there exists ${\rm v}_{i}\in\Omega_{R}$ such that $\Vert {\rm v}_{i}-\upsilon(t)\Vert_{2}<\frac{\varepsilon}{3\tau(\upsilon)}$.
	Select $w\in\mathcal{U}_{R}$ which is equal to ${\rm v}_{i}$ on
	each of these intervals so that $\tau(w)=r/R$, and $\tau(w)<\tau(\upsilon)$. 
	By construction of $w$ we have
	\begin{equation}
		d_{\mathcal{U}}(\upsilon,w)  = \int_{[0,r/R]}\Vert w(t)-\upsilon(t)\Vert_{2}\, \mu(dt)+u_{max}|r/R-\tau(\upsilon)|.
	\end{equation}
	Then by condition (2.) and (3.),
	\begin{equation}
	d_{\mathcal{U}}(\upsilon,w)  < \int_{[0,r/R]}\Vert w(t)-\upsilon(t)\Vert_{2}\, \mu(dt)+\frac{\varepsilon}{6}.
	\end{equation} 
	Next, by construction of $w(t)$ we have 
	\begin{equation}
	d_{\mathcal{U}}(\upsilon,w)  < \int_{[0,r/R]}\frac{\varepsilon}{3\tau(u)}\, \mu(dt)+\frac{\varepsilon}{6}.
	\end{equation}
	Recall that $r/R<\tau(\upsilon)$ so that
	\begin{equation}
	d_{\mathcal{U}}(\upsilon,w)  < \int_{[0,\tau(\upsilon]}\frac{\varepsilon}{3\tau(u)}\, \mu(dt)+\frac{\varepsilon}{6}.
	\end{equation}
	Integrating the right hand side of the inequality yields
	\begin{equation}
	d_{\mathcal{U}}(\upsilon,w)  < \frac{\varepsilon}{3}+\frac{\varepsilon}{6}=\frac{\varepsilon}{2}.
	\end{equation}
	
	Thus, by the triangle inequality
	\begin{equation}
	d_{\mathcal{U}}(u,w)\leq d_{\mathcal{U}}(u,\upsilon)+d_\U(\upsilon,w)<\varepsilon.
	\end{equation}
\end{proof}
%
%
%

%
The previous result showed that \UR consistently approximates \U with respect to the metric $d_\U$. 
However, we are primarily interested in approximating \Ugoal. 
Since \Ugoal is open with respect to $d_\U$, the set $\UR \cap \Ugoal$\footnote{This is the set of signals within the approximation of \U resulting in trajectories that satisfy obstacle avoidance constraints and terminal constraints.} will inherit this property.
We will combine this observation with the continuity of $J_{x_{ic}}$ to show that there exist signals in $\UR \cap \Ugoal$ with cost arbitrarily close to the optimal cost:   
\begin{theorem}\label{thm:approx}
	\label{lem:no_limit_points}For any $u\in cl\left(int\left(\mathcal{U}_\mathrm{goal}\right)\right)$
	and $\varepsilon>0$ there exists $R^{*}>0$ such that for any $R>R^{*}$
	\begin{equation}
	\min_{w\in\mathcal{U}_{R}\cap\mathcal{U}_\mathrm{goal}}\left\{ \left|J_{x_\mathrm{ic}}(u)-J_{x_\mathrm{ic}}(w)\right|\right\} <\varepsilon.
	\end{equation}
\end{theorem}
\begin{proof}
	Let $u$ be a signal in  $cl\left(int\left(\mathcal{U}_\mathrm{goal}\right)\right)$.
	Then every neighborhood of $u$ has a nonempty intersection with $int(\Ugoal)$. Equivalently, for every $\delta>0$, 
	\begin{equation}
		B_{\delta/2}(u)\cap int(\Ugoal)\neq \emptyset.
	\end{equation} 
	Let $\upsilon$ be an element of this intersection so that $\upsilon\in 	B_{\delta/2}(u)$ and $\upsilon \in int(\Ugoal)$. 
	Since $\upsilon$ is an interior point of \Ugoal, there exists $\rho>0$ such that 
	$B_{\rho}(\upsilon)$ is a subset of $int(\Ugoal)$.
	Now consider the neighborhood $B_{\min\{\rho,\delta/2\}}(\upsilon)$, 	which is a subset of $int(\Ugoal)$. 
	By Lemma \ref{lem:density}, there exists an $R^{*}$ such that for all $R>R^*$, there is a signal $u$ in \UR satisfying $d_\U(w,\upsilon)<\min\{\rho,\delta/2\}$
    which implies $w$ is in \Ugoal. 
    Further,
    \begin{equation}
    \begin{array}{rcl}
	    d_\U(w,u)&\leq& d_\U(w,\upsilon)+d_\U(\upsilon,u)\\
	    & \leq & \min\{\rho,\delta/2\} + \delta/2 \\
	    & \leq & \delta.
	\end{array}
    \end{equation}
	Now by the continuity of $J_{x_\mathrm{ic}}$, for $\delta$ sufficiently small, $d_{\mathcal{U}}(u,\upsilon)<\delta$ implies $|J_{x_\mathrm{ic}}(w)-J_{x_\mathrm{ic}}(u)| <\varepsilon$ from which the result follows.
\end{proof}
A sufficient condition for every $u\in\mathcal{U}_\mathrm{goal}$ to
be contained in the closure of the interior of $\mathcal{U}_\mathrm{goal}$ is that $\mathcal{U}_\mathrm{goal}$ be open which is the case when Assumption A-1 is satisfied. 

\bibliographystyle{amsalpha}
\bibliography{chap4}
\chapter{The Generalized Label Correcting Method}\label{chap:glc}
Equipped with the set \UR, one could enumerate the strings of input signal primitives in \UR to find the minimum cost input signal in $\UR\cap \Ugoal$. 
Theorem \ref{thm:approx} tells us that this procedure would solve Problem \ref{Problem}.
However the number of signals in \UR is $h(R)^{{\rm card}(\Omega_R)}$ which grows rapidly as the dispersion of the set of control inputs $\Omega_R$ in $\Omega$ converges to zero.
This type of exhaustive search is analogous to searching over all paths in a graph for a shortest path between two vertices.
However, this is an inefficient approach which is remedied with label correcting algorithms.

\paragraph{Label correcting algorithmms:}
%
In a conventional \emph{label correcting method}\index{label correcting method}, the algorithm maintains the least cost path $p_{\rm label}$ known to terminate at each vertex of the graph. This path \textit{labels} that vertex. 
At a particular iteration, if a path $p_{\rm new}$ under consideration does not have lower cost than the path labeling the terminal vertex, the path under consideration is discarded.
Justification for this operation is that any extension $p_{\rm tail}$ of the path $p_{\rm new}$ which reaches the goal vertex can be also be used to extend $p_{\rm label}$ to reach the goal vertex. The cost of the extension $p_{\rm tail}$ in both cases will be the same, but the total cost of $p_{\rm new}$ concatenated with $p_{\rm tail}$ will certainly be no less than the total cost of $p_{\rm label}$ concatenated with $p_{\rm tail}$. 
Therefore, there is no need to consider any extensions of $p_{\rm new}$.  
As a consequence, the subtree of paths originating from the discarded path will not be evaluated.

%

\paragraph{Generalizing the notion of a label:}
Observe that the label of a vertex in conventional label correcting algorithms is in fact a label for the paths terminating at that vertex.
Then each vertex identifies an equivalence class of paths which terminate at that vertex.
Paths within each equivalence class are ordered by their cost, and the efficiency of label correcting methods comes from narrowing the search to minimum cost paths in their associated equivalence class. 
The generalization is to identify paths associated to trajectories terminating in the same region of the state space instead of the same state.
In the context of trajectory planning, this is complicated by the fact that the optimal cost from any particular state to the goal may be discontinuous. This makes it difficult to prove that a trajectory terminating close to another trajectory with lesser cost can be discarded. 
In this chapter we derive precise conditions for discarding a trajectory that resolves this issue.

\section{Constructing Equivalence Classes of Signals}\label{sec:glc_def}
The equivalence classes of input signals are induced by a partition of \xfree.
A partition of \xfree is said to have radius $r$ if each element of the partition is a set contained in a neighborhood of radius $r$. 
No further assumptions on the geometry of these sets will be required. 
A partition of radius $r$ is illustrated in Figure \ref{fig:partition}
\begin{wrapfigure}{R}{0.4\textwidth}
	\frame{\includegraphics[width=0.4\textwidth]{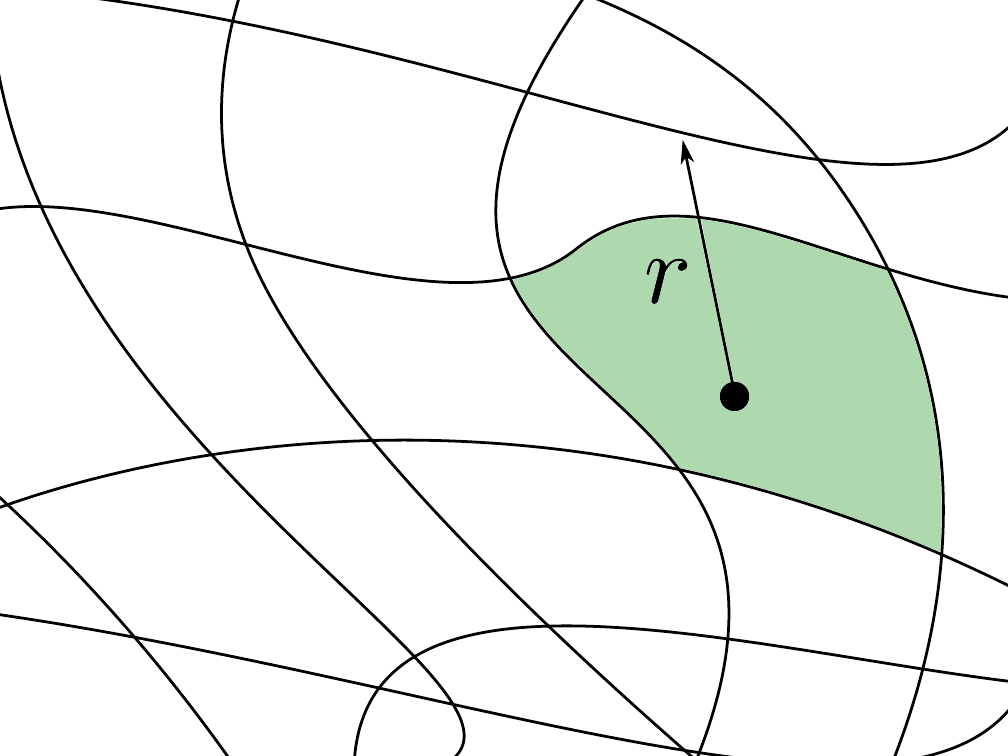}}
	\caption{The green region shows one element of a partition of $\mathbb{R}^2$. This set is contained in a ball of radius $r$ as are all other elements of the partition.}\label{fig:partition}
\end{wrapfigure}

For now we only consider hypercube partitions whose radius is controlled by a function $\eta:\mathbb{N}\rightarrow\mathbb{R}_{>0}$.
%
%
For states $p_1,p_2\in\mathbb{R}^n$ we write $p_1\overset{R}{\sim}p_2$ if
\begin{equation}
\left\lfloor \eta(R)p_1\right\rfloor =\left\lfloor \eta(R)p_2\right\rfloor,
\end{equation}
where $\left\lfloor \cdot\right\rfloor $ is the coordinate-wise floor map (e.g. $\left\lfloor (2.9,3.2) \right\rfloor = (2,3)$). 
The equivalence classes of the $\overset{R}\sim$ relation define a simple hypercube partition of radius $\sqrt{n}/\eta(R)$.    
We extend this relation to control inputs by comparing the terminal state of the resulting trajectory. 
For $u_1,u_2 \in \mathcal{U}_R$ we write $u_{1}\overset{\mathcal{U}_R}{\sim}u_{2}$ if the resulting trajectories terminate in the same hypercube. That is,
\begin{equation}\label{eq:equiv_signal}
u_{1}\overset{\mathcal{U}_R}{\sim}u_{2} \Leftrightarrow [\varphi_{x_\mathrm{ic}}(u_{1})](\tau(u_{1})) \overset{R}{\sim}  [\varphi_{x_\mathrm{ic}}(u_{2})](\tau(u_{2})) .
\end{equation}
Figure \ref{fig:intuition} illustrates the intuition behind this equivalence relation.
\begin{figure}
	\centering
	\frame{\includegraphics[width=1.0\textwidth]{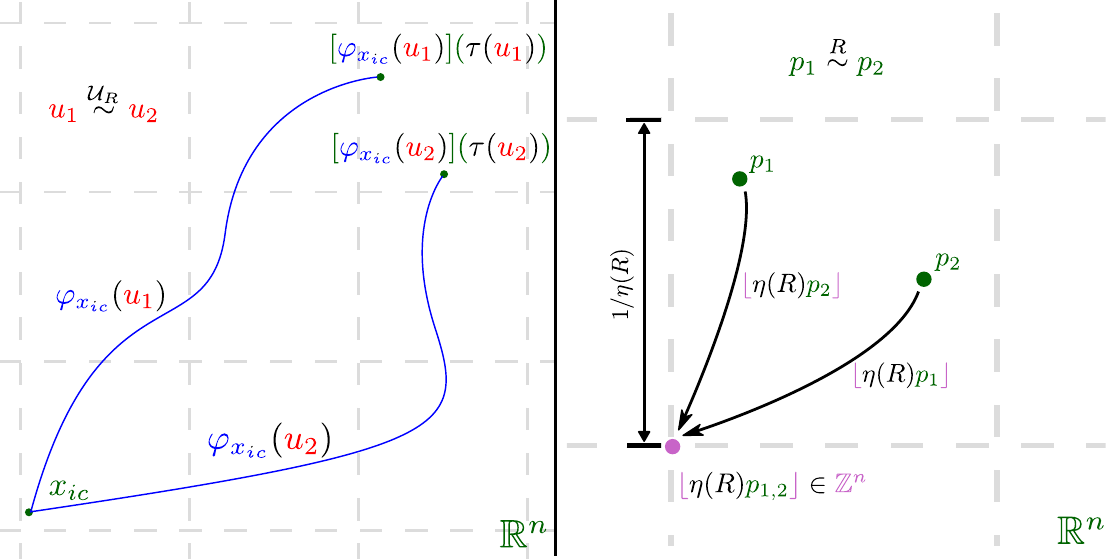}}
	\caption{Illustration of how a partition of \xfree induces a partition on \U. Input signals $u_1$,$u_2$ result in trajectories terminating within the same element of the partition of \xfree. This is the generalization of identifying paths in a graph terminating at the same vertex as equivalent.}\label{fig:intuition}
\end{figure}

The principal contribution of this thesis is the  \emph{GLC conditions}\index{GLC conditions} for ordering input signals among the equivalence classes of the relation $\overset{\mathcal{U}_R}\sim$.
The \GLC conditions define a partial ordering $\prec_{R}$ on $\mathcal{U}_{R}$ which is used to identify signals which can be discarded. 
We write $u_{1}\prec_{R}u_{2}$ if:
\begin{enumerate}[label={GLC-\arabic*},itemindent=0.75cm] 
	\item \label{glc1} $u_{1}\overset{\mathcal{U}_R}\sim u_{2}$,
	\item \label{glc2} $\tau(u_{1})\leq\tau(u_{2}),$
	\item \label{glc3} $J_{x_\mathrm{ic}}(u_{1}) \leq J_{x_\mathrm{ic}}(u_{2}).$
\end{enumerate}
That is, $u_1$ is $\prec_R$-less than $u_2$ if they result in trajectories terminating within the same region of the states space (\ref{glc1}), the duration of $u_1$ is no greater than the duration of $u_2$ (\ref{glc2}), and the cost of $u_1$ is less than the cost of $u_2$ (\ref{glc3}).
These conditions are sharper than the ones presented in~\cite{paden2016generalized} where \ref{glc3} included a threshold for the difference in cost between $u_1$ and $u_2$. 
A signal $u_1$ is called \emph{minimal}\index{minimal (partial order)} if there is no $u_2\in \mathcal{U}_R$ such that $u_{2}\prec_{R}u_{1}$. Such a signal can be thought of as being a good candidate for later expansion during the search. Otherwise, it can be discarded.
In order for the \GLC method to be a resolution complete algorithm, the scaling parameter $\eta$ must satisfy  
\begin{equation}
\lim_{R\rightarrow\infty}\frac{R}{L_f\eta(R)}\left( e^{\frac{L_{f}h(R)}{R}}-1\right)=0.\label{eq:partition_scaling}
\end{equation}
%
%
%

%
%
%

%
%
%
%
%
%

\begin{figure}
	\centering
	\includegraphics[width=0.8\textwidth]{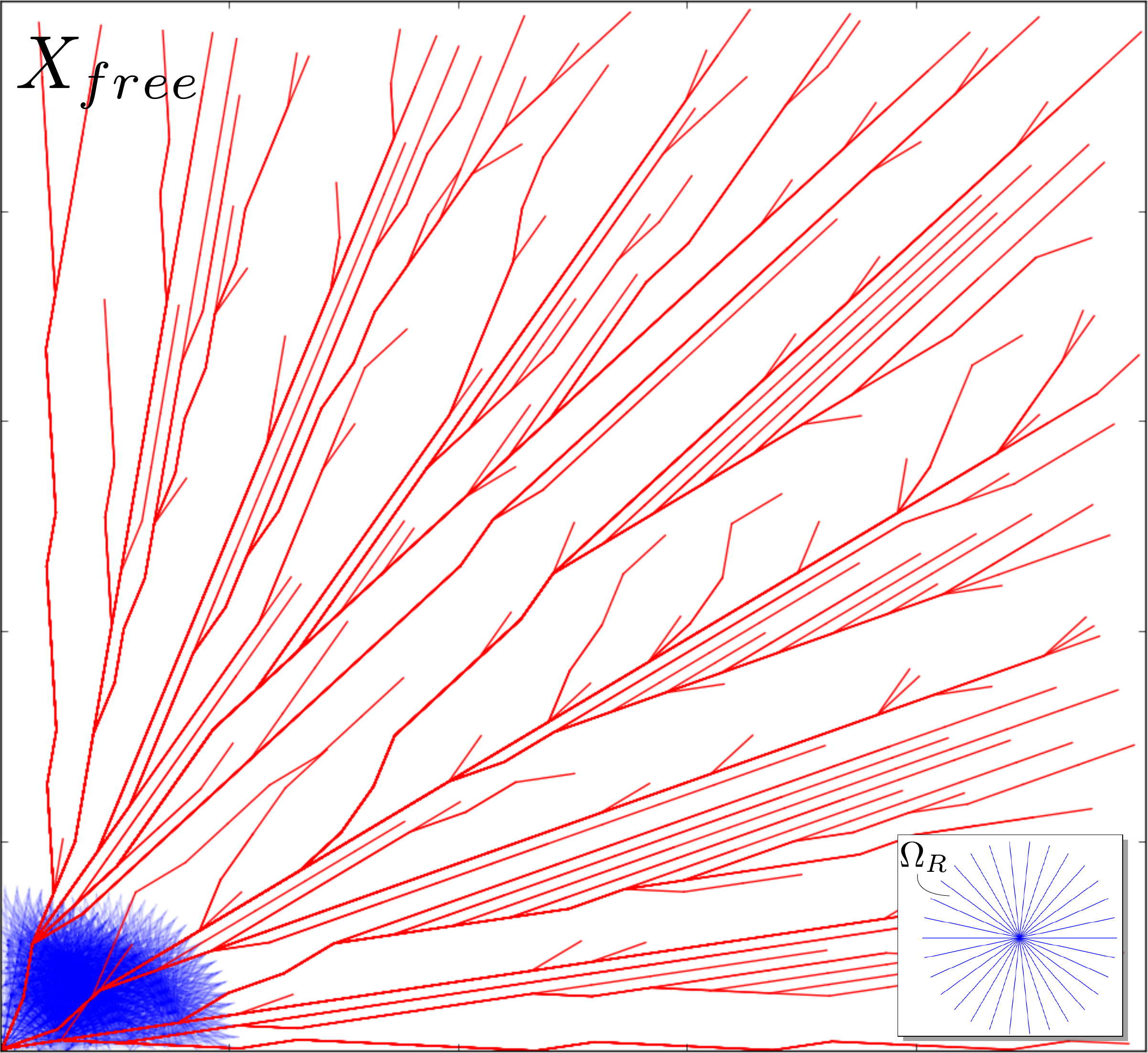}
	\caption{The motions resulting from the primitive control inputs $\Omega_R$ are shown in the bottom right corner. Executing the first $3000$ out of $30^{30}$ signals from \UR in uniform cost order (concatenation of 2-3 primitive  input signals) results in the blue trajectories illustrated which do not effectively cover \xfree.  On the other hand there are only 417 minimal trajectories in \UR which cover \xfree effectively with approximately optimal paths.}\label{fig:glc_demo}
\end{figure}
Figure \ref{fig:glc_demo} contrasts the trajectories resulting from minimal signals in \UR with an exhaustive enumeration of \UR for a two dimensional kinematic point robot. 
%
%
The minimal signals result in approximately optimal trajectories that uniformly cover the free space \xfree while exhaustively enumerating \UR requires many more trajectories to be evaluated. 

\section{Properties of the \GLC Conditions}
This section will develop a number of important concepts and basic properties of the partial order $\prec_R$.

The $\epsilon$-interior of the set \Xfree and \Xgoal, are defined by
\begin{equation}\Stretch
\begin{array}{rcl}
\Xfree^\epsilon&\coloneqq&\{x\in\Xic:\,B_\epsilon (x)\subset \Xfree\},\\
\Xgoal^\epsilon&\coloneqq&\{x\in\Xfree^\epsilon:\,B_\epsilon (x)\subset \Xgoal\}.
\end{array}
\end{equation} 
The inverse image of these sets under the system map $\varphi_{x_{ic}}$ is denoted
\begin{equation}
\Ufree^\epsilon \coloneqq \varphi_{x_{ic}}^{-1}(\Xfree^\epsilon),\qquad \Ugoal^\epsilon \coloneqq \varphi_{x_{ic}}^{-1}(\Xgoal^\epsilon).
\end{equation} 
Note that if $x$ is an element of $\Xfree^\epsilon$, then $B_{\epsilon}(x(t))\subset \xfree$. 
The same is true for \Xgoal.

To simplify notation we will denote the optimal cost of signals in $\Ugoal^\epsilon\cap\UR$ by $c_R^\epsilon$,
\begin{equation}\label{eq:cr_ep}
c_R^{\epsilon}\coloneqq \min_{u\in\UR\cap \Ugoal^\epsilon} \{J_{x_{ic}}(u)\}.
\end{equation}
Like, the optimal cost $c^*$ and the approximate optimal cost $c_R$, $c_R^\epsilon$ can, in some cases, be $\infty$.
An intuitive, but important property regarding the cost $c_R^\varepsilon$ is that it converges to the optimal cost $c^*$ as $\varepsilon$ tends to $0$ and $R$ tends to $\infty$. 
\begin{lemma}
	\label{lem:approx_equal_optimal}If $\lim_{R\rightarrow\infty}\epsilon(R)=0$,
	then $\lim_{R\rightarrow\infty}c_{R}^{\epsilon(R)}=c^{*}$.\end{lemma}
\begin{proof}
	Since the infimum in \eqref{eq:meaningful_problem} may not be attained, the first step is to identify an input signal which is arbitrarily close to the optimal cost.
	By the definition of $c^{*}$ in (\ref{eq:meaningful_problem}), for any
	$\varepsilon>0$ there exists $\omega\in\mathcal{U}_\mathrm{goal}$ such
	that $J_{x_\mathrm{ic}}(\omega)-\varepsilon/2<c^{*}$. 
	
	Next, we use the topological properties discussed in Chapter \ref{chap:topo} to construct a neighborhood of $\omega$ containing signals in \Ugoal.
	Since $\mathcal{U}_\mathrm{goal}$ is open and $\varphi_{x_\mathrm{ic}}$ is continuous, there exists $\tilde{r}>0$ and $\rho>0$
	such that $B_{\tilde{r}}(\omega)\subset\mathcal{U}_\mathrm{goal}$ and $\varphi_{x_\mathrm{ic}}(B_{\tilde{r}}(\omega))\subset B_{\rho}(\varphi_{x_\mathrm{ic}}(\omega))$.
	Thus, $\omega\in\mathcal{U}_\mathrm{goal}^{\rho}$. 
	From the continuity of $J$ in Lemma \ref{lem:cost_cont} there also exists a positive
	$\delta<r$ such that for any signal $\upsilon$ with $d_{\mathcal{U}}(\upsilon,\omega)<\delta$
	we have $|J_{x_\mathrm{ic}}(\omega)-J_{x_\mathrm{ic}}(\upsilon)|<\varepsilon/2$.
	
	The next step is to find a feasible signal in \UR which closely approximates $\upsilon$.
	Choose $R^{*}$ to be sufficiently large such that $R>R^*$ implies $\mbox{\ensuremath{\epsilon}}(R)<\rho$ and $B_{\delta}(\omega)\cap\mathcal{U}_{R}\neq\emptyset$. 
	Such a resolution $R^{*}$ exists by
	Theorem \ref{lem:no_limit_points} and the assumption $\lim_{R\rightarrow\infty}\epsilon(R)=0$.
	Now choose $u\in B_{\delta}(\omega)\cap\mathcal{U}_{R}$.
	Then $|J_{x_\mathrm{ic}}(u)-J_{x_\mathrm{ic}}(\omega)|<\varepsilon/2$ and $u\in\mathcal{U}_\mathrm{goal}^{\rho}\subset\mathcal{U}_\mathrm{goal}^{\epsilon(R)}$.
	
	Finally, we use the triangle inequality to show that $u$ has nearly the optimal cost.
	Then, by definition of $c_{R}^{\epsilon(R)}$, $u\in\mathcal{U}_\mathrm{goal}^{\epsilon(R)}$ implies $c_{R}^{\epsilon(R)}\leq J_{x_\mathrm{ic}}(u)$. 
	Finally, by the triangle inequality, 
	\begin{equation}
	|J_{x_\mathrm{ic}}(u)-c^{*}|<|J_{x_\mathrm{ic}}(u)-J_{x_\mathrm{ic}}(\omega)|+|J_{x_\mathrm{ic}}(\omega)-c^{*}|<\varepsilon.
	\end{equation}
	Rearranging the expression yields $J_{x_\mathrm{ic}}(u)<c^{*}+\varepsilon$
	and thus, $c_{R}^{\epsilon(R)}<c^{*}+\varepsilon$. The result follows
	since the choice of $\varepsilon$ is arbitrary.
\end{proof}

\paragraph{Some new notation:}
At this point it will be advantageous to introduce a concatenation operation on elements of $\mathcal{U}$ and $\mathcal{X}_{x_0}$.
For $s_{1},s_{2}\in\U\cup\Xxo$, their concatenation $s_{1}s_{2}$
is defined by 
\begin{equation}
[s_{1}s_{2}](t)\coloneqq\left\{ \begin{array}{c}
s_{1}(t),\, t\in[0,\tau(s_{1}))\\
s_{2}(t-\tau(s_{1})),\, t\in[\tau(s_{1}),\tau(s_{1})+\tau(s_{2})]
\end{array}\right..\label{eq:concatenation}
\end{equation}
The concatenation is illustrated in Figure \ref{fig:concatenation}.
\begin{figure}[h]
	\centering
	\includegraphics[width=0.95\textwidth]{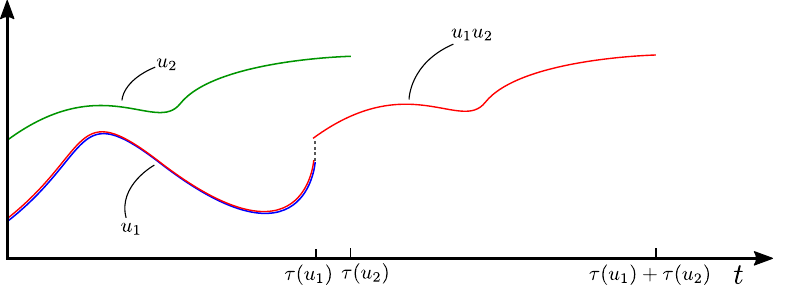}
	\caption{Signals $u_1$ and $u_2$ are illustrated in blue and green respectively. Their concatenation $u_1u_2$ is illustrated in red.}\label{fig:concatenation}
\end{figure}
The concatenation operation will be useful together
with the following equalities for $u_1,u_2\in\U$, 
\begin{equation}\label{eq:cost_homo}
J_{x_{0}}(u_{1}u_{2})=J_{x_{0}}(u_{1})+J_{[\varphi_{x_0}(u_1)](\tau(u_1))}(u_{2})
\end{equation}
\begin{equation}\label{eq:flow_homo}
\ensuremath{\varphi_{x_{0}}(u_{1}u_{2})=\varphi_{x_{0}}(u_{1})\varphi_{[\varphi_{x_{0}}(u_1)](\tau(u_1))}(u_{2})}.
\end{equation}
Equations \eqref{eq:cost_homo} and \eqref{eq:flow_homo} follow directly from \eqref{eq:int_dynamics}.
The interpretation of \eqref{eq:cost_homo} is that the cost of $u_1$ concatenated with $u_2$ is equal to the cost of executing the control signal $u_1$ from $x_0$ plus the cost of executing $u_2$ from the terminal state of the trajectory resulting from $u_1$. 

The next result can be interpreted as a generalization of Bellman's \emph{principle of optimality}\index{principle of optimality}~\cite{bellman1956dynamic}  and is the basis for the generalized label correcting method in the same way that the principle of optimality is the basis for label correcting methods. 
Figure \ref{fig:princ_opt} illustrates the statement of the next theorem. 
\begin{theorem}[Principle of Optimality]
	\label{thm:pruning}
	Let $\gamma=\frac{\sqrt{n}}{\eta(R)} \frac{L_g}{L_f}\left( e^{L_f \left( \frac{h(R)}{R} -\tau(u_j) \right) }-1 \right)$, and
	$\delta=\frac{\sqrt{n}}{\eta(R)}e^{ L_{f} \left( \frac{h(R)}{R}-\tau(u_j) \right)}$.
	If $u_{i},u_{j}\in\UR\cap \Ufree$ and satisfy
	$u_{i}\prec_{R}u_{j}$, then for each descendant of $u_{j}$ in $\UR\cap\Ugoal^{\delta}$
	with cost $c_{j}$, there exists a descendant of $u_{i}$ in $\UR\cap\Ugoal$
	with cost $c_{i}\leq c_{j}+\gamma$.
\end{theorem}
\begin{proof}
	Suppose there exists $w\in\mathcal{U}_{R}$ such that $u_{j}w\in\mathcal{U}_{R}\cap \U^\delta$.
	By \ref{glc1} the signals satisfy $u_{i}\overset{\UR}{\sim}u_{j}$ which means 
	\begin{equation}\label{eq:close}
	\left\Vert[\varphi_{x_{ic}}(u_{i}w)](\tau(u_{i}))-[\varphi_{x_\mathrm{ic}}(u_{j}w)](\tau(u_{j}))\right\Vert_2\leq \frac{\sqrt{n}}{\eta(R)}.
	\end{equation}
	Note that $\tau(w)\leq h(R)/R-\tau(u_j)$ since $u_jw$ has depth no greater than $h(R)$.
	Then, by Lemma (\ref{lem:cont_ic}), for all $t\in[0,\tau(w)]$, 
	\begin{equation}\Stretch\label{eq:-8}
	\begin{array}{rcl}
	\left\Vert[\varphi_{x_\mathrm{ic}}(u_{i}w)](t+\tau(u_{i}))-[\varphi_{x_\mathrm{ic}}(u_{j}w)](t+\tau(u_{j}))\right\Vert&\leq&\frac{\sqrt{n}}{\eta(R)}e^{L_{f}\left(\frac{h(R)}{R}-\tau(u_j)\right)}\\
	&=&\delta.
	\end{array}
	\end{equation}
	Thus,  $u_{i} w \in \Ugoal.$ 
	Next, we show that $c_i\leq c_j+\gamma$. 
	From equation (\ref{eq:cost_homo}), 
	\begin{equation}\label{eq:use_concat}
			J_{x_\mathrm{ic}}(u_{i} w)  =  J_{x_\mathrm{ic}}(u_{i})+J_{[\varphi_{x_\mathrm{ic}}(u_{i})](\tau(u_{i}))}(w).
	\end{equation}
	Combining the continuity of the cost functional with respect to the initial condition (Lemma \ref{lem:cost_sensitivity}) and equation \eqref{eq:close} yields
	\begin{equation}\Stretch \label{eq:use_sensitivity}
	\begin{array}{cl}
		&|J_{[\varphi_{x_\mathrm{ic}}(u_{i})](\tau(u_{i}))}(w)-J_{[\varphi_{x_\mathrm{ic}}(u_{j})](\tau(u_{j}))}(w)| \\
		\leq & \Vert [\varphi_{x_\mathrm{ic}}(u_{i})](\tau(u_{i})) - [\varphi_{x_\mathrm{ic}}(u_{j})](\tau(u_{j})) \Vert_2  \cdot\frac{L_{g}}{L_{f}}\left(e^{L_f \left( \frac{h(R)}{R} -\tau(u_j) \right) }-1\right)\\
		\leq & \frac{\sqrt{n}}{\eta(R)}\frac{L_{g}}{L_{f}}\left(e^{L_f \left( \frac{h(R)}{R} -\tau(u_j) \right) }-1\right)\\
		&=\gamma.
		\end{array}
	\end{equation}
	Then combining \eqref{eq:use_concat} and \eqref{eq:use_sensitivity} yields
	\begin{equation}\Stretch
	\begin{array}{rcl}
	J_{x_\mathrm{ic}}(u_{i} w) & \leq & J_{x_\mathrm{ic}}(u_{j})+J_{[\varphi_{x_\mathrm{ic}}(u_{j})](\tau(u_{j}))}(w)+|J_{[\varphi_{x_\mathrm{ic}}(u_{i})](\tau(u_{i}))}(w)-J_{[\varphi_{x_\mathrm{ic}}(u_{j})](\tau(u_{j}))}(w)|\\
	& \leq & J_{x_\mathrm{ic}}(u_{j}w)+\frac{\sqrt{n}}{\eta(R)}\frac{L_{g}}{L_{f}}\left(e^{L_f \left( \frac{h(R)}{R} -\tau(u_j) \right) }-1\right).\\
	\end{array}
	\end{equation}
	Therefore $c_i\leq c_j+\gamma$. 
\end{proof}
\begin{figure}
	\centering
	\includegraphics[width=0.6\textwidth]{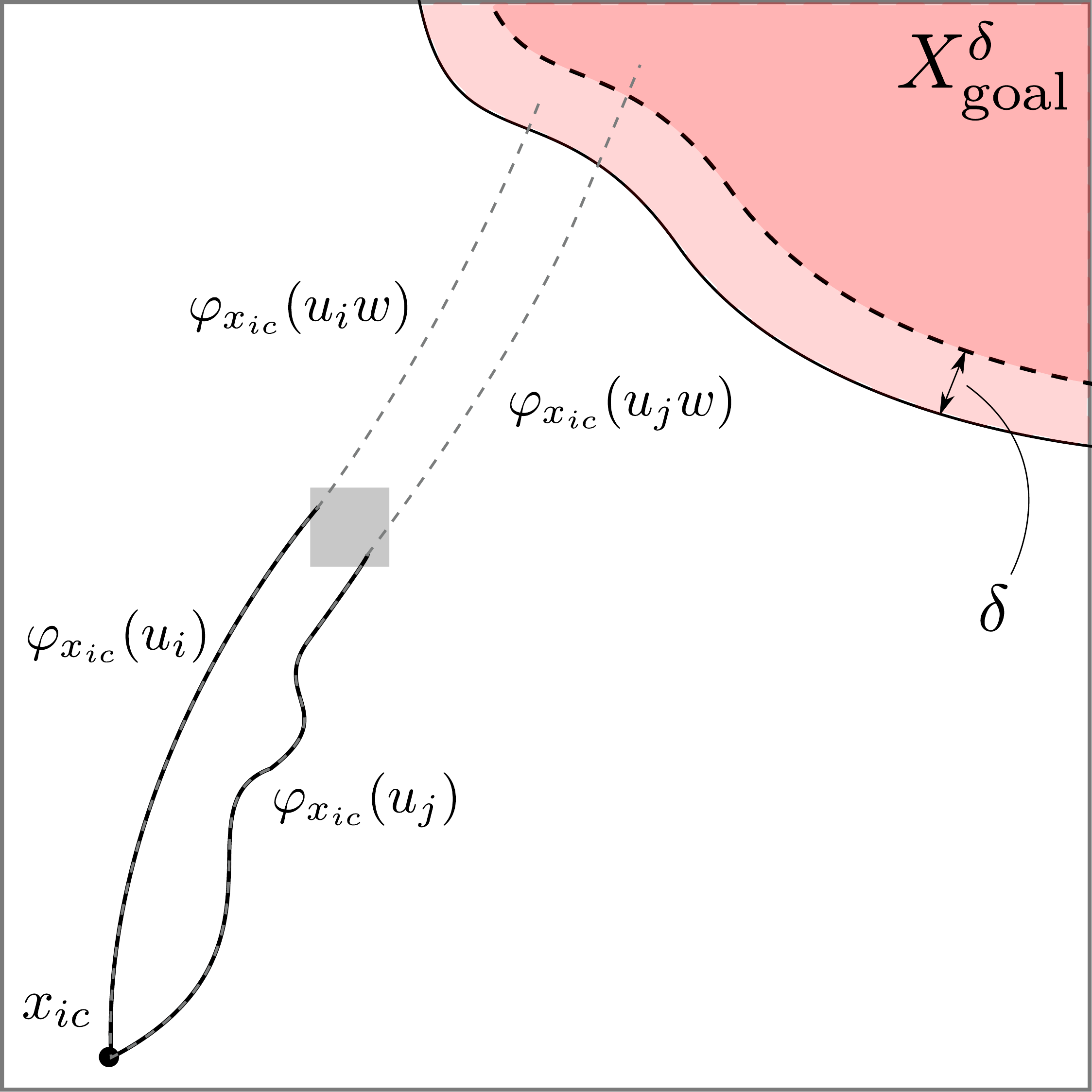}
	\caption{Illustration of the statement of Theorem \ref{thm:pruning}. The input signals $u_i$ and $u_j$ result in trajectories (black curves) terminating in the same partition element (grey square). If $u_{i}\prec_{R}u_{j}$, then for any descendant $u_jw$ of $u_j$, the descendant $u_iw$ of $u_i$ will have lesser cost, and will result in a trajectory terminating near the trajectory resulting from $u_jw$.}\label{fig:princ_opt}
\end{figure}

In reference to the above theorem, since $u_j w\in \UR$ which is limited to signals with depth $h(R)$, the signal $w$, which is concatenated with $u_i$ to form a signal in $\Ugoal\cap\UR$, necessarily satisfies 
\begin{equation}\label{eq:depth_constraint}
\mathtt{depth}(w)\leq h(R)-\mathtt{depth}(u_j),
\end{equation} 
where $\mathtt{depth}$ is the depth of a particular signal.

\section{A \GLC Algorithm}

Pseudocode for a  \GLC algorithm is defined in Algorithm \ref{Alg} below. 
As the name suggests, it is very similar to a canonical label correcting algorithm. 
A set $Q$ serves as a priority queue of candidate signals. A set $\Sigma$ contains signals representing labels of $\overset{\mathcal{U}_R}{\sim}$ equivalence classes.

The method $\texttt{expand}(u)$ returns the set of all children of $u$. 
The method $\texttt{pop}(Q)$ deletes from Q, and returns an input signal $\hat{u}$ such that 
\begin{equation}\label{eq:queue}
\hat{u}\in\underset{u\in Q}{{\rm argmin}}\left\{ J_{x_\mathrm{ic}}(u)\right\},
\end{equation} 
so that the presented \GLC algorithm is a best-first search\footnote{Like canonical label correcting methods, there are many variations that utilize alternative orderings. For example, the addition of an \emph{admissible heuristic}\index{admissible heuristic}~\cite{hart1968formal} in (\ref{eq:queue}) can be used to guide the search without affecting the solution accuracy.  
Admissible heuristics for kinodynamic motion planning are discussed in detail in Chapter \ref{chap:admissible_heuristics}.}. 
The method $\texttt{find}(u,\Sigma)$ returns $w\in\Sigma$ such that $u\overset{R}{\sim}w$ or \NULL if no such $w$ is present in $\Sigma$.
Problem specific collision and goal checking subroutines are used to evaluate $u\in\mathcal{U}_\mathrm{feas}$ and $u\in\mathcal{U}_\mathrm{goal}$.
The method $\texttt{depth}(u)$ returns the number of ancestors of $u$. 
\begin{algorithm} 
	\begin{algorithmic}[1]
		\State $Q\leftarrow \{Id_\mathcal{U}\},\,\Sigma \gets \emptyset,\,S \gets \emptyset$  
		\While {$Q\neq \emptyset$}        
		\State $u \gets \texttt{pop}(Q)$   
		\If{$u \in \mathcal{U}_\mathrm{goal} $} 
		\State \Return $(J_{x_{ic}}(u),u)$ 
		\EndIf
		\State $S \gets \texttt{expand}(u)$   
		\For{$w \in S$} 	 	
		\State $z \gets \texttt{find}(w,\Sigma)$ 	  
		\If{$(w \notin \mathcal{U}_{feas.} \vee (z \prec_R w) \vee \texttt{depth}(w) \geq h(R))$}  	    
		\State $S \gets S\setminus \{w\}$  	  
		\ElsIf{$J_{x_\mathrm{ic}}(w)<J_{x_\mathrm{ic}}(z)$} 	  
		\State $\Sigma \gets (\Sigma \setminus \{z\}) \cup \{w\}$ 	  
		\EndIf 
		\EndFor       
		\State $Q \gets Q \cup S$ 
		\EndWhile  
		\State 
		\Return $(\infty,\NULL)$ 
	\end{algorithmic} \caption{\label{Alg} Generalized Label Correcting (GLC) Method} 
\end{algorithm}
The algorithm begins by adding the root $Id_\mathcal{U}$ to the queue (line 1), and then enters a loop which recursively removes and expands the top of the queue (line 3) adding children to $S$ (line 4). 
If the queue is empty the algorithm terminates (line 2) returning \NULL (line 14). 
Each signal in $S$ (line 5) is checked for membership in $\mathcal{U}_\mathrm{goal}$ in which case the algorithm terminates returning a feasible solution with approximately the optimal cost. 
Otherwise, the signals are checked for infeasibility or suboptimality by the \GLC conditions (line 9). 
Next, a relabeling condition for the associated equivalence classes (i.e. grid cells) of remaining signals is checked (line 11). 
Finally, remaining signals in $S$ are added to the queue (line 13).  

In addition to the problem data, the algorithm requires two functions $h(R)$ and $\eta(R)$ satisfying equations \eqref{eq:horizon} and \eqref{eq:partition_scaling}, and the control inputs $\Omega_R$ satisfying \eqref{eq:disp_conv}.

\section{Proof of Resolution Completeness}
The goal of this section is to prove that Algorithm \ref{Alg} is resolution complete for Problem \ref{Problem}.
Algorithm \ref{Alg} is very similar to classical label correcting algorithms and the proof parallels standard proofs of completeness for these methods.
The reader may benefit from reviewing the proof in the simpler case where the algorithm is applied to a shortest path problem in a graph (e.g. \cite[chapter 2]{bertsekas1995dynamic} or Appendix \ref{app:graphs}).

In Theorem \ref{thm:pruning}, the quantity $\delta$ was constructed based on: the radius of the partition of \Xfree, the sensitivity of solutions to initial conditions determined by Lemma \ref{lem:cont_ic}, and the maximum duration of signals in \UR. In the proof of the next result, we will recursively apply Theorem \ref{thm:pruning} and it will be useful to define the following sum that arises when repeatedly applying Theorem \ref{thm:pruning},  
\begin{equation}
\delta_k\coloneqq\sum_{i=0}^{k}\frac{\sqrt{n}}{L_f \eta(R)}e^{\frac{L_{f}(h(R)-i)}{R}}.
\end{equation}
When $k=h(R)$, the sum satisfies the inequality
\begin{equation}\label{eq:inequality}
\sum_{i=0}^{k}\frac{\sqrt{n}}{L_f \eta(R)}e^{\frac{L_{f}(h(R)-i)}{R}} \leq\frac{R\sqrt{n}}{L_{f}\eta(R)}\left(e^{L_{f}h(R)/R}-1\right),
\end{equation}
which is derived in Appendix \ref{app:inequality_derivation}.
Similarly, it will be convenient to define the quantity $\gamma_k$ related to $\gamma$ of Theorem \ref{thm:pruning},
\begin{equation}
\gamma_k \coloneqq \sum_{i=0}^{k} \frac{\sqrt{n}}{\eta(R)} \frac{L_g}{L_f}\left( e^{L_f \left( \frac{h(R)}{R} -i \right) }-1 \right),
\end{equation} 
which is bounded by 
\begin{equation}\label{eq:messy_formula}
	\sum_{i=0}^{h(R)} \frac{\sqrt{n}}{\eta(R)} \frac{L_g}{L_f}\left( e^{L_f \left( \frac{h(R)}{R} -i \right) }-1 \right) \leq \frac{R\sqrt{n}}{\eta(R)}\frac{L_g}{L_f}\left(e^{L_fh(R)/R} - \frac{h(R)}{R}  \right).
\end{equation}
An important property of the right hand side of \eqref{eq:inequality} and  \eqref{eq:messy_formula} is that they converge to zero when $\eta(R)$ satisfies equation \eqref{eq:partition_scaling}. Thus, $\delta_{h(R)}$ and $\gamma_{h(R)}$ converge to zero as $R$ tends to infinity.   

The pruning operation in lines 9-10 of Algorithm \ref{Alg}, in a sense, discards candidate input signals as liberally as possible.
In the next theorem, it is shown that a signal in $\Ugoal^{\delta_{h(R)}}$ with cost less than $c_R^{\delta_{h(R)}}+\gamma_{h(R)}$ is eventually evaluated by the algorithm in line 4 despite this pruning operation.
\begin{theorem}\label{thm:main}
	The \GLC method described by Algorithm \ref{Alg} terminates in finite time and returns a solution with cost less than or equal to $c^{\delta_{h(R)}}_R+\gamma_{h(R)}$.
\end{theorem}
\begin{proof}
	(Finite running time) The queue is a subset of $\mathcal{U}_R\cup \{Id_\mathcal{U} \}$ and at line 3 in each iteration a lowest cost signal $u$ is removed from the queue. 
	In line 13, only children of the current signal $u$ are added to the queue. Since $\mathcal{U}_R$ is organized as a tree and has no cycles, any signal $u$ will enter the queue at most once.
	Therefore the queue must be empty after a finite number of iterations so the algorithm terminates.
	
	(Approximate Optimality)
	Next, consider as a point of contradiction the hypothesis that the output has cost greater than $c^{\delta_{h(R)}}_R+\gamma_{h(R)}$. 
	Then it is necessary that $c^{\delta_{h(R)}}_R+\gamma_{h(R)}<\infty$, and by the definition of  $c^{\delta_{h(R)}}_R$ in (\ref{eq:cr_ep}), it is also necessary that $\UR \cap \Ugoal^{\delta_{h(R)}}$ is non-empty. 
	
	Choose $u^*\in \UR\cap\Ugoal^{\delta_{h(R)}}$ with cost $J_{x_\mathrm{ic}}(u^*)=c_R^{\delta_{h(R)}}$. 
	It follows from the contradiction hypothesis that $u^*$ does not enter the queue.
	Otherwise, by (\ref{eq:queue}) it would be evaluated before any signal of cost greater than $c_R^{\delta_{h(R)}}$ and the algorithm would terminate returning this input signal.
	If $u^*$ does not enter the queue, then a signal $u_0$ must at some iteration be present in $\Sigma$ which prunes an ancestor $a_0$ of $u^*$ ($u_0 \prec_R a_0$ in line 9). 
	This ancestor must satisfy $\texttt{depth}(a_0)>0$ since the ancestor with depth $0$ is $Id_\U$ which enters queue in line 1. 
	By Theorem \ref{thm:pruning}, $u_0$ has a descendant of the form $u_0d_0\in \Ugoal^{\delta_{h(R)}-\delta_0}$ and $J_{x_\mathrm{ic}}(u_0d_0)\leq c_R^{\delta_{h(R)}}+\gamma_0$. 
	Additionally, $\mathtt{depth}(a_0)>0$ implies $\mathtt{depth}(d_0)\leq h(R)-1$ by equation \eqref{eq:depth_constraint}.

	Having pruned $u^*$ in lines 9-10, the signal $u_0\in\Sigma$, or a sibling which prunes $u_0$ (and by the transitivity of $\prec_R$, prunes $u^*$) must at some point be present in the queue (cf. line 12-13). 
	Of these two, denote the one that ends up in the queue by $\tilde{u}_0$.
	Since $\tilde{u}_0$ is at some point present in the queue and $\tilde{u}_0d_0\in\Ugoal^{\delta_{h(R)}-\delta_0}$, a signal $u_1\in\Sigma$ must prune an ancestor $a_1$ of $\tilde{u}_0d_0$ ($u_1 \prec_R a_1$ in lines 9-10). 
	Since $\tilde{u}_0$ is at some point present in the queue, the ancestor $a_1$ of $\tilde{u}_0d_0$, must have greater depth than $\tilde{u}_0$. 
	By Theorem \ref{thm:pruning}, $u_1$ has a descendant of the form $u_1d_1\in \Ugoal^{\delta_{h(R)}-\delta_1}$ and $J_{x_\mathrm{ic}}(u_1d_1)\leq c_R^{\delta_{h(R)}}+\gamma_1$. 
	Additionally, $\mathtt{depth}(a_1)> \tilde{u}_0$ implies $\mathtt{depth}(d_1)\leq \mathtt{depth}(d_0)-1 \leq h(R)-2$ by equation \eqref{eq:depth_constraint}.

	Continuing this line of deduction leads to the observation that a signal $u_{h(R)-1}$, with a descendant of the form $u_{h(R)-1}d_{h(R)-1}\in \Ugoal^{\delta_{h(R)}-\delta_{h(R)-1}}$ and $J_{x_\mathrm{ic}}(u_{h(R)-1}d_{h(R)-1})\leq c_R^{\delta_{h(R)}}+\gamma_{h(R)-1}$, will be present in the queue; and $\mathtt{depth}(d_{h(R)-1})\leq 1$. 
	Since $u_{h(R)-1}$ is at some point present in the queue, a signal $u_{h(R)}\in\Sigma$ must prune an ancestor $a_{h(R)}$ of $u_{h(R)-1}d_{h(R)-1}$ ($u_{h(R)} \prec_R a_{h(R)}$ in lines 9-10).
	Since $u_{h(R)-1}$ is at some point present in the queue, the ancestor $a_{h(R)}$ of $u_{h(R)-1}d_{h(R)-1}$, must have greater depth than $u_{h(R)-1}$, and therefore, is equal to $u_{h(R)-1}d_{h(R)-1}$. 
	Thus, $u_{h(R)}\in \Ugoal^{\delta_{h(R)}-\delta_{h(R)-1}}$ and $J(u_{h(R)})\leq c_R^{\delta_{h(R)}}+\gamma_{h(R)}$. 
	Then $u_{h(R)}$ or a sibling which prunes $u_{h(R)}$ will be added to the queue; a contradiction of the hypothesis since this signal will be removed from the queue and the algorithm will terminate, returning this signal in line 7. 

\end{proof}

The choice of $\delta_{h(R)}$ and $\gamma_{h(R)}$ in Theorem \ref{thm:main} converge to zero as $R$ tends to infinity by  (\ref{eq:partition_scaling}).
Then by Lemma \ref{lem:approx_equal_optimal} we have $c_{R}^{h(R)} + \gamma_{h(R)}\rightarrow c^{*}$. 
An immediate corollary is that the \GLC method is resolution complete.
It is important to note that for low resolution $R$, it is possible that $c_R^{h(R)}=\infty$. 

\begin{cor}\label{cor:main}
	Let $w_R$ be the signal returned by the \GLC method for resolution $R$. Then $\lim_{R\rightarrow \infty} J_{x_{ic}}(w_R)=c^*$. That is, the \GLC method is a resolution complete algorithm for the optimal kinodynamic motion planning problem.
\end{cor}

\section{Kinodynamic Motion Planning Examples}

The \GLC method (Algorithm \ref{Alg}) was tested on five problems and compared, when applicable, to the implementation of \RRTs from~\cite{rrt_implementation} and \SST from~\cite{BBekris2015}. 
The goal is to examine the performance of the \GLC method on a wide variety of problems. 
The examples include under-actuated nonlinear systems, multiple cost objectives, and environments with/without obstacles.
Note that adding obstacles effectively speeds up the \GLC method since it reduces the size of the search tree.
Another focus of the examples is on real-time application.
In each example the running time for \GLC method to produce a (visually) acceptable trajectory is comparable to the execution time.
Of course this will vary with problem data and computing hardware.    

\paragraph{Implementation Details:}

The \GLC method was implemented in C++ and run with a 3.70GHz Intel Xeon CPU. 
The set $Q$ was implemented with an STL priority queue so that the $\texttt{pop}(Q)$ method and insertion operations have logarithmic complexity in the size of $Q$. 
The set $\Sigma$ was implemented with an STL set which uses a binary search tree so that $\texttt{find}(w,\Sigma)$ also has logarithmic complexity in the size of $\Sigma$. 

Sets $X_\mathrm{free}$ and $X_\mathrm{goal}$ are described by algebraic inequalities.
The approximation of the input space $\Omega_{R}$ is constructed by uniform deterministic sampling of $R^{m}$ (i.e. the resolution $R$ raised to the power $m$ which is the dimension of the input space) controls from $\Omega$.
Recall $m$ is the dimension of the input space.

Evaluation of $u\in\mathcal{U}_\mathrm{feas}$ and $u\in\mathcal{U}_\mathrm{goal}$ is approximated by first numerically computing $\varphi_{x_\mathrm{ic}}(u)$ with Euler integration (except for \RRTs which uniformly samples along the local planning solution). 
The number of time-steps is given by $N=\left\lceil \tau(u)/\Delta\right\rceil $ with duration $\tau(u)/N$. Maximum time-steps $\Delta$ are 0.005 for the first problem, 0.1 for the second through fourth problem, and 0.02 for the last problem. Feasibility is then approximated by collision checking at each time-step along the trajectory.

An additional scaling of the input signal primitive duration by $\gamma$ is introduced in these examples.
Theoretically, the output of the algorithm for any value of $\gamma$ will be consistent with Theorem \ref{thm:main}, but in practice it is advantageous to select a scaling $\gamma$ representing a characteristic time-scale for the problem. For example, if the motion is expected to take on the order of microseconds, then selecting $\gamma=10^{-6}$ would adjust the time-scale of the search appropriately.

\subsection{Shortest Path Problem}\label{ex:kine}

A shortest path problem in $\mathbb{R}^n$ can be represented by the dynamic model
\begin{equation}
\frac{d}{dt} x(t)=u(t),
\end{equation} 
the running-cost
\begin{equation}
g(x,u)=1,
\end{equation} 
and the control input space 
\begin{equation}
\Omega=\{u\in\mathbb{R}^{2}:\,\Vert u\Vert_{2}=1\}.
\end{equation}
In this example, \xfree and \xgoal are described by polyhedral sets illustrated in Figure \ref{fig:kine_bench}.

The parameters for the \GLC method are summarized in Table \ref{tab:shortest_path_params}. 
Note that $h(R)$ and $\eta(R)$ satisfy the asymptotic constraints \eqref{eq:horizon} and  \eqref{eq:partition_scaling}.
\begin{table}[h]
	\vspace{0.075in}
	\begin{center}
		\caption{Tuning parameter selection for the shortest path example. \label{tab:shortest_path_params}}
		\begin{tabular}{ m{6.0cm} | m{5cm}} 
			Resolution range: $R$ & $R=\{20,25,...,200\}$ \\
			\hline
			Horizon limit: $h(R)$ & $100R\log(R)$ \\
			\hline
			Partition scaling: $\eta(R)$& $R^{2}/300$ \\ 
			\hline
			Control primitive duration: $\gamma/R$ & $10/R$
		\end{tabular}
		\vspace{0em}
	\end{center}
	\vspace{-3mm}
\end{table}
The \GLC, \SST, and \RRTs methods were all tested in this example. 
The average performance from 10 trials is reported for the \SST and \RRTs methods. 
Figure \ref{fig:kine_bench} summarizes the results.
This example also considers a guided search using the cost of the \RRTs local planning subroutine solution to the goal as a heuristic.

\begin{figure}[h]
	\centering
	\includegraphics[width=0.7\textwidth]{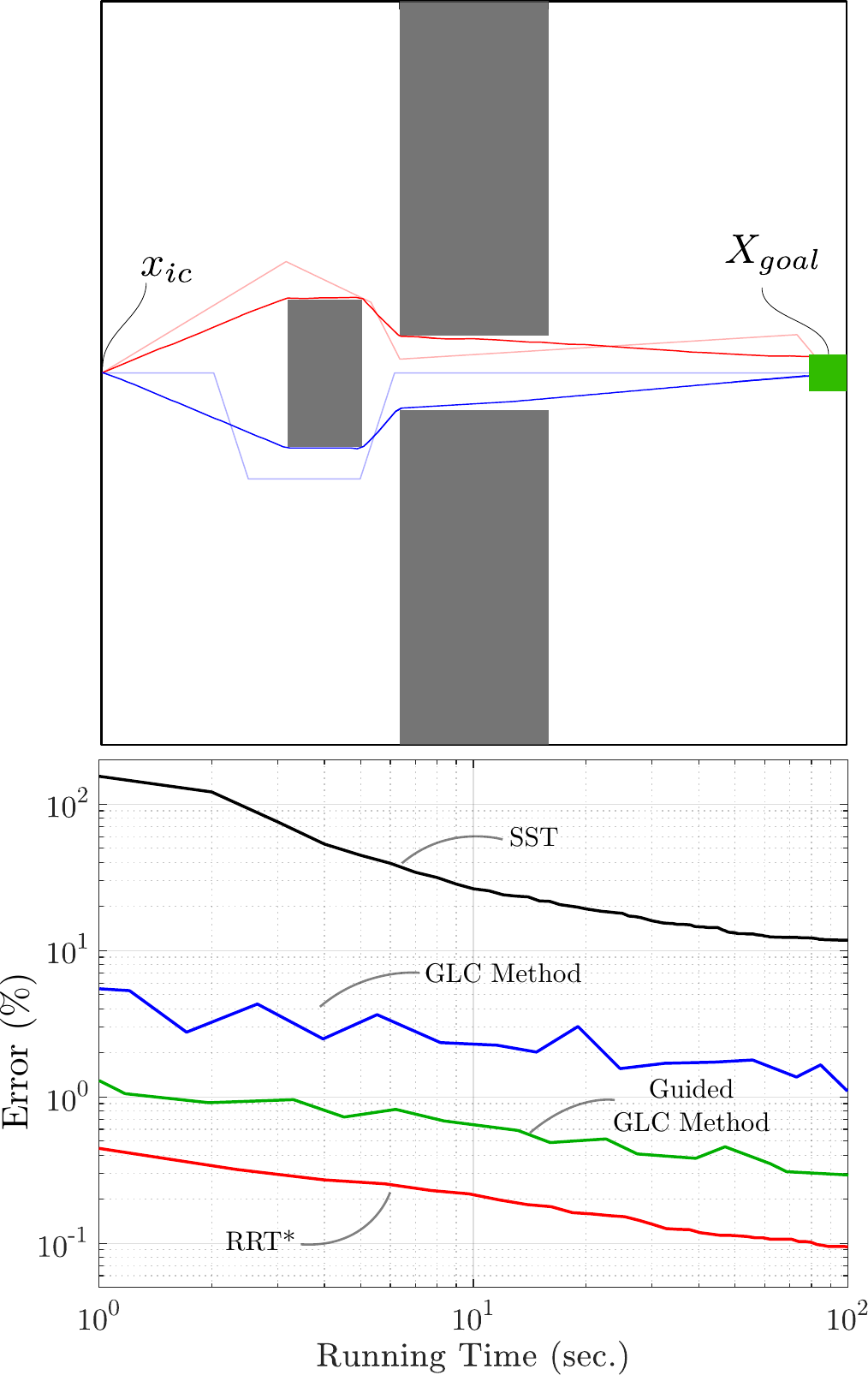}
	\caption{(top) The free space \xfree is shown in white with obstacles shown in grey. The green square is the goal region \xgoal. The light red and blue paths are low resolution solutions obtained by the \RRTs method and the \GLC method respectively. The dark colored paths correspond to high resolution solutions. (bottom) The running time required by each algorithm to produce a solution of a particular accuracy is shown. The \SST method is shown in black, the \GLC method in blue, and the \RRTs method in red. Additionally, a guided implementation of the \GLC method utilizing the \RRTs steering solution as a heuristic is shown in green.}\label{fig:kine_bench}
\end{figure}
This is the only example in this section where the exact solution to the problem is known. In this case we can compare the relative convergence rates of the \GLC method, \SST, and \RRTs. 

\subsection{Torque-Limited Pendulum Swing-Up}
A minimum-time torque-limited pendulum swing-up problem can be represented by the dynamic model
\begin{equation}
\dot{\theta}(t)=\omega(t),\qquad \dot{\omega}(t)=-\sin(\theta(t))+u(t),
\end{equation}
with the running cost 
\begin{equation}
g(\theta,\omega,u)=1,
\end{equation} 
and the control input space 
\begin{equation}
\Omega=[-0.2,0.2].
\end{equation}
The free space and goal set are described by
\begin{equation}
\xfree=\mathbb{R}^2,\qquad \xgoal=\{(\theta(t),\omega(t))\in \mathbb{R}^2:\,\Vert(\theta(t) \pm \pi,\omega(t) \Vert_2 \leq 0.1\}.
\end{equation}
The initial state is the origin $(\theta(0),\omega(0))=(0,0)$.
The parameters for the \GLC method are summarized in Table \ref{tab:shortest_path_params}. 
\begin{table}[h]
	\vspace{0.075in}
	\begin{center}
		\caption{Tuning parameter selection for the pendulum swing-up example. \label{tab:pend_swingup_params}}
		\begin{tabular}{ m{6.0cm} | m{5cm}} 
			Resolution range: $R$ & $R=\{4,5,6,7,8\}$ \\
			\hline
			Horizon limit: $h(R)$ & $100R\log(R)$ \\
			\hline
			Partition scaling: $\eta(R)$& $R^{5/2}/16$ \\ 
			\hline
			Control primitive duration: $\gamma/R$ & $6/R$
		\end{tabular}
		\vspace{0em}
	\end{center}
	\vspace{-3mm}
\end{table}
The \GLC and \SST methods were tested in this example with the average performance from 10 trials reported for the \SST method. 
Figure \ref{fig:pend_bench} summarizes the results.

\begin{figure}[h]
	\centering
	\includegraphics[width=0.7\textwidth]{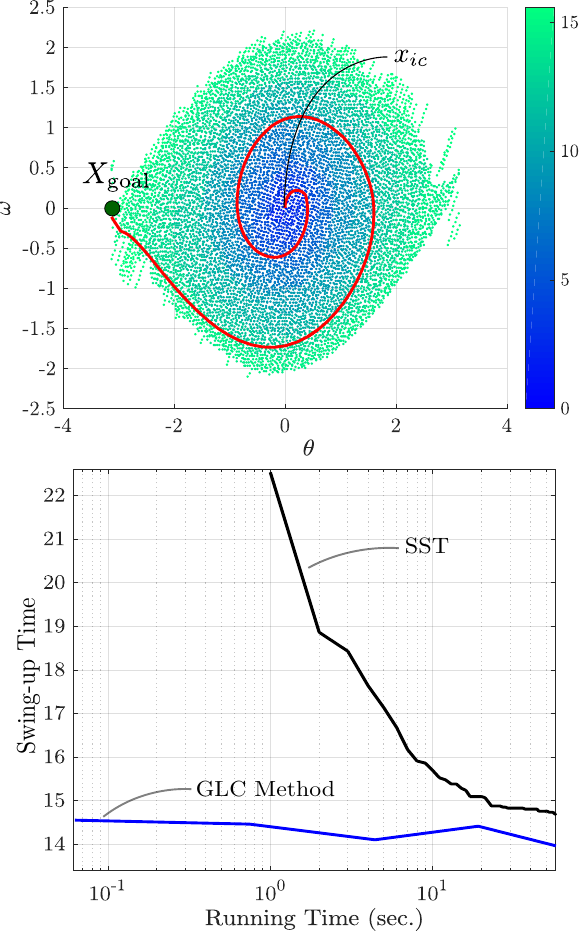}
	\caption{(top) The state space trajectory of the pendulum swing-up motion is shown in red. Colored markers show terminal states of minimal trajectories which have been evaluated with the color indicating the cost of the trajectory terminating at that point. (bottom) The running-time and cost of of the output to the \GLC algorithm for increasing resolution is shown in blue. The 10 trial average of the \SST algorithms output from 1 second of run-time to 500 seconds of run-time is shown in black. Note that the swing-up time is dimensionless.}\label{fig:pend_bench}
\end{figure}

\subsection{Torque Limited Acrobot Swing-Up}

The acrobot is a double link pendulum actuated at the middle joint. 
The expression for the four dimensional system dynamics are cumbersome to describe and we refer to~\cite{spong1995swing} for the details. 
The model parameters, free space, and goal region are identical to those in the benchmark provided in \cite{BBekris2015} with the exception that the radius of the goal region is reduced from $2.0$ to $0.5$.
A minimum-time running-cost 
\begin{equation}
g(x,u)=1,
\end{equation}
is used.
The control input space is $\Omega=[-4.0,4.0]$ representing the minimum and maximum torque that can be applied at the middle joint. 
The parameters for the \GLC method are summarized in Table \ref{tab:acro_swingup_params}.
\begin{table}[h]
	\vspace{0.075in}
	\begin{center}
		\caption{Tuning parameter selection for the acrobot swing-up example. \label{tab:acro_swingup_params}}
		\begin{tabular}{ m{6.0cm} | m{5cm}} 
			Resolution range: & $R\in\{4,5,...,10\}$ \\
			\hline
			Horizon limit: $h(R)$ & $100R\log(R)$ \\
			\hline
			Partition scaling: $\eta(R)$& $R^{2}/16$ \\ 
			\hline
			Control primitive duration: $\gamma/R$ & $6/R$
		\end{tabular}
		\vspace{0em}
	\end{center}
	\vspace{-3mm}
\end{table}
The \GLC and \SST methods were tested in this example with the average performance from 10 trials reported for the \SST method. 
Figure \ref{fig:acro_bench} summarizes the results.

\begin{figure}[h]
	\centering
	\includegraphics[width=0.5\textwidth]{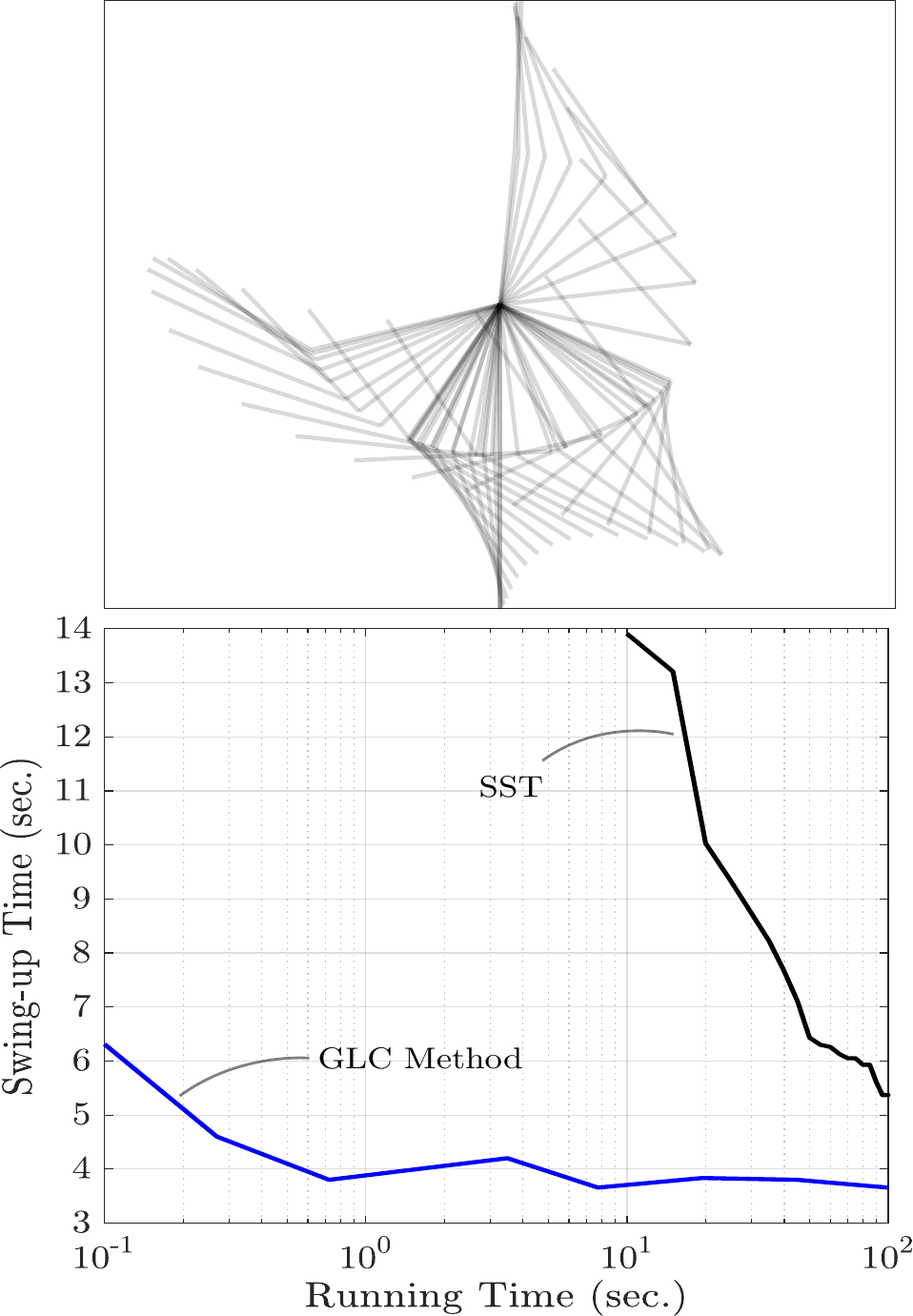}
	\caption{(top) The swing-up motion of the acrobot returned by the \GLC algorithm. (bottom) The running-time and cost of of the output to the \GLC algorithm for increasing resolution is shown in blue. The 10 trial average of the \SST algorithms output from 1 second of run-time to 100 seconds of run-time is shown in black.}\label{fig:acro_bench}
\end{figure}

\subsection{Thrust Limited 3D Point Robot}\label{ex:3D_robot}

To emulates the mobility of an agile aerial vehicle (e.g. a quadrotor with high bandwidth attitude control), the system dynamics are modeled by
\begin{equation}
\dot{x}(t)=v(t),\qquad \dot{v}(t)=5.0\cdot u(t)-0.1\cdot v(t)\cdot \Vert v(t) \Vert_2,
\end{equation} 
where $x(t)$, $v(t)$, and $u(t)$ are each elements of $\mathbb{R}^{3}$; there are a total of six states and three control inputs.
The quadratic dissipative force anti-parallel to the velocity $v(t)$ models aerodynamic drag during high speed flight.
The minimum time running cost is 
\begin{equation}
g(x,v,u)=1,
\end{equation}
and the input space is 
\begin{equation}
\Omega=\{u\in\mathbb{R}^{3}:\Vert u\Vert_{2}\leq1\}.
\end{equation}
The planning task is for the point robot to navigate from an initial state in one room to a destination in an adjacent room connected by a small window.
This free space is illustrated in Figure \ref{fig:uav_bench}.
The dark blue marker in the leftmost room illustrates the goal configuration, and the light blue marker in the rightmost room illustrates the starting configuration. 
The velocity is initially zero, and the terminal velocity is left free.
A guided search is also considered with heuristic given by the distance to the goal divided by the maximum speed $v$ of the robot (A maximum speed of $\sqrt{{50}}$ can be determined from the dynamics and input constraints). 
The parameters for the \GLC method are summarized in Table \ref{tab:shortest_path_params}.

\begin{table}[h]
	\vspace{0.075in}
	\begin{center}
		\caption{Tuning parameter selection for the thrust limited 3D point robot example. \label{tab:uav_params}}
		\begin{tabular}{ m{6.0cm} | m{5cm}} 
			Resolution range: & $R\in\{8,9,...,12\}$ \\
			\hline
			Horizon limit: $h(R)$ & $100R\log(R)$ \\
			\hline
			Partition scaling: $\eta(R)$& $R^{2}/64$ \\ 
			\hline
			Control primitive duration: $\gamma/R$ & $10/R$
		\end{tabular}
		\vspace{0em}
	\end{center}
	\vspace{-3mm}
\end{table}
The control input approximation $\Omega_R$ is generated by an increasing number of points distributed uniformly on $\Omega$. However, because of the symmetry of the sphere, the computational procedure for generating $\Omega_R$ in this case produces the same set of points up to a random orthogonal transformation~\cite{paden2017}.
Since there is some randomness in this example, the results reported in Figure \ref{fig:uav_performance} are averaged over 10 trials both for the \GLC and \SST algorithm. 
Figure \ref{fig:uav_bench} shows the three outputs of the \GLC algorithm at three different resolutions to illustrate the improvement in solution quality with increasing resolution.

\begin{figure}[h]
	\centering
	\includegraphics[width=1.0\textwidth]{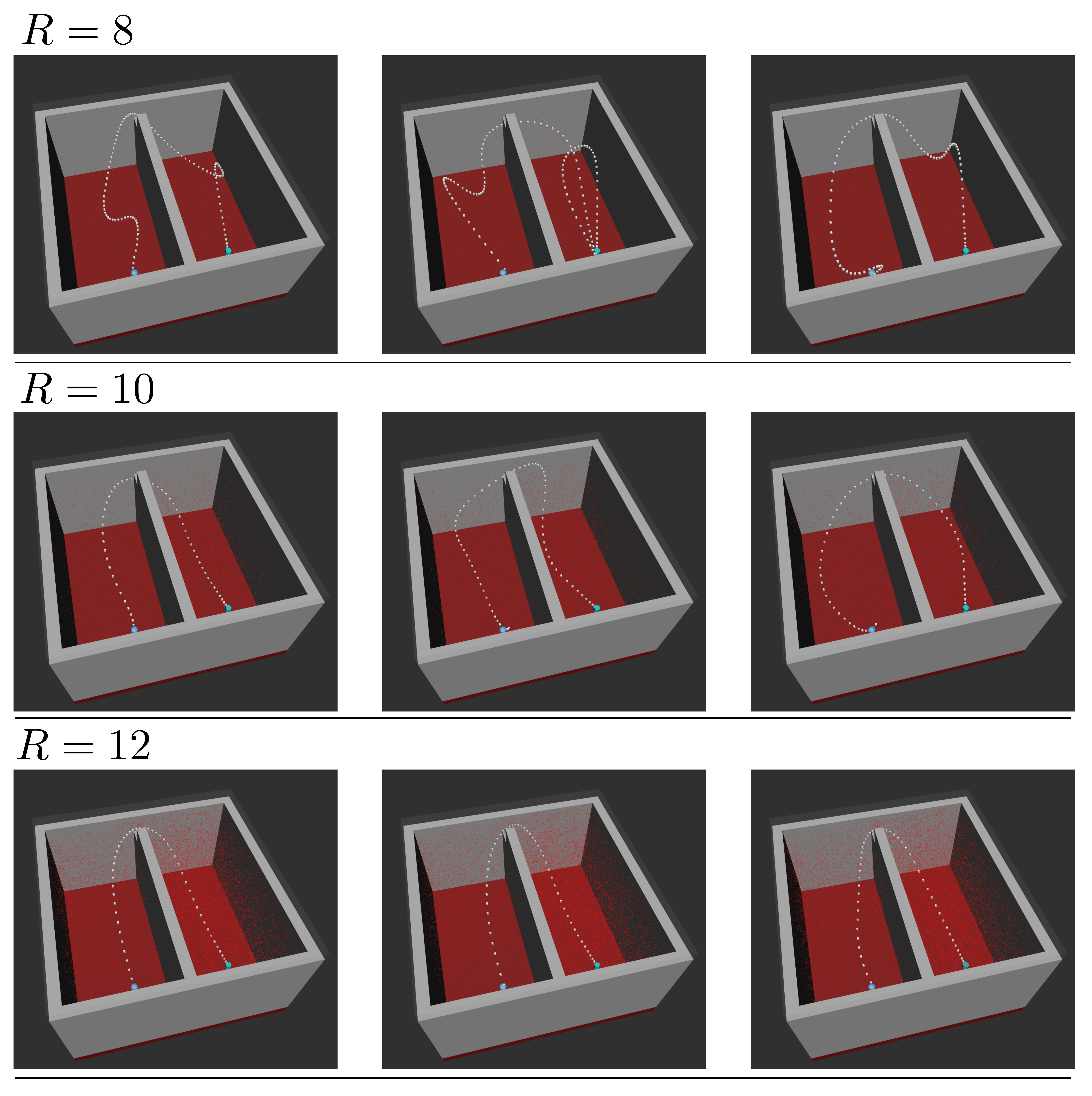}
	\caption{The top row of graphics show trajectories beginning at the start configuration (light blue marker) and terminating at the goal configuration (dark blue marker) returned by the \GLC algorithm with the resolution $R=8$. Since the input space $\Omega_R$ is deterministic up to a random orthogonal transformation in this example, there is some variation in the output. The middle and bottom graphics show the same results for $R=10$ and $R=12$ respectively. With increasing resolution the trajectories converge to the optimal trajectory.}\label{fig:uav_bench}
\end{figure}

\begin{figure}[h]
	\centering
	\includegraphics[width=0.8\textwidth]{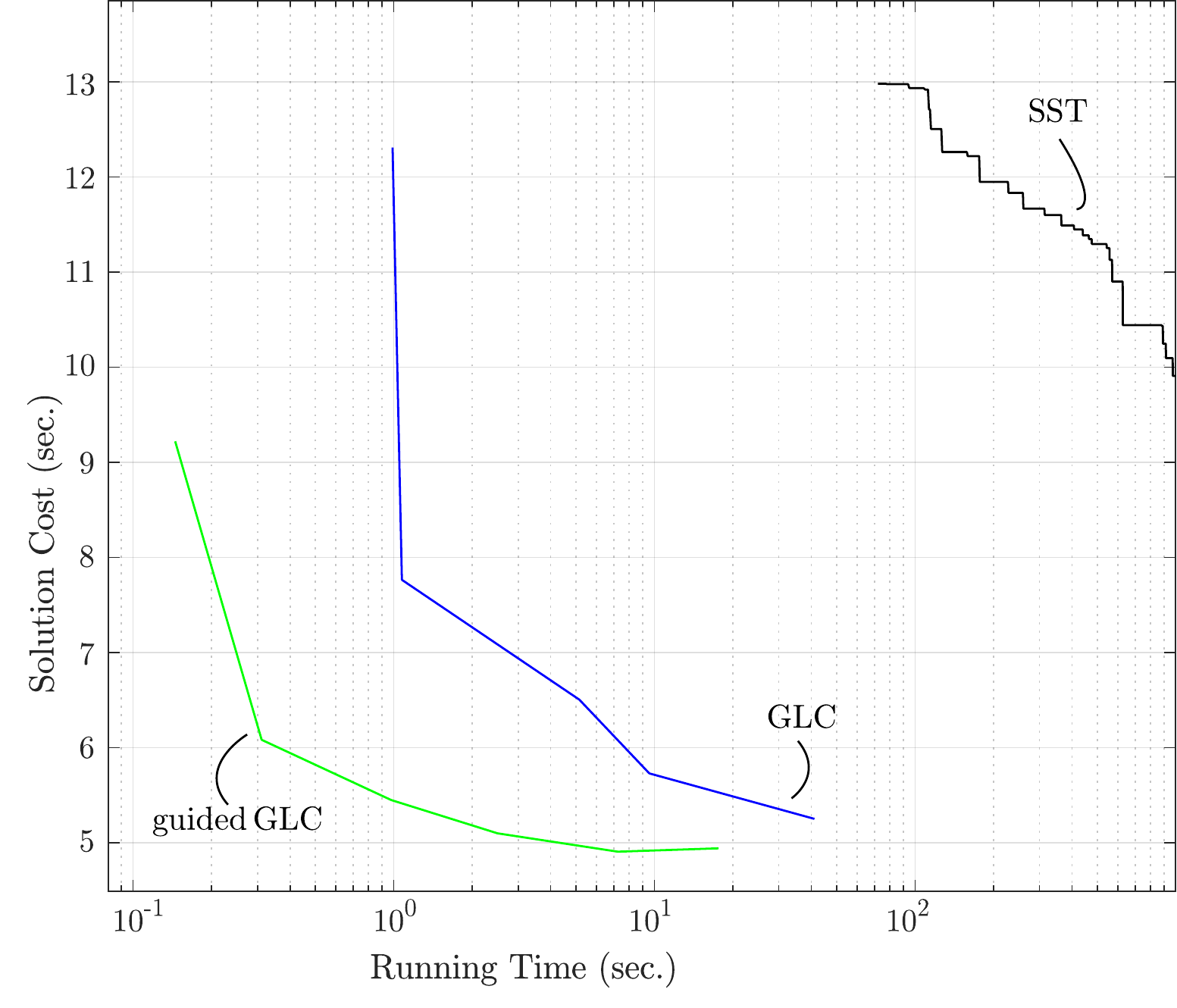}
	\caption{The running time of the \GLC algorithm (blue) and a guided variant (green) for $R\in \{8,9,10,11,12\}$ are compared to the \SST algorithm (black).}\label{fig:uav_performance}
\end{figure}

\subsection{Nonholonomic Wheeled Robot}

The dynamic model emulating the mobility of a wheeled robot is given by
\begin{equation}
\dot{p_1}(t)=\cos(\theta(t)),\quad \dot{p_2}(t)=\sin(\theta(t)),\quad \dot{\theta}(t)=u(t).
\end{equation}
A running cost which penalizes a combination of time and lateral acceleration (believed to be correlated with rider comfort) is given by
\begin{equation} 
g(p_1,p_2,\theta,u)=1+2u^{2}.
\end{equation}
The input space, which is consistent with a turning radius of $1$, is
\begin{equation} 
\Omega=[-1,1].
\end{equation}
The parameters for the \GLC method used in this example are summarized in Table \ref{tab:shortest_path_params}.

\begin{table}[h]
	\vspace{0.075in}
	\begin{center}
		\caption{Tuning parameter selection for the wheeled robot example. \label{tab:car_params}}
		\begin{tabular}{ m{6.0cm} | m{5cm}} 
			Resolution range: & $R\in\{4,5,...,9\}$ \\
			\hline
			Horizon limit: $h(R)$ & $5R\log(R)$ \\
			\hline
			Partition scaling: $\eta(R)$& $R^{5/\pi}/15$ \\ 
			\hline
			Control primitive duration: $\gamma/R$ & $10/R$
		\end{tabular}
		\vspace{0em}
	\end{center}
	\vspace{-3mm}
\end{table}
\begin{figure}[h]
	\centering
	\includegraphics[width=0.7\textwidth]{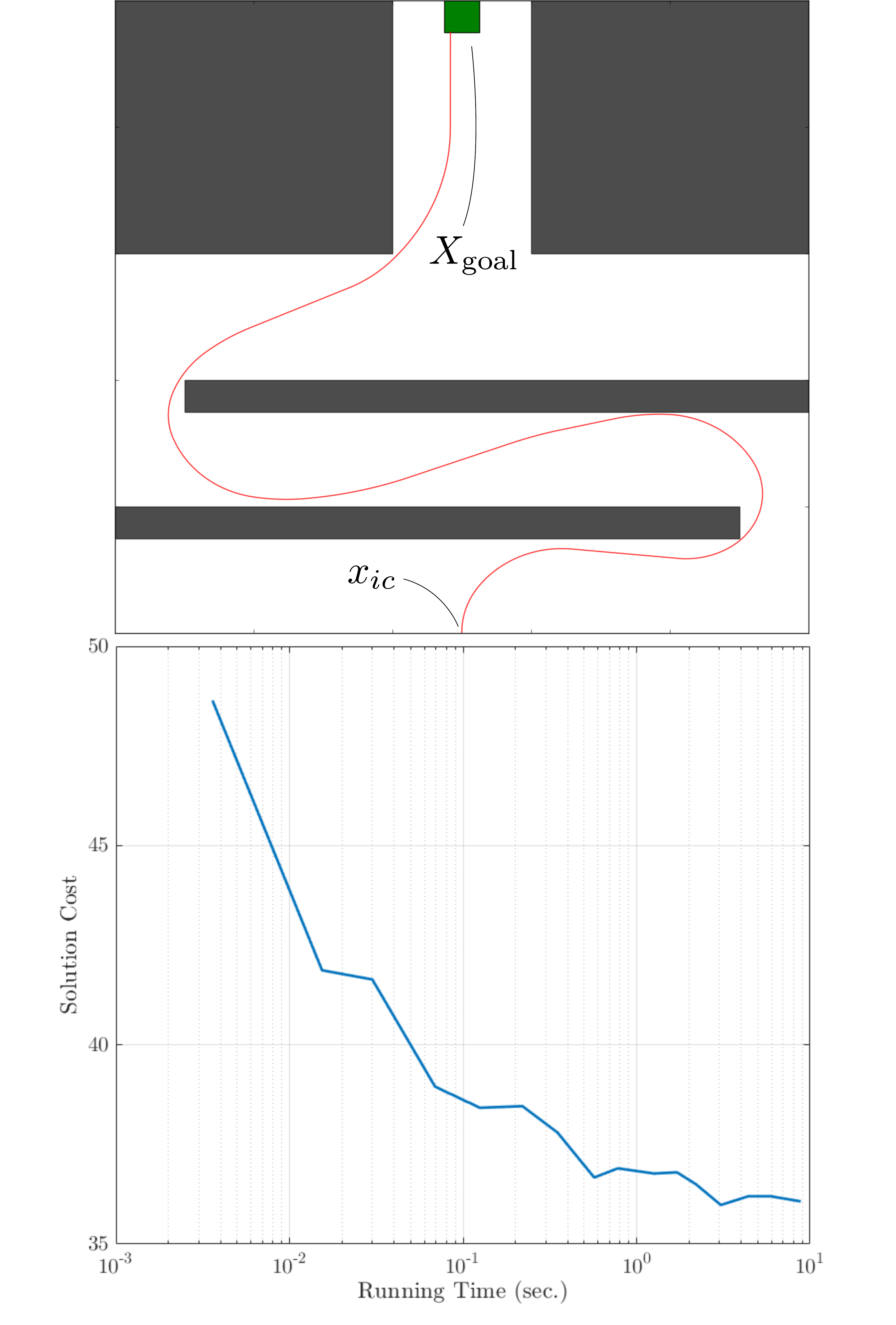}
	\caption{(top) The blue curve illustrates a trajectory from the initial configuration to the goal set which minimizes a mixed penalty on time and lateral acceleration which reflects a comfort objective for a passenger vehicle. (bottom) The running time of the \GLC algorithm required to obtain a solution of given cost.}\label{fig:car_bench}
\end{figure} 

Figure \ref{fig:car_bench} summarizes the performance of the \GLC algorithm on this problem.
\subsection{Observations and Discussion}
The first observation is that all algorithms tested generally produce lower cost solutions with increased run-time. 
In example \ref{ex:kine}, where the optimal cost was known a-priori, the \GLC and \RRTs algorithms were observed to converge to this value, and the \SST algorithm converged to some approximately optimal cost.
In the remaining examples, where the optimal cost was unknown and \RRTs was not applicable, the \GLC and \SST algorithms were observed to produce solutions of generally decreasing cost with increased run-time.
The local planning subroutine for the \RRTs algorithm in example \ref{ex:kine} was simply the line segment connecting two points. 
This is requires virtually no computation to generate making \RRTs the preferred algorithm.

Another observation was that the cost of solutions output by the \GLC algorithm generally decreased with increasing resolution, but did not decrease monotonically.
This is not inconsistent with Corollary \ref{cor:main} since the result only claims convergence.
The cause of the nonmonotonic convergence is that each run of Algorithm \ref{Alg} operates on $\mathcal{U}_R$ for a fixed $R$, and since $\mathcal{U}_{R}\not\subset\mathcal{U}_{R+1}$ it is possible that an optimal signal in $\mathcal{U}_R$ may be better than any signal in $\mathcal{U}_{R+1}$ for certain values of $R$.
In example \ref{ex:3D_robot}, $\Omega_R$ was randomly generated which required averaging several trials. 
This had a smoothing effect on the performance curves which is analogous to the effect of averaging the outputs of \RRTs and \SST. 

A very important distinction between the \GLC and the other methods \RRTs and \SST is that the \GLC algorithm  must run to completion before a solution is returned, while \RRTs and \SST are incremental algorithms that can be interrupted at any time and return the current best solution.
This is a desirable feature for real-time planning and will be an important next step in the development of this approach.
While, \SST offers similar theoretical guarantees to the \GLC algorithm, the difference in running time is several orders of magnitude in all of the trials making the \GLC preferable under most circumstances.

A final remark is that the \GLC algorithm suffers from the curse of dimensionality like all other motion planning algorithms. 
The PSPACE-hardness of these problems suggests this issue will never be resolved. 
The examples demonstrate the \GLC algorithm on state spaces of dimension two to six, and we observe that with increasing dimension, the time required to obtain visually acceptable trajectories increases rapidly with dimension.
Example \ref{ex:3D_robot} had six states and required one to five seconds to produce visually acceptable trajectories. This suggests that six states is roughly the limit for real-time applications.

\bibliographystyle{amsalpha}
\bibliography{chap5}

\chapter{Admissible Heuristics for Optimal Kinodynamic Motion Planning}\label{chap:admissible_heuristics}

Many graph search problems arising in robotics and artificial intelligence that would otherwise be intractable can be solved efficiently with an effective heuristic informing the search.
However, efficiently obtaining a shortest path on a graph requires the heuristic to be admissible as described in the seminal paper introducing the $\rm A^*$ algorithm~\cite{hart1968formal}. In short, an admissible heuristic provides an estimate of the optimal cost to reach the goal from every vertex, but never overestimates the optimal cost.
A good heuristic is one which closely underestimates the optimal cost-to-go from every vertex to the goal.
The workhorse heuristic in shortest path problems is the Euclidean distance from a given state to the goal.
This heuristic is admissible irrespective of the obstacles in the environment since, in the complete absence of obstacles, the length of the shortest path from a particular state to the goal is the Euclidean distance between the two points. 
Returning to example \ref{ex:kine} of the last chapter, using the Euclidean distance as a heuristic in the \GLC algorithm reduces the number of iterations required by $84\%$ in comparison to a uniform cost search.
Figure \ref{fig:intro_demo} shows a side-by-side comparison of the region of \xfree explored by the informed search versus the uniform cost search.

The value of admissible heuristics for kinodynamic planning has already been identified in a number of works where admissible heuristics for individual problems have been derived and used to reduce the running time of planning algorithms~\cite{informedRRT,batchInformedRRT,hauser2015asymptotically,paden2016generalized}. 
The \GLC method is among these. 

While admissibility of a heuristic is an important concept it gives rise to two challenging questions in the context of kinodynamic motion planning.
First, without a priori knowledge of the optimal cost-to-go, how do we verify the admissibility of a candidate heuristic; and second, how do we systematically construct good heuristics for kinodynamic motion planning problems?  
The goal of this chapter is to present a principled study of admissible heuristics for kinodynamic motion planning which answers these questions.

\begin{figure}
	\centering{}
	\includegraphics[width=1.0\columnwidth]{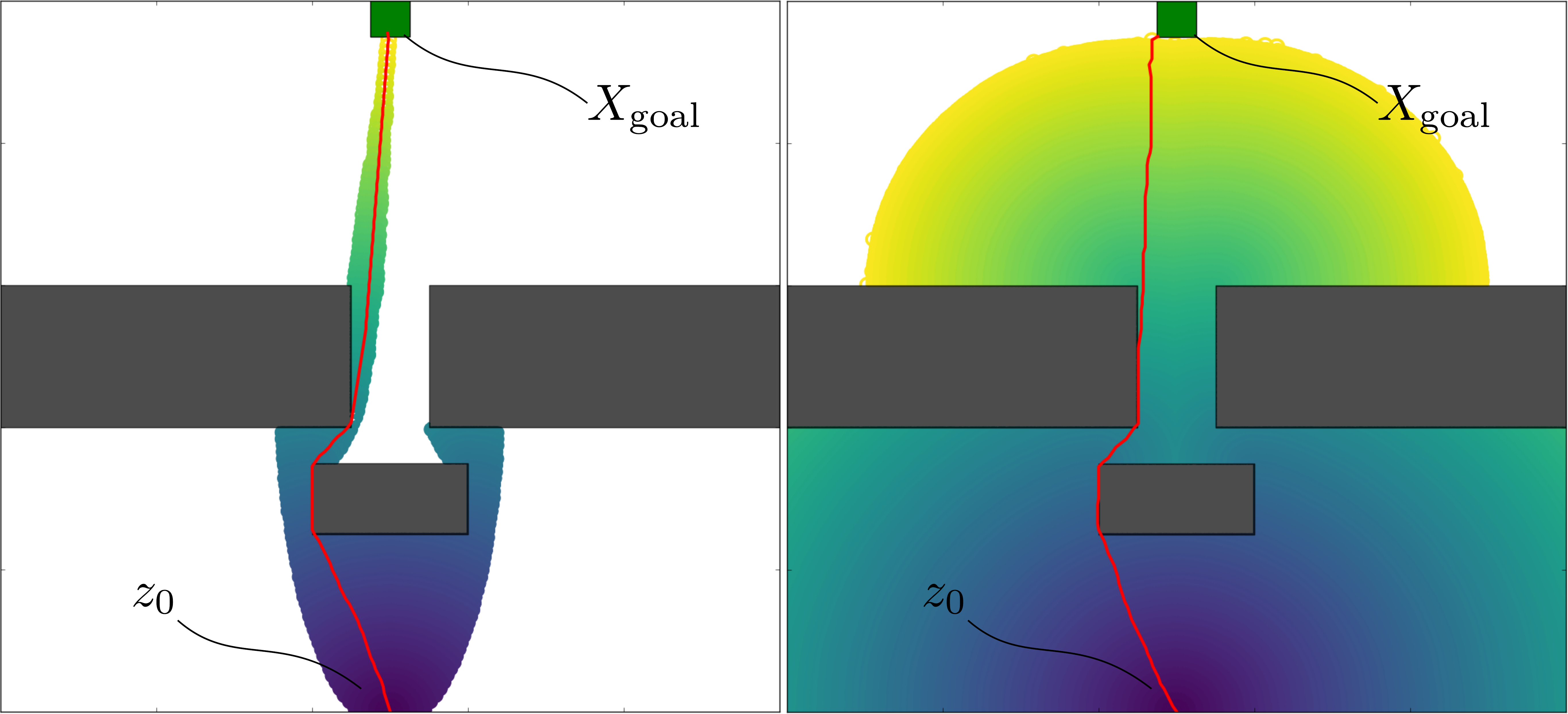}
	\caption{Colored markers indicate the terminal state of a trajectory evaluated by the \GLC algorithm with color indicating the cost to reach that state. The left tile shows the region explored by the informed \GLC algorithm while the right tile is result of the uniform cost search.}\label{fig:intro_demo}
\end{figure}

\paragraph{The Value Function:}
The optimal cost-to-go or \emph{value function}\index{value function} $V:\xfree\rightarrow \mathbb{R}$ is central to this discussion and describes the greatest lower bound on the cost to reach the goal set from the initial state $x_0\in \xfree$.
That is,
\begin{equation}\label{eq:value_func}
V(x_0)=\inf_{u\in \Ugoal} \left\{ J_{x_{0}}(u) \right\},
\end{equation}
where \Ugoal and \Xgoal are redefined (with some abuse of notation) with respect to an initial condition $x_0$.
It is important to note that this definition of $V$ implies that $V$ is well defined and unique.
The following properties of $V$ follow immediately from the assumption $g(z,w) \geq 0$ in (\ref{eq:cost}):
\begin{equation}
V(z) \geq 0\quad \forall z \in X_\mathrm{free}, \qquad
V(z) = 0 \quad \forall z \in X_\mathrm{goal}.
\label{eq:positive_edge_weight}
\end{equation} 
It is well known in optimal control theory that when the value function is differentiable\footnote{The equivalence holds even when the value function is not differentiable if a generalized solution concept known as a \emph{viscosity solution}\index{viscosity solution} is used~\cite{crandall1983viscosity}.}, obtaining $V(x_0)$ in \eqref{eq:value_func} for every $x_0$ in \xfree is equivalent to solving the celebrated \emph{Hamilton-Jacobi-Bellman equation}\index{Hamilton-Jacobi-Bellman equation},
\begin{equation}\tag{HJB}\label{eq:HJB}
\underset{w \in \Omega}{\inf} \left\{ \left\langle \nabla V(z), f(z,w) \right\rangle + g(z,w) \right\}=0, \quad\forall z\in X_\mathrm{free}\setminus \bar{X}_\mathrm{goal},
\end{equation}
with the boundary condition $V(z)=0$ for all $z$ in the closure of $X_\mathrm{goal}$. 
In general solving this equation is equivalent to solving Problem \ref{Problem} from every initial condition in \xfree which is considerably more difficult.
\section{Graph-Based Approximations}\label{sec:admissible_heuristics}
Most kinodynamic motion planning algorithms, the \GLC algorithm included, structure an approximation of the problem as a graph $(\mathcal{V},E)$ with paths in the graph corresponding to trajectories satisfying the dynamic model \eqref{eq:int_dynamics}.  
Conceptually, shortest paths on the graph are in some sense faithful approximations of optimal feasible trajectories for the problem. 
This was analyzed for the \GLC algorithm in Section \ref{sec:consistency}. 

The non-negativity of the cost function (\ref{eq:cost}) enables a nonnegative edge-weight to be assigned to each edge corresponding to the cost of the trajectory in relation with that edge.
The approximated problem can then be addressed using shortest path algorithms for graphs.

The value function $\hat{V}:\mathcal{V}\rightarrow \mathbb{R}$ on the weighted graph is analogous to the value function $V$ in the original problem.
For a vertex $z$ in the graph, $\hat{V}(z)$ is the cost of a shortest path to one of the goal vertices: $\mathcal{V}\cap X_\mathrm{goal}$.  
Since the feasible trajectories represented by the graph are a subset of the feasible trajectories of the problem we have the inequality
\begin{equation}\label{eq:restriction}
V(z)\leq \hat V(z)\qquad \forall z \in \mathcal{V}.
\end{equation}

\subsection{Admissible Heuristics}
Recall that the \GLC algorithm utilized a priority queue which evaluated candidate paths in a graph in order of their cost in \eqref{eq:queue}.
Evaluating candidates in some order of merit is known as a best first search.
This is accomplished with an operation which is traditionally called "$\mathtt{pop}$" in software libraries\footnote{This operation is technically not a function in the usual sense. Rather, it is a relation.}.
When we have a \emph{heuristic}\index{heuristic} $H:\xfree\rightarrow \mathbb{R}$ that estimates the remaining cost to reach the goal, the ordering of the priority queue is modified so that the $\mathtt{pop}$ operation returns an element of the queue $Q$ satisfying,
\begin{equation}
\mathtt{pop}(Q)\in \underset{u\in Q}{\rm argmin} \{J_{x_{ic}}(u)+ H([\varphi_{x_{ic}}(u)](\tau(u))) \}.
\end{equation} 
This orders the priority queue according to the cost of the input signal $u$ plus the estimate of the remaining cost to reach the goal from the terminal state of the trajectory associated to the input signal $[\varphi_{x_{ic}}(u)](\tau(u)))$.

A heuristic $H$ for a problem with value function $V$ is \emph{admissible}\index{admissible heuristic} if it satisfies
\begin{equation}\label{eq:admissible}
H(z)\leq {V}(z)\qquad \forall z \in X_\mathrm{free}.
\end{equation}
When a heuristic is admissible, the first signal found to be in \Ugoal is an optimal signal within the approximation \UR.
In light of (\ref{eq:restriction}), an admissible heuristic for the kinodynamic motion planning problem will also be admissible for any graph-based approximation to the problem.
For the remainder of this chapter, the set of candidate heuristics will be restricted to differentiable scalar valued functions on ${X}_{\rm free}$. 

Another important concept related to heuristics is referred to as consistency\footnote{This usage of the term consistency is distinct from the usage in chapter \ref{chap:approximation}. This is an unfortunate double usage of the term in robotics.}. 
Consistency of a heuristic is similar to the triangle inequality for metric spaces (cf. Appendix \ref{app:topo}). 
To define consistency, the value function for the kinodynamic motion planning problem must be parametrized by the goal set. 
This is denoted $V(z;X_\mathrm{goal})$.
A heuristic $H$ is \textit{consistent}\index{consistent heuristic} if,
\begin{equation}\label{eq:consistent}\arraycolsep=1.4pt\def\arraystretch{1.5}
\begin{array}{l}
H(z)=0,\quad \forall z\in X_\mathrm{goal},\\
H(z) \leq V(z;\{y\}) + H(y), \quad \forall y,z\in X_\mathrm{free}.
\end{array}
\end{equation}
Note that the inequality above involves the optimal cost-to-go from $z$ to $y$. 
While a best first search with an admissible heuristic will return an optimal solution, if the heuristic is not consistent, then the search may not benefit from using the heuristic in terms of iterations required to terminate the search. 

Since the value function is unknown it is difficult to check that (\ref{eq:admissible}) and \eqref{eq:consistent} are satisfied for a particular heuristic $H$.
The next result characterizes a convex subset of admissible heuristics. This sufficient condition for admissibility can be checked directly from the problem data.
\begin{theorem}\label{thm:admissibility}
	A heuristic $H$  is admissible if the following two conditions are satisfied:
	\begin{equation}\tag{AH1}\label{eq:AH1}
	H(z) \leq 0 \qquad \forall z \in  X_\mathrm{goal},
	\end{equation}
	\begin{equation}\tag{AH2}\label{eq:AH2}
	\left\langle \nabla H(z),f(z,w)\right\rangle +g(z,w) \geq 0,\quad \forall z\in \xfree\setminus cl(\xgoal),\,\,{\rm and}\,\, \forall w\in \Omega.
	\end{equation} 
	Further, the heuristic is consistent if 
	\begin{equation}\tag{CH1}\label{eq:ch1}
		H(z) = 0 \qquad \forall z \in  X_\mathrm{goal},
	\end{equation}
\end{theorem}
\begin{proof}
	(admissibility) Let $u$ be an input signal resulting in a feasible trajectory $x$ terminating in the goal.
	By construction the terminal state $x(\tau(u))$ is in the goal set $\xgoal$. Therefore, $H(x(\tau(u)))\leq 0$ by (\ref{eq:AH1}). 
	This implies the inequality
	\begin{equation}
	H(x(0))  \leq  H(x(0))-H(x(T)).
	\end{equation}
	Then by the fundamental theorem of calculus 
	\begin{equation}
	H(x(0)) \leq -\int_{0}^{\tau(u)} \frac{d}{dt} H(x(t)) \,\mu(dt).
	\end{equation}
	Applying the chain rule to the integrand yields
	\begin{equation}
	H(x(0)) \leq -\int_{0}^{\tau(u)}  \left\langle \nabla H(x(t)),f(x(t),u(t)) \right\rangle \, \mu(dt).
	\end{equation}
	Since $x$ is feasible, $x(t)\in \xfree$ for all $t\in[0,\tau(u)]$. Directly applying \eqref{eq:AH2}, 
	\begin{equation}
			H(x(0)) \leq \int_{0}^{\tau(u)}  g(x(t),u(t)), \mu(dt).
	\end{equation}
	The right hand side is now simply the cost $J_{x_{ic}}(u)$ associated to the input signal $u$.
		\begin{equation}
		H(x(0)) \leq J_{x_{ic}}(u).
		\end{equation}
Since the choice of $u$ was arbitrary within $\Ugoal$, we conclude that $H(x(0))\leq J_{x_{ic}}(u)$ for every $u\in \Ugoal$ with associated trajectory $x$. Equivalently,
\begin{equation}
H(z)\leq V(z) \qquad \forall z\in \xfree. 
\end{equation} 

(consistency) Let $u$ be an input signal resulting in a trajectory $x$ from $x(0)=z$ and terminating at $x(\tau(u))=y$.
By a similar line of reasoning as for the admissibility argument observe that 
	\begin{equation} \arraycolsep=1.4pt\def\arraystretch{1.5}
	\begin{array}{rcl}
	H(x(0))-H(x(\tau(u)))&=& -\int_{0}^{\tau(u)} \frac{d}{dt} H(x(t)) \,\mu(dt) \\
	&=&  -\int_{0}^{\tau(u)}  \left\langle \nabla H(x(t)),f(x(t),u(t)) \right\rangle \,\mu(dt) \\
	&\leq&  \int_{0}^{\tau(u)}  g(x(t),u(t) \,\mu(dt) \\
	&=& J(x,u).
	\end{array}
	\end{equation} 
	Thus, $H(z)-H(y)$ lower bounds $J_{x_0}(u)$ for any input signal resulting in a trajectory from $z$ to $y$.
	Since $V(\,\cdot \,;\{y\})$ is the greatest lower bound to the cost of such trajectories we have 
	\begin{equation}
	H(z)-H(y) \leq V(z;\{y\}),\qquad \forall y,z\in X_\mathrm{free}.
	\end{equation}
	Rearranging the expression above yields the definition of consistency for $H(\,\cdot\,;X_\mathrm{goal})$.
\end{proof}

It is interesting to note that the admissibility (and consistency) conditions are affine constraints which implies that the subset of admissible heuristics satisfying \eqref{eq:AH1} and \eqref{eq:AH2} are convex. 
\begin{cor}
	The set of candidate heuristics satisfying the admissibility conditions \eqref{eq:AH1} and \eqref{eq:AH2} is convex.
\end{cor}
\begin{proof}
Choose $\lambda\in [0,1]$, and heuristics $H_1$ and $H_2$ satisfying \eqref{eq:AH1} and \eqref{eq:AH2}. 
The convex combination is given by $\lambda H_1+(1-\lambda)H_2$. By \eqref{eq:AH1}, both $H_1(z)$ and $H_2(z)$ are less than or equal to zero for all $z$ in \xgoal. Thus, 
\begin{equation}
\lambda H_1(z)+(1-\lambda)H_2(z)\leq 0\qquad \forall z\in \xgoal.
\end{equation}
Then inserting the convex combination into \eqref{eq:AH2} yields
\begin{equation}\Stretch\begin{array}{ll}
&\left\langle \nabla (\lambda H_1(z)+(1-\lambda) H_2(z)),f(z,w)\right\rangle +g(z,w) \\
= &\lambda \left\langle \nabla H_1(z),f(z,w)\right\rangle +(1-\lambda) \left\langle  \nabla H_2(z),f(z,w)\right\rangle +g(z,w)\\
=& \lambda \left(\left\langle \nabla H_1(z),f(z,w)\right\rangle + g(z,w) \right)+(1-\lambda) \left( \left\langle  \nabla H_2(z),f(z,w)\right\rangle +g(z,w)\right)\\
\geq &0
\end{array}
\end{equation}
Since $H_1$ and $H_2$ individually satisfy \eqref{eq:AH2},  $\lambda\geq0$, and $1-\lambda\geq0$, the convex combination of $H_1$ and $H_2$ also satisfies \eqref{eq:AH2}. 
\end{proof}

It is interesting that the set of admissible heuristics is convex, but the fact that this characterization is only a subset of all possible heuristics raises the question of whether there are any good heuristics (close to the value function $V$) contained in this subset.
The next section discusses the usefulness of this set showing that it contains the value function.

\section{Selecting an Admissible Heuristic}\label{sec:inf_lp}
Considering that a good heuristic is one which closely underestimates the value function and that the admissibility conditions of Theorem \ref{thm:admissibility} characterize a set of such underestimates, a natural approach to the selection of a heuristic would be to find an admissible heuristic whose integral on \xfree is maximized\footnote{Note that the measure $m$ must be finite on $\xfree$.}.
Since the constraints are affine and this would be a linear functional of $H$, the heuristic selection is then a (infinite-dimensional) linear program:
\begin{equation}\arraycolsep=1.4pt\def\arraystretch{1.5}
\boxed{
	\begin{array}{rll}
	\underset{H}{\max} & \int_{\xfree}H(z)\,\,m(dz) & \\
	{\rm subject\,to:}  & 		H(z) \leq 0 \qquad \forall z \in  \xgoal,
	& \\
	& \left\langle \nabla H(z),f(z,w)\right\rangle +g(z,w) \geq 0 \\
	& \forall z\in \xfree\setminus cl(\xgoal),\,{\rm and}\,\, \forall w \in \Omega.
	
	\end{array} \label{eq:convex_opt}\tag{LP}}
\end{equation}

To justify this formulation, the next result shows that it is equivalent to the HJB equation and the solution to this linear program is in fact the value function.
\begin{theorem}\label{thm:HJBequivalence}
	A function $V$ is the solution to the linear optimization \ref{eq:convex_opt} if and only if it is the solution to the Hamilton-Jacobi-Bellman equation \ref{eq:HJB}.
\end{theorem}
\begin{proof}
	($\Leftarrow$) Let $V$ be a solution to \eqref{eq:HJB}. 
	Then $V$ is equal to the value function so every admissible heuristic $H$ must satisfy $H(z)\leq V(z)$ on $\xfree$. 
	Therefore, the solution to \eqref{eq:HJB} upper bounds the objective  to \eqref{eq:convex_opt},
	\begin{equation}\label{eq:upper_bound}
	\int_{\xfree} H(z)\,\,m(dz) \leq \int_{\xfree} V(z)\,\,m(dz).  
	\end{equation}	
	
	Next, observe that a solution to \eqref{eq:HJB} satisfies the admissibility conditions \eqref{eq:AH1} and \eqref{eq:AH2}:
	The boundary condition $V(z)=0$ on $cl(\xgoal)$ implies $V$ satisfies \eqref{eq:AH2} and the equation 
	\begin{equation}\label{eq:hjbsol}
	\underset{w \in \Omega}{\inf} \left\{ \left\langle \nabla V(z), f(z,w) \right\rangle + g(z,w) \right\}=0, \quad\forall z\in X_\mathrm{free}\setminus cl({X}_\mathrm{goal}),
	\end{equation}
	implies	
	\begin{equation}
	 \left\langle \nabla V(z), f(z,w) \right\rangle + g(z,w)\geq 0\quad \forall z\in X_\mathrm{free}\setminus cl({X}_\mathrm{goal}),\,{\rm and} \, \forall w\in \Omega,
	\end{equation}
	which is \eqref{eq:AH2}.
	Thus, the upper bound in \eqref{eq:upper_bound} is attained at the solution to \eqref{eq:HJB}.

	($\Rightarrow$) As a point of contradiction, suppose that $H$ is a solution to \eqref{eq:convex_opt} and is not a solution to \eqref{eq:HJB}.
	Then there exists $\tilde{z}$ such that 
	\begin{equation}\label{eq:strict_ineq}
	 \left\langle \nabla H(\tilde{z}), f(\tilde{z},w) \right\rangle + g(\tilde{z},w)> 0\quad \forall w\in \Omega.	   
	 \end{equation}
	 The existence of a point where there is strict inequality comes from the assumption that $H$ does not solve \eqref{eq:HJB}, but satisfies the constraints of \eqref{eq:convex_opt}.
	 By the continuity of $H$, there is a neighborhood $\mathcal{N}_{\tilde{z}}$ of $\tilde{z}$ on which \eqref{eq:strict_ineq} holds. 
	 Let $K$ be a compact subset of $\mathcal{N}_{\tilde{z}}$ and denote the minimum of $H$ attained on this compact set by $\varepsilon$. 
	 Additonally, denote the supremum of $\Vert f(z,w) \Vert_2$ over $K\times \Omega$ by $F$.
	 Then 
	 \begin{equation}\label{eq:F}
	 \left\langle \nabla H(z), f(z,w) \right\rangle + g(z,w)-\varepsilon \geq 0 \quad \forall z\in K,\,{\rm and} \,\forall w\in \Omega.	   
	 \end{equation}
	 Next, let $\Psi:\xfree\rightarrow \mathbb{R}^n$ be a smooth, nonnegative bump function supported on $K$ satisfying $\Psi(x) \leq \varepsilon$, and $\Vert \nabla \Psi(x) \Vert_2\leq \min\{1,\varepsilon/ F\}$.
	 Perturbing $H$   by $\Psi$ yields
	 \begin{equation}\label{eq:bumb}\Stretch
	 \begin{array}{ll}
	 &\left\langle \nabla (H(z)+\Psi(z)), f(z,w) \right\rangle + g(z,w)\\
	 =&	\left\langle \nabla H(z), f(z,w) \right\rangle +  \left\langle \nabla \Psi(z), f(z,w) \right\rangle + g(z,w)  \\  
	 \geq & \left\langle \nabla H(z), f(z,w) \right\rangle -  \Vert \nabla \Psi(z)\Vert_2 \Vert f(z,w) \Vert_2 + g(z,w)\\ 
	 \geq & \left\langle \nabla H(z), f(z,w) \right\rangle -   \min\{1,\varepsilon/F\}\cdot F + g(z,w)\\
	 \geq & \left\langle \nabla H(z), f(z,w) \right\rangle -   \varepsilon + g(z,w), \quad \forall z\in K,\,{\rm and} \,\forall w\in \Omega.	 	   
	 \end{array}
	 \end{equation}
	 Thus, the perturbation by $\Psi$ leaves the heuristic admissible by \eqref{eq:F} and this perturbation is a strict improvement to the objective in \eqref{eq:convex_opt} a contradiction since $H$ solves \eqref{eq:convex_opt}.
\end{proof}
Thus, \eqref{eq:HJB} and (\ref{eq:convex_opt}) can both be solved to obtain the exact optimal cost-to-go for a problem. 
This equivalence between \eqref{eq:HJB} and \eqref{eq:convex_opt} is a stronger version of the result than can be found in the published work~\cite{paden2017verification}.
While the HJB equation is a nonlinear partial differential equation, the problem (\ref{eq:convex_opt}) is a linear program lending itself to the methods of convex analysis.
%

\subsection{Problem Relaxations}
In many applications the set $\xfree$ is not entirely known a priori. 
This is particularly true when the heuristic is computed off-line and used in a real-time application where a perception system constructs or modifies $\xfree$ for the current task (e.g. detecting obstacles in a robot workspace).
This consideration motivates the following observation: suppose $\xfree$, $\xgoal$, $\Omega$, $f$, and $g$ is problem data for problem $P$; and $\tilde{X}_{\rm free}$, $\tilde{X}_{\rm goal}$, $\tilde{\Omega}$, $\tilde{g}$, and $\tilde{f}$ is problem data for problem $\tilde{P}$.
Let $V_P$ and $V_{\tilde{P}}$ denote the optimal cost-to-go for each of these problems.
If the two problems are related by
\begin{equation}\label{eq:relax}\arraycolsep=1.4pt\def\arraystretch{1.5}
\begin{array}{c}
\xfree\subset\tilde{X}_{\rm free},\quad \xgoal \subset \tilde{X}_{\rm goal}, \quad \Omega \subset \tilde{\Omega},\quad f=\tilde{f},\\
g(z,w)\geq \tilde{g}(z,w)\quad\forall z\in \xfree,\,{\rm and}\,w\in\Omega,
\end{array}
\end{equation}
then any feasible trajectory of problem $P$ must also be feasible for problem $\tilde{P}$. Additionally, this trajectory will have the same or lesser cost for $\tilde{P}$.
Therefore, $V_{\tilde{P}}(z) \leq V_P(z)$ for any $z\in \Xfree$. 
We can conclude that an admissible heuristic $H_{\tilde{P}}$ for problem $\tilde{P}$ must also be admissible for problem $P$ since $H_{\tilde{P}}(z)\leq V_{\tilde{P}}(z)$ implies $H_{\tilde{P}}(z)\leq V_{P}(z)$. 
Problem $\tilde{P}$ is referred to as a \emph{relaxation}\index{problem relaxation} of problem $P$.
In the case of an unknown environment, one can derive an admissible heuristic for a relaxed problem which considers only constraints known in advance.
This heuristic remains admissible if $\xfree$ is updated to a smaller set due to perceived obstacles. 
Alternatively, it may be easier to verify the admissibility of a candidate heuristic with (\ref{eq:AH1}) and (\ref{eq:AH2}), or evaluate (\ref{eq:convex_opt}) on a relaxation of a particular problem known a-priori.
This comes at the expense of increasing the gap between the heuristic and the optimal cost-to-go for the actual problem which can make the heuristic less effective for a search-based algorithm but allows for off-line construction.

\section{Sum-of-Squares (SOS) Approximation to \\ Equation (\ref{eq:convex_opt}) }\label{sec:Characterising-Admissibility-as}
One way to tackle (\ref{eq:convex_opt}) in the case of problem data consisting of semi-algebraic sets and polynomials is a SOS programming approximation.
%
%

%
SOS programming~\cite{parrilo2004sum} is a method of optimizing a functional of a polynomial subject to semi-algebraic constraints. 
The technique involves replacing semi-algebraic constraints with sum-of-squares constraints which can then be solved as a semi-definite program (SDP). 
The advantage of this particular approach to approximating (\ref{eq:convex_opt}) is that the result of the SOS program is guaranteed to satisfy the admissibility constraint.
\subsection{Sum-of-Squares Polynomials}
A polynomial $p$ in the ring $\mathbb{R}[z]$ in $n$ variables is said to be a sum-of-squares if it can be written as 
\begin{equation}
p(z)=\sum_{k=1}^{d}q_{k}(z)^{2},\label{eq:-5}
\end{equation}
for polynomials $q_{k}(z)$. 
Clearly, $p(z)\geq0$ for all  $z \in \mathbb R^n$. 
Note also that $p(z)$ is a sum-of-squares if and only if it can be written as 
\begin{equation}
p(z)=\mathfrak{m}(z)^{T}Q\mathfrak{m}(z),\label{eq:-4}
\end{equation}
for a positive semi-definite matrix $Q$ and the vector of   monomials $\mathfrak{m}(z)$ up to degree $d$.
For a polynomial $p$ admitting a decomposition of the form (\ref{eq:-4}) we write $p \in SOS$. 
Equation (\ref{eq:-4}) is a collection of linear equality constraints between the entries of $Q$ and the coefficients of $p(z)$. 
Finding entries of a positive semi-definite $Q$ such that the equality constraints are satisfied is then a semi-definite program (SDP). 
The complexity of finding a solution to this problem using interior-point methods is polynomial in the size of $Q$.
This method of analyzing polynomial inequalities has had a profound impact in many fields. As a result there are a number of efficient solvers~\cite{sedumi,sdpt3} and modeling tools~\cite{sostools,yalmip} available.

\subsection{Optimizing the Heuristic}

To proceed with computing a heuristic using the SOS programming framework, the problem data must consist of polynomials and intersections of semi-algebraic sets. 
Let
\begin{equation}\arraycolsep=1.4pt\def\arraystretch{1.5}
\begin{array}{rcl}
X_{\rm free} & = & \left\{ z \in \mathbb{R}^n:\,h_z(z)\geq 0 \right\}, \\
\Omega   & = & \left\{ w \in \mathbb{R}^m:\,h_w(w) \geq 0 \right\}, 	
\end{array}
\end{equation}
for vectors of polynomials $h_z$ and $h_w$ with "$\geq$" denoting element-wise inequalities.
Assume also that $f$, $g$ and the candidate heuristic $H$ are polynomials in the state and control variables.
Then the admissibility condition (\ref{eq:AH1})
%
is a polynomial inequality. 
To restrict the non-negativity constraint of the heuristic to $X_{\rm free}$ and $\Omega$, we add auxiliary vectors of SOS polynomials $\lambda_{z}(z)\geq0$ and $\lambda_{w}(w)\geq0$ to the equation as 
\begin{equation}\arraycolsep=1.4pt\def\arraystretch{1.5}
\begin{array}{c}
\left\langle \nabla H(z),f(z,w)\right\rangle +g(z,w) 
-\left\langle\lambda_{z}(z), h_{z}(z)\right\rangle-\left\langle\lambda_{w}(w), h_{w}(w)\right\rangle \geq 0, \\
\forall w\in \mathbb{R}^m,\,\,{\rm and}\,\, z\in \mathbb{R}^n,
\end{array}\label{eq:-7}
\end{equation}
which trivially implies the positivity of (\ref{eq:AH1}) over $\xfree$
and $\Omega$. 
When $H$ is a polynomial, the objective in (\ref{eq:convex_opt}) is linear in the coefficients of $H$.
Thus, it is an appropriate objective for an SOS program.
The SOS program which is solved to obtain an admissible heuristic is then 
\begin{equation}\label{eq:sosformula}\arraycolsep=1.4pt\def\arraystretch{1.5}
\begin{array}{rll}
\underset{H,\lambda_{z},\lambda_{w}}{\max}& \int_{\xfree}H(z)\,m(dz)& \\  
{\rm subject\, to:}& H(z)\leq 0, \quad \forall z\in cl({X}_\mathrm{goal}),&\\
&\left\langle \nabla H(z),f(z,w)\right\rangle +g(z,w) -\left\langle\lambda_{z}(z), h_{z}(z) \right\rangle - \left\langle \lambda_{w}(w), h_{w}(w)\right\rangle &\in SOS,\\
& \lambda_{z}(z),\lambda_{w}(w) \in  SOS.&  \\
\end{array} 
\end{equation}

\section{Examples of Heuristics for Kinodynamic Motion Planning}\label{sec:Examples_usage}
The examples of this section demonstrate how to directly apply Theorem \ref{thm:admissibility} to verify admissibility of candidate heuristics as well as the heuristic synthesis optimization (\ref{eq:convex_opt}). 
\subsection{Admissibility of Euclidean Distance in Shortest Path Problems}\label{ex:spp}
In the first example we show how to use Theorem \ref{thm:admissibility} to verify the classic Euclidean distance heuristic for kinematic shortest path problems.

	Consider a holonomic shortest path problem,
	\begin{equation}
	\dot{x}(t)=u(t), 
	\end{equation}
	where $\xfree =  \mathbb{R}^n$, $\xgoal=\{0\}$, and $u(t)\in \{ w\in \mathbb{R}^n :\, \Vert w \Vert = 1 \}$.
	The running cost is given by 
	\begin{equation}
		g(z,w)=1.
	\end{equation}
	We would like to verify the heuristic
	\begin{equation}
	H(z)=\Vert z \Vert.
	\end{equation}
	Applying the admissibility Lemma we obtain
	\begin{equation}  \arraycolsep=1.4pt\def\arraystretch{1.5}
	\begin{array}{rcl}
	\left\langle \nabla H(z),f(z,w)\right\rangle +g(z,w) & = & \dfrac{\left\langle z,w\right\rangle }{\Vert z\Vert}+1\\
	& \geq & \dfrac{-\Vert z\Vert\Vert w\Vert}{\Vert z\Vert}+1\\
	& \geq & -1+1\\
	& = & 0,
	\end{array}
	\end{equation}
	which reverifies the fact that the Euclidean distance is an admissible heuristic for the shortest path problem.
	The crux of this derivation is simply applying the Cauchy-Schwarz inequality. 
	Figure \ref{fig:intro_demo} illustrates the use of this heuristic in an environment with obstacles.
\subsection{Wheeled Robot Example}\label{ex:dubins}
	Consider a simple wheeled robot with state coordinates in $ \mathbb{R}^3$, and whose mobility is described by
	\begin{equation}\label{eq:reed_shepp_car}	
	\dot{p}_1(t) = \cos(\theta(t)),\quad
	\dot{p}_2(t) = \sin(\theta(t)),\quad
	\dot{\theta}(t) = u(t),
	\end{equation}
	where $(p_1(t),p_2(t),\theta(t))$ are individual coordinates of the state.
	Let $X_\mathrm{free}=\mathbb{R}^3$, $X_\mathrm{goal}=\{(0,0,0)\}$, and  $\Omega=[-1,1]$.
	The running cost
	\begin{equation}
	g(p_1(t),p_2(t),\theta(t))=1,
	\end{equation}
	results in a cost equal to the path length in the $x$-$y$ plane.
	We will consider two candidate heuristics: the line segment connecting the $p_1$-$p_2$ coordinate to the origin (discussed in~\cite{dolgov2010path}), and the magnitude of the heading error,
	\begin{equation}\label{eq:unicycle_heuristic}
	H_1(p_1(t),p_2(t),\theta(t))=\sqrt{p_1(t)^2+p_2(t)^2},\quad
	H_2(p_1(t),p_2(t),\theta(t))=\vert \theta(t) \vert.
	\end{equation}
	%
	%
	Clearly, (\ref{eq:AH1}) is satisfied for both heuristics. 
	Then (\ref{eq:AH2}) is once again verified with the Cauchy-Schwarz inequality. 

	\begin{figure}
		\centering{}
		\includegraphics[width=1.0\columnwidth]{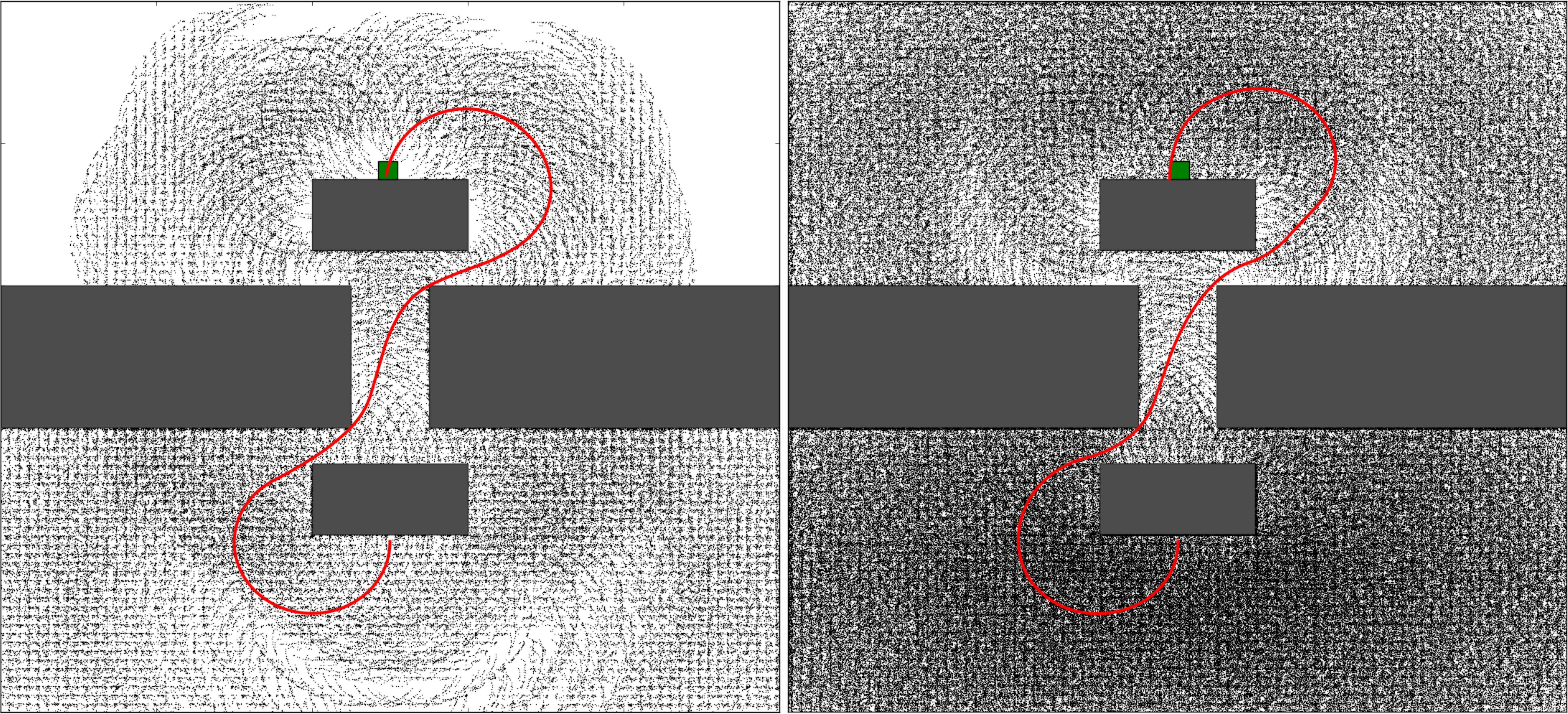}
		\caption{Comparison of the vertices examined by the heuristically guided search (left) and the uniform cost search (right) for a simple wheeled robot motion planning problem. Equation \ref{eq:dubins_heuristic}) was used as the heuristic.}\label{fig:dubins_demo}
	\end{figure}

	For brevity, but with some abuse of notation, the time argument is dropped from the trajectory in the following derivations.
	Inserting the expression for the heuristics into  (\ref{eq:AH2}) yields 
	\begin{equation}\arraycolsep=1.4pt\def\arraystretch{2.0}
	\begin{array}{rcl}
	\left\langle \nabla H_1(p_1,p_2,\theta),f(p_1,p_2,\theta,u)\right\rangle +g(p_1,p_2,\theta,u) &=& \dfrac{\left\langle (p_1,p_2,0),(\cos(\theta),\sin(\theta),u)\right\rangle }{\sqrt{p_1^2+p_2^2}}+1\\
	
	&\geq&-\dfrac{\sqrt{p_1^2+p_2^2} \sqrt{\cos^2(\theta)+\sin^2(\theta)} }{\sqrt{p_1^2+p_2^2}}+1\\
	&=&0.
	\end{array}
	\end{equation}
	Similarly, for $H_2$,
	\begin{equation}  \Stretch
	\begin{array}{rcl}
	\left\langle \nabla H_2(p_1,p_2,\theta),f(p_1,p_2,\theta,u)\right\rangle +g(p_1,p_2,\theta,u) 
	&=& \dfrac{\left\langle (0,0,\theta),(\cos(\theta),\sin(\theta),u)\right\rangle }{|\theta|}+1\\
	&\geq&  -\dfrac{|\theta||u|}{|\theta|}+1\\
	&=&  0.
	\end{array}
	\end{equation}
	Therefore, both heuristics are admissible.
	Note that the maximum of the two heuristics is also an admissible heuristic, 
	\begin{equation}\label{eq:dubins_heuristic}
	H(p_1,p_2,\theta )=\max \{H_1(p_1,p_2,\theta ),H_2(p_1,p_2,\theta )\}.
	\end{equation}

Figure \ref{fig:dubins_demo} illustrates a shortest path query in a simple environment with and without the use of the heuristic in (\ref{eq:dubins_heuristic}).
The heuristically guided search obtains a solution in 81,686 iterations while the uniform cost search requires 403,197 iterations.
In contrast, using just $H_1$  in Equation (\ref{eq:unicycle_heuristic}) as a heuristic requires 104,492 iterations.
\subsection{Autonomous Underwater Vehicle Example}\label{ex:AUV}
	Consider an autonomous underwater vehicle (AUV) navigating a strong current relative to the vehicle's top speed as discussed in~\cite{garau2014path}.
	This scenario is common for long range underwater gliders which travel at roughly $0.5\,\nicefrac[]{m}{s}$ relative to currents traveling at $1.0$-$1.5\,\nicefrac[]{m}{s}$.
	Let the state space  be $\mathbb{R}^n$ ($n=2$ or $3$) representing position in a local Cartesian coordinate system.
	The AUV's motion is modeled by
	\begin{equation}
	\dot{x}(t)=c(x(t))+u(t),
	\end{equation} 
	where $c:\mathbb{R}^n\rightarrow\mathbb{R}^n$ describes the current velocity at each point in the state space. The input $u(t)$ is the controlled velocity relative to the current.
	The thrust constraint is represented by 
	\begin{equation}
	\Omega=\left\{w\in \mathbb{R}^n:\,\Vert w \Vert \leq w_{\rm max} \right\},
	\end{equation}
	where $w_{\max}$ is the maximum achievable speed relative to the current.
	The goal set is $\xgoal=\{0\}$.
	The running cost
	\begin{equation}
	g(x(t),u(t))=1+\Vert u(t)\Vert,
	\end{equation}
	reflects a penalty on the duration of the path as well as the total work done by AUV's motors.
	
	The admissibility of the heuristic proposed in~\cite{garau2005path} can be evaluated using Theorem \ref{thm:admissibility} .
	Let $v_{\max}\coloneqq\underset{z}{\max}\{\Vert c(z)\Vert\}$ be the maximum speed of the current, and consider the heuristic	
	\begin{equation}\label{eq:AUVheuristic}
	H(z)=\frac{\Vert z \Vert}{(v_{\max}+w_{\max})}. 
	\end{equation}
	Clearly (\ref{eq:AH1}) is satisfied. Now evaluate (\ref{eq:AH2}),
	\begin{equation}\Stretch 
	\begin{array}{rcl}
	\left\langle \nabla H(z),f(z,w) \right\rangle+g(z,w)&=&\frac{\left\langle z,c(z)+w \right\rangle}{\Vert z\Vert(v_{\max}+w_{\max})}+1+\Vert w \Vert
	\\
	&\geq& \frac{-\Vert z \Vert \cdot \Vert c(z) \Vert - \Vert z \Vert \cdot \Vert w \Vert }{\Vert z\Vert(v_{\max}+w_{\max})}+1+\Vert w \Vert
	\\
	&\geq&  \frac{- \Vert c(z) \Vert -  \Vert w \Vert }{(v_{\max}+w_{\max})}+1+\Vert w \Vert.
	\end{array} 
	\end{equation}
	Notice that $v_{\max}$ and $w_{\max}$ were defined so that $\Vert w \Vert + \Vert c(z) \Vert \leq v_{\max}+w_{\max}$.
	Thus, 
	\begin{equation}\Stretch 
	\begin{array}{rcl}
	\left\langle \nabla H(z),f(z,w) \right\rangle+g(z,w)&\geq&-1+1+\Vert w \Vert\\
	&\geq&0,
	\end{array} 
	\end{equation}
	so the heuristic is admissible. Figure \ref{fig:AUV_demo} illustrates a numerical example in a 2D environment where the maximum current speed is 2.6 times the AUV's maximum relative speed.
	The heuristically guided search expands $49\%$ fewer vertices than the uniform cost search.
	\begin{figure}
		\centering{}
		\includegraphics[width=0.8\columnwidth]{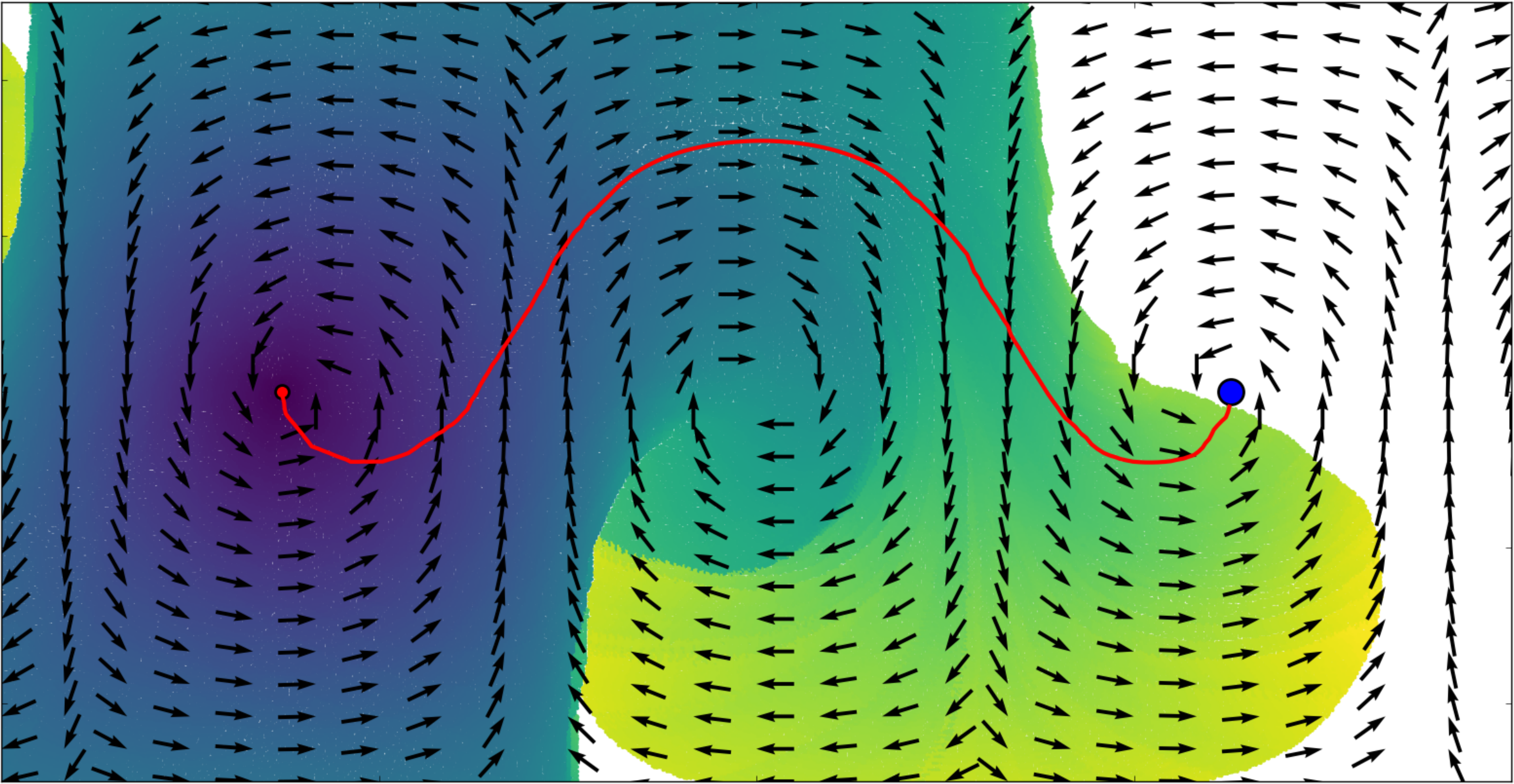}
		
		\vspace{1mm}
		\includegraphics[width=0.8\columnwidth]{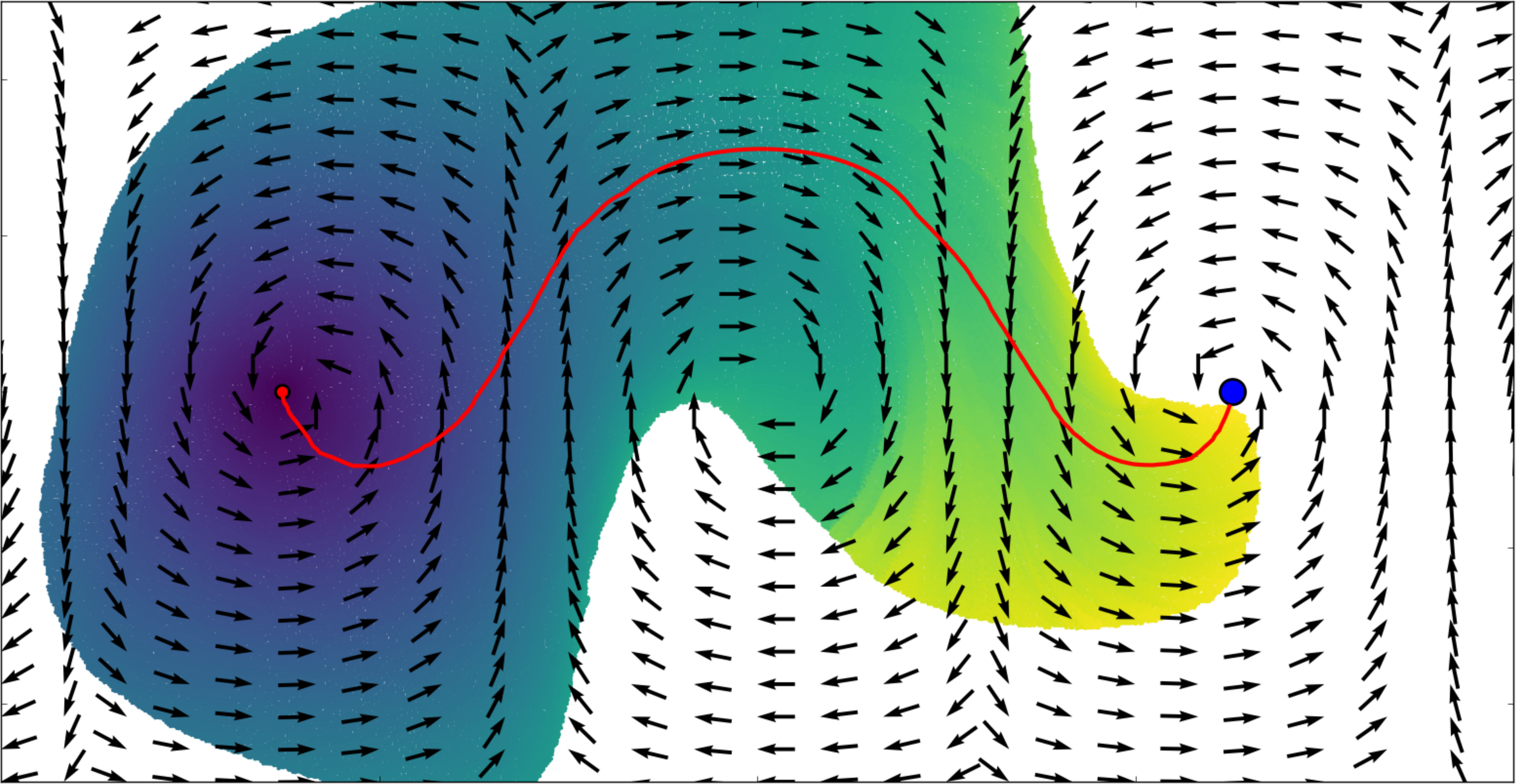}
		%
		\caption{Comparison of a minimum cost path search in a graph approximating feasible paths for an AUV in a strong current. The heuristically guided search using equation (\ref{eq:AUVheuristic}) of Example \ref{ex:AUV}  obtains the solution in 271,642 iterations (left) while the uniform cost search requires 537,168 iterations (right). Initial and final states are illustrated with red and blue markers respectively. Colored regions represent states explored by the algorithm and the color indicates the relative cost to reach that state.}\label{fig:AUV_demo}
	\end{figure}

\section{Sum-of-Square Heuristic Synthesis Examples} \label{sec:Examples}

The last examples demonstrate the SOS programming formulation described in Section \ref{sec:Characterising-Admissibility-as}. 
A closed form solution for the optimal cost-to-go is known for the examples in this section which provides a useful point of comparison for the computed heuristic. 
The example problem was implemented using the SOS module in YALMIP \cite{yalmip} and solved using SDPT3 for the underlying semi-definite program \cite{sdpt3}. 
To further illustrate the approach, YALMIP scripts for these examples can be found in~\cite{SOS_heuristics}.

\subsection{Single Integrator Example} \label{ex:1d_kine}
To illustrate the procedure, we revisit Example \ref{ex:spp} in the 1-dimensional case. 
The differential constraint is given by $f(x,u)=u$ where $x,u\in\mathbb{R},$
$X_\mathrm{free}=[-1,1],$ $\Omega=[-1,1]$, and $X_\mathrm{goal}=\{0\}$. 
Again we use the minimum time objective where $g(x,u)=1$.
The value function $V(x)=|x|$ is obtained by inspection.
The heuristic is parameterized by the coefficients of a univariate polynomial of degree $2d$
\begin{equation}
H(x)=\sum_{i=0}^{2d}c_{i}x^{i}.
\end{equation}
Using a discrete measure on $[-1,1]$ concentrated at the boundary the SOS program is, 
\begin{equation}\arraycolsep=1.4pt\def\arraystretch{1.5}  \label{eq:simple_demo}
\begin{array}{rll}
\underset{H,\lambda_{x},\lambda_{u}}{\max} & \left\{ H(1)+H(-1)\right\} & \\
{\rm subject\, to:} &H(0)=0, & \\
&\left(\frac{d}{dx}H(x)\right)u+1-\lambda_{x}(x)(1-x^{2})-\lambda_{u}(u)(1-u^{2}) \in SOS,&\\
&\lambda_{x}(x),\lambda_{u}(u) \in SOS.&
\end{array}
\end{equation}
The numerical solution for polynomial heuristics with increasing degree is shown in Figure \ref{fig:1dheuristic}. 
\begin{figure}
	\centering{}\includegraphics[width=0.7 \columnwidth]{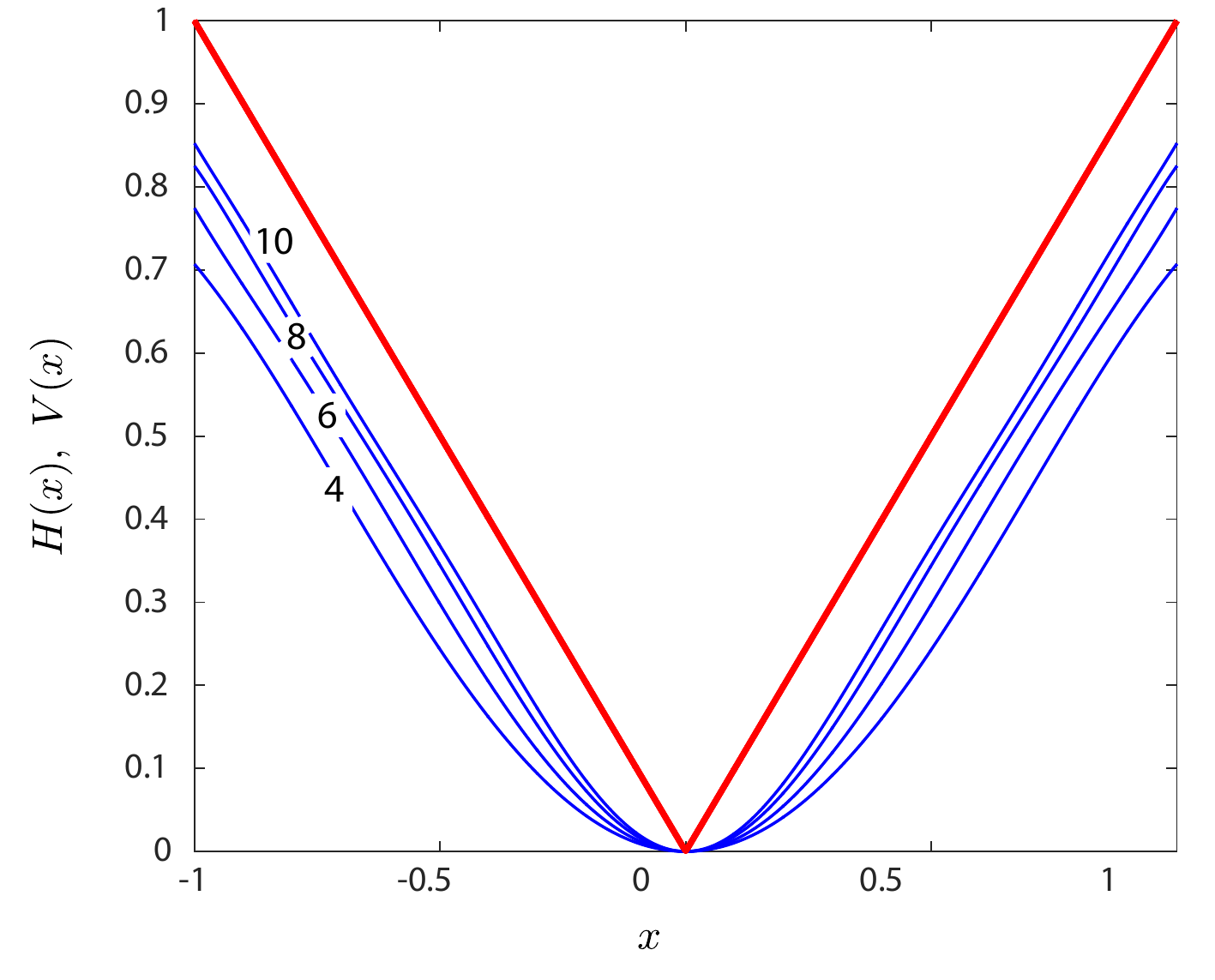}\caption{Univariate polynomial heuristics of degrees 4, 6, 8, and 10 for the 1D single integrator shown in blue. The value function is shown in red. Polynomial heuristics with higher degree provide better underestimates of the value function. \label{fig:1dheuristic}}
\end{figure}

\begin{figure}
	\centering
	\includegraphics[width=0.7\columnwidth]{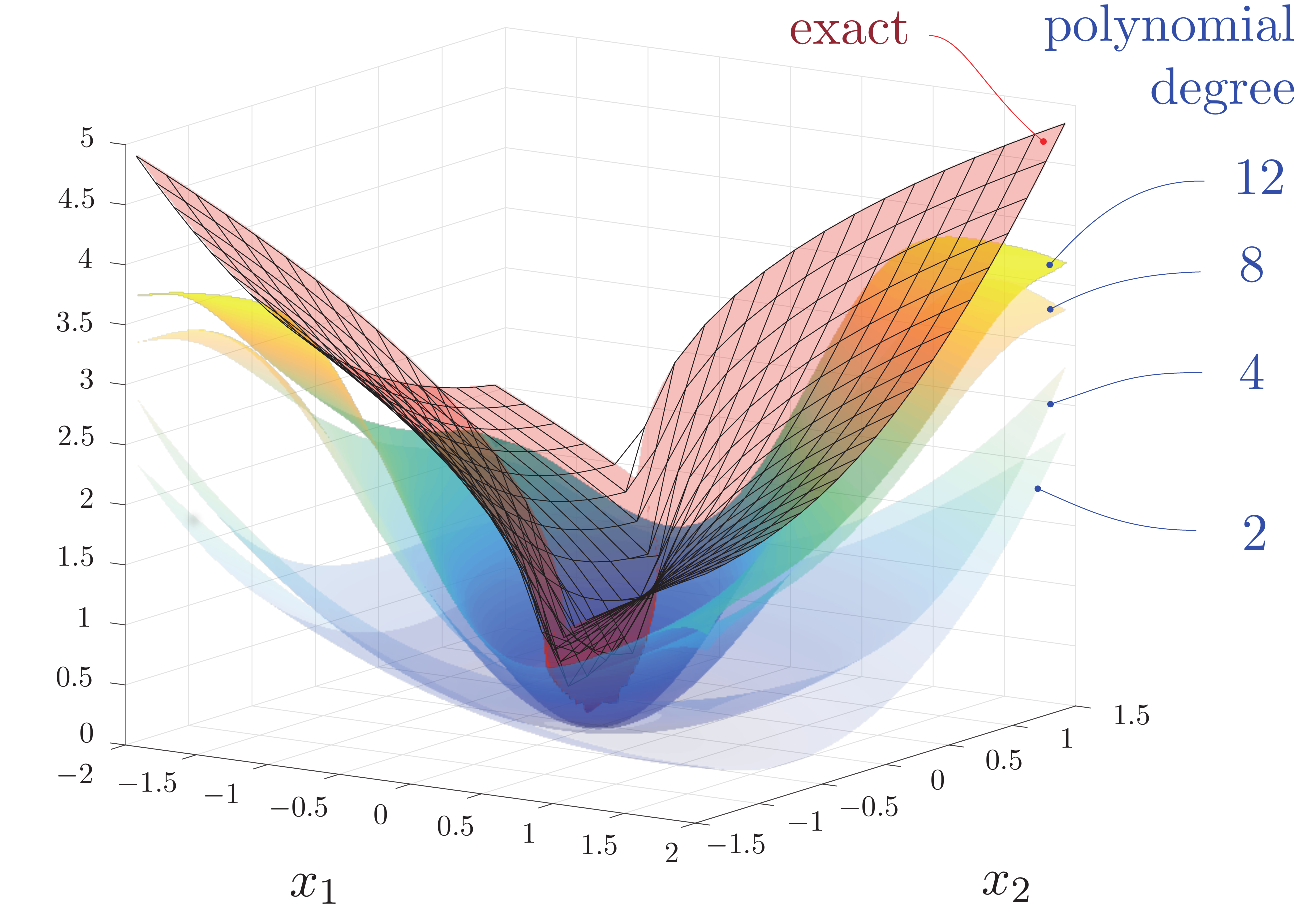}
	\includegraphics[width=0.7\columnwidth]{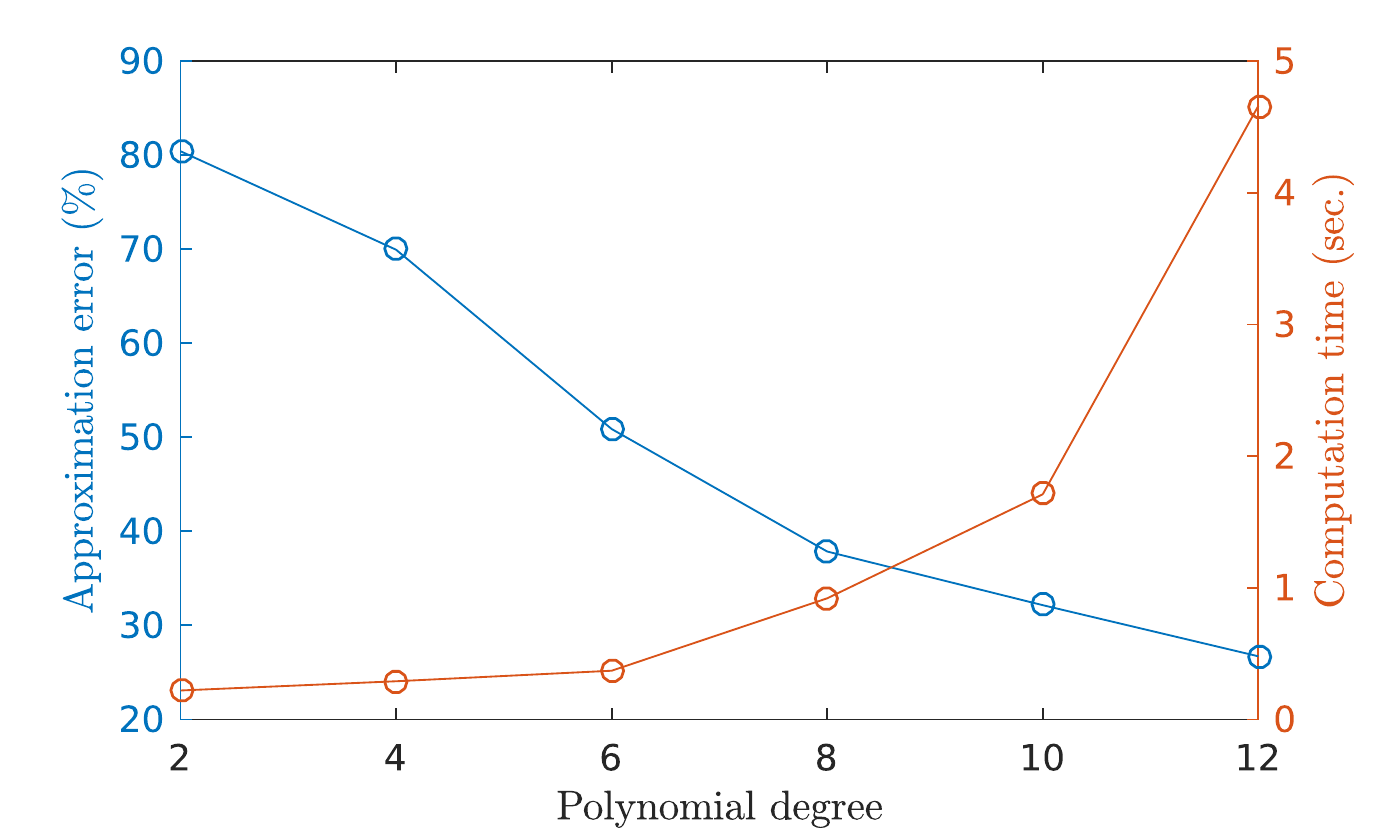}
	
	\caption{(top) Polynomial heuristics of degree 2, 4, 8, and 12 for the double integrator model in comparison with the known value function shown in red. (bottom) Running time and approximation accuracy to compute heuristics of degrees from 2 to 12. } \label{fig:1d_dbl}
\end{figure}

\subsection{Double Integrator Example}\label{ex:2dInt}
	
	Consider a double integrator model,
	\begin{equation}
	\dot{x}_1(t)=x_2(t),\quad \dot{x}_2(t)=u(t),
	\end{equation}
	with the minimum time running cost
	\begin{equation}
	g(x_1(t),x_2(t),u(t))=1,
	\end{equation}
	free space $\xfree=[-3,3]^2$, input space $\Omega = [-1,1]$, and goal set $\xgoal= \{0\}$. 
	%
	%
	
	%
	Polynomial heuristics of degree $2d$ of the form
	\begin{equation}
	H(x_1, x_2) = \sum_{p + q\leq 2d} c_{p, q}\; x_1^p x_2^q,
	\end{equation}
	are computed with the measure $m$ in \eqref{eq:convex_opt} supported on 
	$
	[-2,2] \times [-\sqrt 2, \sqrt 2]
	$. 
	This focuses the optimization in a region around the goal while maintaining admissibility of the heuristic over all of $X_\mathrm{free}$.
	The SOS approximation of \eqref{eq:convex_opt} is formulated as follows
	\begin{equation}
	\arraycolsep=1.4pt\def\arraystretch{1.5} 
	\begin{array}{rl}
	\underset{H,\lambda_{x},\lambda_{u}}{\max}&\int_{\xfree}H(x_{1},x_{2})\,\, m(dx) \\ 
	{\rm subject\, to:} & H(0,0)\leq 0,\\ 
	&\left\langle \nabla H(x_{1},x_{2}),(x_2,u)\right\rangle+1 \\
	&-\lambda_{x_{1}}(x_{1})\left(9-x_{1}^{2}\right)  \\
	&-\lambda_{x_{2}}(x_{2})\left(9-x_{2}^{2}\right)  \\
	&-\lambda_{u}(u)(1-u^{2}) \in SOS,\\
	&\lambda_{x_{1}}(x_{1}),\lambda_{x_{2}}(x_{2}),\lambda_{u}(u) \in SOS.
	\end{array} 
	\end{equation}
	Note that the polynomial $H$ can be integrated over the rectangular region $S$ in closed form with standard integration rules or with the integration functionality in YALMIP.

	The optimized heuristics of increasing degree are shown in Figure \ref{fig:1d_dbl} together with the value function for $X_\mathrm{free}=\mathbb{R}^n$. 
	The percent error in Figure \ref{fig:1d_dbl} is defined as 
	\begin{equation}
	{\rm \%\, error}\, =100\cdot \int_{\Xfree} \frac{V(z)-H(z)}{V(z)}m(dz).
	\end{equation}
\subsection{Observations}
The optimization of Example \ref{ex:2dInt} focused on the region $[-2,2] \times [-\sqrt 2, \sqrt 2]$ instead of $[-3,3]^2$.
The reason is that some states in $[-3,3]^2$ cannot be reached and have infinite cost.
%
%
A remarkable observation is that the resulting SOS program does not admit a maximum when the integral includes a subset of $X_\mathrm{free}$ where the value function is unbounded.   
This is consistent with the theoretical results since the heuristic is free to go unbounded over this set as well.
Therefore, it is important to optimize the heuristic over a region where the optimal cost-to-go is known to be bounded. 

While semi-definite programs run in polynomial time in problem size, there is combinatorial growth in the number of monomials with increasing degree resulting in large semi-definite programs.
For an $n$-dimensional state space, the number of monomials of degree $d$ or less is 
\begin{equation}
\left(\begin{array}{c}
n+d\\
n
\end{array}\right).
\end{equation}
This rapid growth in problem size is one of the current limitations of this approach to (\ref{eq:convex_opt}). 
However, the technique produces a reasonable polynomial approximation that is guaranteed to be admissible.

There are a number possibilities for improving this technique.
The DSOS and SDSOS~\cite{ahmadi2014dsos} programming techniques which are being developed will enable the derivation of heuristics with higher degree polynomials.
Additionally, a straight forward solution is to apply finite difference methods to directly approximate (\ref{eq:convex_opt}) as a finite-dimensional linear program, but the strict admissibility is lost in this case.

\bibliographystyle{amsalpha}
\bibliography{chap6}

\chapter{Conclusion}\label{chap:conclusion}
This thesis presented a direct, forward search algorithm for approximate optimal kinodynamic motion planning called the generalized label correcting method. 
The contribution of this method to robotics is that it is resolution complete, does not rely on abstract subroutines which may be unavailable in practice, and the assumptions placed on problem data are readily verified.
In Chapter \ref{chap:topo}, techniques from topology and functional analysis were used to establish that the set of solutions to the problem is open and the performance objective is continuous. 
These properties were the foundation for the numerical approximation of the problem in Chapter \ref{chap:approximation}, as well as the derivation of the generalized label correcting conditions in Chapter \ref{chap:glc}.
The generalized label correcting method is a slight variation on canonical label correcting methods, making the method familiar and easily implemented.
Numerical experiments tested the algorithm on a wide variety of motion planning problems demonstrating that, in addition to the theoretical guarantees, the method performs well in practice.

Like a best-first search, the generalized label correcting method can utilize an admissible heuristic to reduce the running time of the algorithm.
This motivated the study of admissible heuristics for kinodynamic motion planning in Chapter \ref{chap:admissible_heuristics}, which resulted in tools for the verification and synthesis of admissible heuristics. 
The main result of Chapter \ref{chap:admissible_heuristics} was an admissibility condition which is readily verified for a candidate heuristic and which characterizes a convex set of admissible heuristics. 
Additionally, the synthesis of an admissible heuristic was formulated as an infinite-dimensional linear program whose feasible set consisted of admissible heuristics and whose exact solution is equal to the optimal cost-to-go. 
Sum-of-squares programming was then proposed as a relaxation of this infinite-dimensional convex program which provides the best provably admissible heuristic within a finite-dimensional subspace of polynomials.

\section{Extensions and Improvements}

There are several directions of future research that will improve the utility of these methods. 
First and foremost, an anytime variation is needed to maximize effectiveness in real-time applications. 
Such a variation was omitted from this thesis to avoid obscuring the simplicity of the method.
Secondly, the partitioning of the state space in Section \ref{sec:glc_def} is admittedly simple and can be improved with additional information about the dynamical system. 
For example, the Ball-Box theorem from differential geometry can be leveraged to determine the most appropriate asymptotic scaling of the partition.

One limitation of the generalized label correcting algorithm as it was presented is that it was constructed under very weak assumptions. 
While this is an attractive property of the method, it leaves information in more structured problem instances unused.
It is possible that, with added assumptions, sharper pruning conditions can be derived by similar arguments to those presented in Chapter \ref{chap:glc}. 

Equipped with the understanding of the set of feasible trajectories, an exciting direction for future research is the application of this approach to more complex logical specifications.
Simply adding specifications with boolean logic would greatly increase the expressiveness of the generalized label correcting method.
However, more sophisticated modal logics, such as linear temporal logic, could be considered as well.

\bibliographystyle{amsalpha}
\appendix
\chapter{Review of Set Theory} \label{app:sets}

%
%
A \emph{set} is a collection of logical variables which are called the \emph{elements} of a set. 
A set is entirely determined by its elements. 
Shorthand for the sentence "$x$ is an element of $S$" is "$x\in S$".

\section{Notation}

Sets can be defined implicitly by the variables $x$ satisfying a logical formula $\mathcal{P}(x)$. 
Such a set is denoted $\{x:\,\mathcal{P}(x)\}$. 
Complex formulas can be constructed from atomic formulas using the logical connectives; "and ($\land$)", "or ($\vee$)", "not ($\neg$)", "implies ($\Leftarrow$)", and "equivalent to ($\Leftrightarrow$)"; and the quantifiers "there exists ($\exists$)", and "for all ($\forall$)".
Let $S$ and $U$ be sets.
The set $S$ is said to be s \emph{subset}\index{subset} of $U$, denoted  $S \subset U$, if every element of $S$ is also an element of $U$:
\begin{equation}
(S\subset U)\Leftrightarrow (x\in S \Rightarrow x\in U). 
\end{equation}
The set $S$ is \emph{equal}\index{equal} to $U$ if $S$ is a subset of $U$ and $U$ is a subset of $S$:
\begin{equation}
(S = U)\Leftrightarrow (S\subset U \land U \subset S).
\end{equation}
It is often useful to consider the set of all subsets of a set $S$.
This set, denoted $2^S$, is called the \emph{power set}\index{power set} of $S$,
\begin{equation}
2^S=\{U:\,U\subset S\}.
\end{equation}

A set which contains no elements is called the \emph{empty set}\index{empty set}, denoted $\emptyset$.

\section{Operations on Sets}
The \emph{complement}\index{complement} of $S$, denoted $S^c$, is the set whose elements are not elements of $S$:
\begin{equation}
S^c=\{x:\, \neg(x\in S)\}.
\end{equation} 
The \emph{union}\index{union} of $S$ and $U$, denoted $S\cup U$, is the set whose elements are an element of $S$ or $U$:
\begin{equation}
S\cup U=\{x:\,x\in S \vee x\in U\}.
\end{equation}
The \emph{intersection}\index{intersection} between $S$ and $U$, denoted $S\cap U$, is the set whose elements are elements of $S$ and $U$:
\begin{equation}
S\cap U = \{ x: x\in S \land x\in U \}.
\end{equation}  
Two sets $S$ and $U$ are said to be \emph{disjoint}\index{disjoint} if their intersection is empty, $S\cap U =\emptyset$.
The \emph{exclusion}\index{set exclusion} of $U$ from $S$, denoted $S\setminus U$, is the set whose elements belong to $S$ and do not belong to $U$:
\begin{equation}
	S\setminus U = \{x:\,x\in S \land (\neg(x\in U))  \}.
\end{equation}

There are a number of useful identities that follow directly from these operations. These are: associativity of intersection and union operations, $(A\cup B)\cup C = A\cup(B\cup C)$ and $(A\cap B)\cap C = A\cap(B\cap C)$; commutativity of intersection and union operations, $A\cup B= B\cup A$ and $A\cap B = B\cap A$; distributivity $A\cup (B\cap C)=(A\cup B)\cap(A\cup C)$ and $A\cap (B\cup C)=(A\cap B)\cup(A\cap C)$; DeMorgan's law, $(A \cup B)^c=A^c\cap B^c$; and $A\setminus B = A\cap B^c$.

A \emph{cover}\index{cover} of a set $S$ is a collection of sets whose union contains $S$ as a subset. 
That is, $\{U_i\}_{i\in\mathcal{I}}$ is a cover of $S$ if 
\begin{equation}
S\subset \bigcup_{i\in \mathcal{I}} U_i.
\end{equation} 
If $\{U_i\}_{i\in\mathcal{I}}$ is a cover of $S$, then a collection of subsets $\{U_i\}_{i\in\mathcal{J}}$ is a \emph{subcover}\index{subcover} of   $\{U_i\}_{i\in\mathcal{I}}$ if $\mathcal{J}\subset \mathcal{I}$ and
\begin{equation}
S\subset \bigcup_{i\in\mathcal{J}} U_i.
\end{equation}
A cover $P$ of $S$ composed of pair-wise disjoint subsets is called a \emph{partition}\index{partition} of $S$.

\section{Binary Relations}\label{app:relations}
Given two sets $U$ and $W$, a \emph{binary relation}\index{binary relation} $R$ between $U$ and $W$ is a subset of $U\times W$ identifying the elements of $U$ related to elements of $W$. 
If $u\in U$ is related to $w\in W$, then $(u,w)\in R$ This is sometimes alternatively denoted $uRw$.
As an example, suppose $S$ is a set of books. 
A binary relation $R$ on $S\times S$ could be ordered pairs of books where the first was published before the second. 
Then for $x,y\in S$, if $(x,y)\in R$, the statement "$x$ was published before $y$" is true.

Two important concepts are the \emph{inverse relation}\index{inverse relation} and \emph{composition relation}\index{composition relation}.
For any relation $R\subset U\times W$, the inverse relation is defined 
\begin{equation}
R^{-1}\coloneqq\{(w,u)\in W\times U:\,(u,w)\in R \}.
\end{equation}  
For two relations $R_1\subset V\times U$ and $R_2\subset U\times W$, the composition relation on $V\times W$ is defined 
\begin{equation}
R_2\circ R_1\coloneqq \{(v,w)\in V\times W:\, (\exists u)\,( (v,u)\in R_1 \land (u,w)\in R_2) \}.
\end{equation}

\subsection{Functions}
A \emph{function}\index{function} $f$ from $X$ into $Y$, usually denoted $f:X\rightarrow Y$, is a binary relation in $X\times Y$ with the properties,
\begin{enumerate}
	\item For all $x\in X$ there exists a $y\in Y$ such that $(x,y)\in f$.
	\item If $(x,y_1)\in f$ and $(x,y_2)\in f$, then $y_1=y_2$.
\end{enumerate} 
In light of the uniqueness property, we write $f(x)$ to denote the element in $Y$ associated to $x$.
The set $X$ is referred to as the \emph{domain}\index{domain} of $f$ and $Y$ is referred to as the \emph{codomain}\index{codomain}.

The inverse relation $f^{-1}$ of a function $f$ is not necessarily a function. 
When the inverse relation of a function is itself a function, the original function is said to be \emph{invertible}\index{invertible function}. 

The following are a few useful properties relating how basic set operations are preserved between the domain and codomain of a function. Let $f$ be a function from $X$ into $Y$ with $A_i\subset X$ and $B_i\subset Y$.
\begin{enumerate}
	\item $f(A_1\cup A_2) = f(A_1)\cup f(A_2)$, 
	\item $f(A_1\cap A_2)\subset f(A_1)\cap f(A_2)$,
	\item $f(A_1)\setminus f(A_2) \subset f(A_1\setminus A_2)$
	\item $f^{-1}(B_1\cup B_2) = f^{-1}(B_1)\cup f^{-1}(B_2)$, 
	\item $f^{-1}(B_1\cap B_2) = f^{-1}(B_1)\cap f^{-1}(B_2)$,
	\item $f^{-1}(B_2\setminus B_1) = f^{-1}(B_2)\setminus f^{-1}(B_1)$
\end{enumerate}
For any set $S$, a \emph{sequence}\index{sequence} on $S$ is a function in $\mathbb{N}\times S$ or $\{1,2,...,n\}\times S$. 
While one would normally write $s(i)$ for the element $(i,s)$ in the sequence, the notation $\{s_i\}_{i\in \mathbb{N}}$ is usually used. When the index set is clear, this is abbreviated $\{s_i\}$. 
A \emph{subsequence}\index{subsequence} of $s_{j}$ of $s_i$ is a function from a subset of the natural numbers into $S$ which is equal to $s_i$ when $i=j$. 

If $f$ is a function from $X$ into $X$, then a \emph{fixed point}\index{fixed point} of $f$ is an element $x\in X$ satisfying $(x,x)\in f$ (usually written $f(x)=x$).
\subsection{Orders}\label{app:orders}
An \emph{order}\index{order} on a set $X$ is a binary relation $R \subset X\times X$ with the properties,
\begin{enumerate}
	\item $(x,x)\in R$.
	\item If $(x,y)\in R$ and $(y,x)\in R$, then $x=y$.
 	\item If $(x,y)\in R$ and $(y,z)\in R$, then $(x,z)\in R$.
 	\item $(x,y)\in R$ or $(y,x)\in R$ for all $x,y\in X$. 
\end{enumerate}
A relation in $X\times X$ is a \emph{partial order}\index{partial order} if it only satisfies items 1-3. It is a \emph{strict partial order}\index{strict partial order} if it satisfies items 2-3 and the negation of item 1. 
The notation for an order $R\subset X\times X$ is usually denoted $x\leq y$ for $(x,y)\in R$.

Let $(X,\leq)$ be an ordered set. 
A \emph{lower bound}\index{lower bound} $l$ on a subset $S$ of $X$ is an element of $X$ such that $l\leq s$ for all $s\in S$. 
Similarly, an \emph{upper bound}\index{upper bound} $u$ on a subset $S$ of $X$ is an element of $X$ such that $s\leq u$ for all $s\in S$. 
Now let $L_S$ be the set of lower bounds of $S$, and $U_S$ be the set of upper bounds of $S$.
If $L_S$ is nonempty, then $S$ is said to be bounded from below. Similarly, if $U_S$ is nonempty, then $S$ is said to be bounded from above.
The ordered set $(X,\leq)$ is said to have the \emph{greatest lower bound property}\index{greatest lower bound property} if, for every nonempty $S$ that is bounded from below, there exists a lower bound $l^*\in X$ of $S$ which is also an upper bound of $L_S$. 
We call $l^*$ a greatest lower bound or \index{infimum}\emph{infimum}.
Similarly, the ordered set $(X,\leq)$ is said to have the \emph{least upper bound property}\index{least upper bound property} if, for every nonempty $S$ that is bounded from above, there exists an upper bound $u^*\in X$ of $S$ which is also a lower bound of $U_S$. 
We call $u^*$ a least upper bound or \index{supremum}\emph{supremum}.

A direct consequence of these definitions is that an ordered set has the greatest lower bound property if and only if it has the least upper bound property. Additionally, if an ordered set admits an infimum/supremum it is unique. 

If a subset of an ordered set contains its infimum/supremum, then we say that the set admits a minimum/maximum element.

\subsection{Equivalence Relations}
An \emph{equivalence relation}\index{equivalence relation} on a set $X$ is a binary relation $R \subset X \times X$ with the properties,
\begin{enumerate}
	\item $(x,x)\in R$ for all $x\in X$.
	\item $(x,y)\in R$ if and only if $(y,x)\in R$. 
	\item If $(x,y)\in R$ and $(y,z)\in R$, then $(x,z)\in R$.
\end{enumerate}
The notation for an equivalence relation is usually denoted $x\sim y$ if $(x,y)\in R$.
The \emph{equivalence class}\index{equivalence class} of an element $z\in X$ is the set of all $x\in X$ such that $z\sim x$. This is denoted
\begin{equation}
[z]=\{x\in X:\,z\sim x\}.
\end{equation}
Clearly, the set of all equivalence classes is a cover of $X$ since
\begin{equation}
\bigcup_{x\in X} \{x\} \subset \bigcup_{x\in X} [x].
\end{equation}
Additionally, if $[x]$ and $[y]$ are equivalence classes, either $[x]=[y]$ or $[x]\cap [y]=\emptyset$. Thus, every equivalence relation induces a partition of $X$ composed of the equivalence classes of the relation.
Conversely, every partition $P$ of $X$ induces an equivalence relation where $x\sim y$ if $x,y\in P$.


\newpage

\chapter{Review of Mathematical Analysis} 

This appendix on is based on personal notes from the following sources~\cite{kelley1975general,munkres2000topology,rudin1964principles}. 

\section{Topological Spaces}\label{app:topo}
Equipping sets with a topology makes it possible to discuss continuity of functions, limits of sequences, compactness, and other useful properties.

A set $X$ together with a collection of subsets $\mathcal{T}_X$ is called a \emph{topological space}\index{topological space} on $X$ if
\begin{itemize}
	\item[TS1] $\emptyset\in \mathcal{T}_X$ and $X\in \mathcal{T}_X$.
	\item[TS2] If $S_i\in \mathcal{T}_X$ for every $i\in I$, then   $\bigcup_{i\in I} S_i\in \mathcal{T}_X$.
	\item[TS3] $U\cap S\in \mathcal{T}_X$ for every $U,S\in \mathcal{T}_X$.
\end{itemize} 
A collection of subsets of $X$ satisfying axioms TS1-TS3 is called a \emph{topology}\index{topology} on $X$.
The elements of a topology are called \emph{open sets}\index{open set}.
A set which is the complement of an open set is called a \emph{closed set}\index{closed set}. 
Note that a set can be both open and closed (e.g. $\emptyset$ in any topological space).

The \emph{interior}\index{interior} of a subset $S$, denoted $int(S)$, is the union of open subsets of $S$. 
The \emph{closure}\index{closure} of a subset $S$, denoted $cl(S)$ is the intersection of closed sets containing $S$.  
A \emph{neighborhood}\index{neighborhood} $\mathcal{N}_p$ of a point $p$ in $X$ is an element of $\mathcal{T}_X$ containing $p$.

A point $p$ in a subset $S$ of $X$ is said to be a \emph{limit point}\index{limit point} of $S$ if every neighborhood of $p$ has a nonempty intersection with $S\setminus \{p\}$.
A sequence $\{x_i\}$ in $X$ is said to \emph{converge}\index{convergent sequence} to a point $p$ in $X$ if, for any neighborhood $\mathcal{N}_p$ of $p$, there exists an $N\in\mathbb{N}$ such that $n>\mathbb{N}$ implies $x_i\in \mathcal{N}_p$.

If $S$ is a subset of $X$, then the \emph{subspace topology}\index{subspace topology} $\mathcal{T}_S$ on $S$ is defined by 
\begin{equation}
\mathcal{T}_S=\left\{ S\cap U:\,\,U\in \mathcal{T}_X \right\}.
\end{equation}
\begin{lemma}\label{lem:subspace}
	If $S$ is an element of $\mathcal{T}_X$, then the subspace topology $\mathcal{T}_S$ is a subset of $\mathcal{T}_X$.  
\end{lemma}
\begin{proof}
	If $S$ is an element of $\mathcal{T}_X$ then so is $S\cap U$ for any $U$ in $\mathcal{T}_X$ by TS3.
\end{proof}
\par{\noindent}It is important to note that if $S$ is not open, elements of the subspace topology defined by $S$ may not be elements of the original topology.

A collection of subsets $\mathcal{B}_X$ of $X$ is called a \emph{base}\index{base} for a topology if it is a cover of $X$, and for every $B_i,B_j\in\mathcal{B}_X$, if $x$ is an element of $B_i\cap B_j$, then there exists a $B_k\subset B_i\cap B_j$ in $\mathcal{B}_X$ such that $x$ is an element of $B_k$. 
A topological space is \emph{second countable}\index{second countable topological space} if it has a countable base. 
\begin{lemma}
	Let $\mathcal{B}_X$ be a base for a topology on $X$, and let $\mathcal{T}_B$ be the collection of subsets of $X$ satisfying: for every $U \in \mathcal{T}_B$ and $x\in U$, there exists $B\in \mathcal{B}_X$ such that $x\in B$ and $B\subset U$. The collection of subsets $\mathcal{T}_B$ is a topology on $X$ called the topology generated by $\mathcal{B}_X$. 
\end{lemma}
\begin{proof}
	We need to show that $\mathcal{T}_B$ is a topology.
	
	(TS1) It is vacuously true that the empty set belongs to $\mathcal{T}_B$. Next, since $\mathcal{B}_X$ is a cover of $X$, there exists a $B\in \mathcal{B}_X$ such that $x\in B$ for every $x\in X$ and clearly $B\subset X$. Thus, $X\in \mathcal{T}_B$.
	
	(TS2) Let $U_i$ be an element of $\mathcal{T}_B$ for all $i\in \mathcal{I}$ and define 
	\begin{equation}
		U\coloneqq\bigcup_{i\in\mathcal{I}} U_i.
	\end{equation}
	For each $x\in U$, there is an $i\in \mathcal{I}$ such that $x\in U_i$. Now by construction of $\mathcal{T}_B$, there exists $B\in \mathcal{B}_X$ such that $x\in B$ and $B\subset U_i$. Then $x\in U$ and $B \subset U$ so that $U\in \mathcal{T}_B$.
	
	(TS3) Consider $U_1,U_2\in \mathcal{T}_B$ and $x\in U_1\cap U_2$. By construction of $\mathcal{T}_B$, there exists $B_1,B_2 \in \mathcal{B}_X$ such that $x\in B_{1}\subset U_1$, and $x\in B_{2}\subset U_{2}$. By definition of a base, there exists $B_3$ such that $B_3\subset B_1\cap B_2$ and $x\in B_3$. Thus, $B_3\subset U_1\cap U_2$ which, by construction of $\mathcal{T}_B$, implies $U_1\cap U_2\in \mathcal{T}_B$.     
\end{proof}

\subsection{Compactness}

A subset of a topological space is said to be \emph{compact}\index{compact sets} if every open cover has a finite subcover. 
\begin{lemma}
	If $S$ is a closed subset of a compact set $C$, then $S$ is compact.
\end{lemma}
\begin{proof}
	Let $\{U_i\}_{i\in\mathcal{I}}$ be an open cover of $S$. Then $\{S^c\}\cup\{U_i\}_{i\in\mathcal{I}}$ is an open cover of $C$ which will have a finite subcover $\{W_i\}_{i\in\mathcal{J}}$. Then $\{W_i\}_{i\in\mathcal{J}}\setminus \{S^c\}$ is a finite subcover of $\{U_i\}_{i\in\mathcal{I}}$.
\end{proof}
\begin{lemma}
	If $C$ is a compact subset of a topological space, then every sequence $\{c_n\}$ contained in $C$ has a subsequence converging to a point in $C$.
\end{lemma}
\begin{proof}
	Suppose that the image of the sequence $\{c_n\}$ is finite. Then $c_n$ maps to at least one point $p$ in $C$ an infinite number of times in which case the subsequence  whose image is $p$ converges to $p$. Now suppose $\{c_n\}$ is infinite, and as a point of contradiction that this subset has no limit points. Then it is closed which implies it is compact as well since a closed subset of a compact set is compact. Then every point in $\{c_n\}$ will have a neighborhood $\mathcal{N}_{c_n}$ containing only the point $c_n$. This open cover of $\{c_n\}$ cannot have a finite subcover so it is not compact which is a contradiction. Thus, $\{c_n\}$ must have a limit point in $C$ from which a subsequence converging to that limit point can be constructed.  
\end{proof}   
\par{\noindent} It is interesting to note that the converse is also true, but this fact is not required for the arguments in this thesis.

\subsection{Continuity}

A function from a topological space $(X,\mathcal{T}_X)$ into a topological space $(Y,\mathcal{T}_Y)$ is said to be \emph{continuous}\index{continuous} if the inverse image of every open subset of $Y$ under $f$ is an open subset of $X$.
That is,
\begin{equation}
f^{-1}(U)\in \mathcal{T}_X\quad \forall U\in \mathcal{T}_Y.
\end{equation}
Alternatively, the function $f$ is continuous on a subset $S$ of $X$, if $f$ satisfies the above definition of continuity in the subspace topology associated with $S$. That is,
\begin{equation}
f^{-1}(U)\cap S \in \mathcal{T}_S \qquad \forall U\in \mathcal{T}_Y.
\end{equation}
It is important to point out that is is not necessarily true that $f^{-1}(U)\cap S \in \mathcal{T}_S$ implies $f^{-1}(U)\cap S \in \mathcal{T}_X$ (i.e. the inverse image intersected with $S$ need not be open in $\mathcal{T}_X$). 
\begin{lemma}\label{lem:compact_image}
 	Let $f$ be a continuous function from a topological space $(X,\mathcal{T}_X)$ into a topological space $(Y,\mathcal{T}_Y)$. If $C$ is a compact subset of $X$, then $f(C)$ is a compact subset of $Y$.   
\end{lemma}
\begin{proof}
 	Let $\{U_i\}_{i\in\mathcal{I}}$ be an open cover of $f(C)$. By the definition of continuity and TS2, $\bigcup_{i\in\mathcal{I}} f^{-1}(U_i)$ is an open cover of $C$. Since $C$ is compact this cover contains a finite subcover $\{f^{-1}(U_i)\}_{i\in\mathcal{J}}$ of $C$. Then 
	\begin{equation}
	 f\left( \bigcup_{i\in\mathcal{J}} f^{-1}(U_i) \right) \subset  \bigcup_{i\in\mathcal{J}} f\left( f^{-1}(U_i) \right) \subset  \bigcup_{i\in\mathcal{J}} U_i, 
	\end{equation}
	 which is an open cover of $f(C)$. Therefore, $C$ is compact.
 \end{proof}
\section{Metric Spaces}\label{sec:metric_space}
Metric spaces are sets equipped with a distance function that assign a distance between two points in the set. This distance function is a generalization of distance as we encounter it in two or three dimensional space.
A set $X$ together with a function $d_X:X\times X\rightarrow [0,\infty)$ is called a \emph{metric space}\index{metric space} if
\begin{itemize}
	\item[MS1] $0\leq d_X(x,y)$  for every $x,y$ in $X$.
	\item[MS2]  $d_X(x,y)=0$ if and only if $x=y$.
	\item[MS3] $d(x,y)=d(y,x)$ for every $x,y$ in $X$.
	\item[MS4] $d(x,z)\leq d(x,y)+d(y,z)$ for every $x,y,z$ in $X$
\end{itemize}
A distance function satisfying axioms MS1-MS4 is called \emph{metric}\index{metric} on $X$.
Notice that if $(X,d_X)$ is a metric space and $S\subset X$, then $(S,d_X)$ is a metric space. This is referred to as a \emph{metric subspace}\index{metric subspace}.

The \emph{open ball}\index{open ball} of radius $r$ centered at $p$ in a metric space $(X,d_X)$ is defined 
\begin{equation}
B_r(p)=\left\{x\in X:\, d_X(p,x)<r \right\}.
\end{equation} 
\begin{lemma}\label{app:induced_topolot}
	Every metric space $(X,d_X)$ has an induced topology generated by $\{B_r(x)\}_{x\in X,\,r>0}$. This topology is called the metric topology.
\end{lemma}
\begin{proof}
	It is sufficient to show that  $\{B_r(x)\}_{x\in X,\,r>0}$ is a base for a topology. Clearly, $\{B_r(x)\}_{x\in X,\,r>0}$ is a cover of $X$ since $x\in B_r(x)$ for any $r>0$. Now let $x$ be an element of $B_{r_1}(p_1)\cap B_{r_2}(p_2)$. Denote $d_X(x,p_1)$ by $d_1$ and $d_X(x,p_2)$ by $d_2$. Then $B_{(r_1-d_1)}(x)\subset B_{r_1}(p_1)$ and $B_{(r_2-d_2)}(x)\subset B_{r_2}(p_2)$. Next, let $r_3=\min\{r_1-d_2,r_2-d_2\}$. Then $B_{r_3}(x)\subset B_{r_1}(p_1)\cap B_{r_2}(p_2)$. Therefore,  $\{B_r(x)\}_{x\in X,\,r>0}$ is a base for a topology. 
\end{proof}
With the induced topology, every metric space inherits the analytical tools available for a topological space.

\subsection{Continuity in Metric Spaces}\label{app:continuity}

Metric spaces gives rise to an alternative interpretation of continuity:
\begin{lemma}\label{app:ep_delt_cont}
	Suppose $f$ is a function from a metric space $(X,d_X)$ into $(Y,d_Y)$. If for every $p\in X$ and $\varepsilon>0$ there exists a $\delta>0$ (depending on $p$ and $\varepsilon$) such that $d_X(p,x)<\delta$ implies $d_Y(f(p),f(x))<\varepsilon$, then $f$ is continuous in the metric topology.
\end{lemma}
\begin{proof}
	Let $U$ be an open set in $Y$ with respect to the metric topology. Choose any $p\in f^{-1}(U)$. It will be sufficient to show that $p$ is an interior point of $f^{-1}(U)$. Since $U$ is open, there exists $B_\varepsilon(f(p))$ which is a subset of $U$. Now suppose that for every $\varepsilon>0$ there exists a $\delta>0$ such that $d_X(p,x)<\delta$ implies $d_Y(f(p),f(x))<\varepsilon$. Then $f^{-1}(B_\varepsilon(f(p)))\subset B_\delta(p)$. Therefore, $B_\delta(p)\subset f^{-1}(U)$ and, $p$ is an interior point of $f^{-1}(U)$. Thus, $f^{-1}(U)$ is open for any open set $U$ which is the definition of continuity.
\end{proof}
\par{\noindent}The converse is also true making this an alternative definition of continuity.

One can define stronger forms of continuity in metric spaces. 
A function $f$ from a metric space $(X,d_X)$ into a metric space $(Y,d_Y)$ is said to be \emph{uniformly continuous}\index{uniform continuity} if for every $\varepsilon>0$ there exists a $\delta>0$ such that $d_X(x_1,x_2)<\delta$ implies $d_Y(f(x_1),f(x_2))<\varepsilon$. 
In contrast to the basic form of continuity, the $\delta$ in this definition can be independent of the points $x_1$, $x_2$ and depends only on $\varepsilon$.
\begin{lemma}\label{lem:uniform_cont}
	A continuous function $f$ from a compact metric space $(X,d_X)$ into a metric space $(Y,d_Y)$ is    uniformly continuous.
\end{lemma}

\begin{proof} Since $f$ is continuous, for every $\varepsilon>0$ and $x\in X$, there exists a ball $B_{\delta_x}(x)$ such that $f(B_{\delta_x}(x))\subset B_{\varepsilon} (f(x))$. For each $\varepsilon>0$, the balls $B_{\delta_x/2}(x)$ form an open cover of $X$. Thus, there is a finite subcover defined by a finite set of points $\{x_1,...,x_n\}$. Now let $\delta^*=\min\{\delta_{x_1},...,\delta_{x_n}\}$. Since this set is finite, it admits a strictly positive minimum. Now choose $p_1,p_2\in X$ such that $d_X(p_1,p_2)<\delta^*/2$. Then $d_X(p_1,x_m)<\delta^*/2$ for one of the points $x_m$ and by the triangle inequality $d_X(p_2,x_m)<2\delta^*$. Thus, $p_1,p_2\in B_{\delta^*}(x_m)$ which implies $d_X(f(p_1),f(p_2))<\varepsilon$. $\delta^*$ is independent of $p_1$ and $p_2$ and therefore, $f$ is uniformly continuous.     
\end{proof}

An even stronger form of continuity is Lipschitz continuity. 
A function $f$ from a metric space $(X,d_X)$ into a metric space $(Y,d_Y)$ is said to be \emph{Lipschitz continuous}\index{Lipschitz continuity} if there exists a constant $L$ such that 
\begin{equation}
d_X(x_1,x_2)< L\cdot d_Y(f(x_1),f(x_2))\qquad \forall x_1,x_2\in X. 
\end{equation}
The function ${\rm Lip}$ from the set of functions between two metric spaces into $\mathbb{R}$ will be  defined as the greatest lower bound on the set of Lipschitz constants of its argument. 
That is, if $f$ is a function between two metric spaces, then ${\rm Lip}(f)$ is the infimum over Lipschitz constants of $f$.

\begin{lemma}\label{lem:compact_closed}
	If $C$ and $U$ are disjoint subsets of a metric space $(X,d_X)$, with $C$ compact and $U$ closed, then there exists a $\rho>0$ such that 
	\begin{equation}
	\rho < d_X(c,u),\quad \forall c\in C,u\in U .
	\end{equation} 
\end{lemma}
\begin{proof}
	As a point of contradiction suppose that the claim is false. Then there exists a sequence $\{c_n\}$ contained in $C$ and a point $u$ in $U$ such that $\lim_{n\rightarrow \infty} d(c_n,u) = 0$. Equivalently, $\lim_{n\rightarrow\infty} c_n =u$. Since $C$ is a compact set, there is a subsequence $\{c_m\}$ of $\{c_n\}$ which converges to a point $c$ in $C$. Then by the triangle inequality,
	\begin{equation}
	\begin{array}{rcl}
	d(c,u) &\leq& \lim_{m\rightarrow \infty} d(c_m,c)+\lim_{m\rightarrow \infty}d(c_m,u)\\
	& = & 0+0.
	\end{array}
	\end{equation}
	Thus, $c=u$. Since $c$ is a point in $C$, $u$ is a point in $U$, and $C\cap U$ is the empty set, we arrive at a contradiction.
\end{proof}
\begin{cor}[to Lemma \ref{lem:compact_closed}]\label{cor:compact_closed}
	If $C$ is a compact subset of an open set $S$ in a metric space $(X,d_X)$, then there exists a $\rho>0$ such that 
	\begin{equation}
		B_{\rho}(c)\subset S,\qquad \forall c\in C.
	\end{equation}
\end{cor}
\begin{proof}
	Since $C$ is a subset of $S$, it is disjoint with $S^c$, and since $S$ is open, $S^c$ is closed. Then by Lemma \ref{lem:compact_closed} there is a $\rho>0$ such that $d_X(c,u)>\rho$ for all $c$ in $C$ and $u$ in $S^c$. Equivalently, $B_\rho (c)\cap S^c=\emptyset$ which implies $B_\rho(c)\subset S.$
\end{proof}

\subsection{Complete Metric Spaces}
A sequence $x_n$ in a metric space $(X,d_X)$ is said to be a \emph{Cauchy sequence}\index{Cauchy sequence} if for every $\varepsilon>0$, there exists an $N$ such that $n,m>N$ implies $d_X(x_n,x_m)<\varepsilon$.
The metric space $(X,d_X)$ is said to be a \emph{complete metric space}\index{complete metric space} if every Cauchy sequence converges to a point in $X$. 
\begin{lemma}\label{lem:closed_complete}
	If $C$ is a closed subset of a complete metric space $(X,d_X)$, then $(C,d_X)$ is a complete metric space.
\end{lemma}
\begin{proof}
	Let $x_n$ be a Cauchy sequence in $C$. Then since $X$ is complete, $\lim_{n\rightarrow\infty} x_n =x$ for some $x\in X$. However, $x$ is then a limit point of $C$ and since $C$ is closed, $x\in C$. Thus, $(C,d_X)$ is complete.
\end{proof}
A function $f$ from a metric space $(X,d_X)$ into $(X,d_X)$ is called a \emph{contraction}\index{contraction} if there exists $\gamma\in [0,1)$ such that 
\begin{equation}
d_X(f(x_1),f(x_2))\leq \gamma d_X(x_1,x_2)\qquad \forall x_1,x_2\in X.
\end{equation}

\begin{theorem}[Banach fixed point theorem]\label{thm:banach_fixed_point}
	If $f$ is a contraction on a complete metric space $(X,d_X)$, then there exists a unique fixed point of $f$ in $X$.
\end{theorem}
\begin{proof}
	Let $x_0$ be any point in $X$ and define the sequence $x_n$ recursively as 
	\begin{equation}
		x_{n+1}=f(x_n).
	\end{equation}
	Then since $f$ is a contraction, there exists $\gamma<1$ such that
	\begin{equation}
	d_X(x_{n+1},x_n)=d_X(f(x_n),f(x_{n-1}))\leq \gamma d_X(x_n,x_{n-1}), 
	\end{equation}
	which by induction yields
	\begin{equation}\label{eq:geo_series}
	d_X(x_{n+1},x_n)\leq \gamma^n d_X(x_1,x_0).
	\end{equation}
	Next we show that $x_n$ is a Cauchy sequence.
	Choose $m>n$ and apply the triangle inequality MS3 to $d_X(x_n,x_m)$,
	\begin{equation}
		d_X(x_n,x_m)\leq \sum_{i=n+1}^m d_X(x_i,x_{i-1})
	\end{equation}
	Now apply equation \ref{eq:geo_series} to each term in the right hand side of above expression,
	\begin{equation}
		d_X(x_n,x_m)\leq (\gamma^n +\gamma^n+1 + ... + \gamma ^{m-1})d_X(x_1,x_0).
	\end{equation}
	Now using the formula for geometric series we obtain
	\begin{equation}
		\frac{\gamma^n}{1-\gamma} d_X(x_1,x_2)
	\end{equation}
	which can be made arbitrarily small by choosing $n$ large enough. Thus, $x_n$ is a Cauchy sequence so it must converge to $x^*\in X$ by assumption that $(X,d_X)$ is complete.
	Then 
	\begin{equation}
	\lim_{n\rightarrow\infty} x_n =x^*,
	\end{equation}
	which is equivalent to 
	\begin{equation}
	\lim_{n\rightarrow \infty} f(x_{n-1}) = x^*. 
	\end{equation}
	Now note that since $f$ is a contraction, it is Lipschitz continuous and therefore 
	\begin{equation}
	\lim_{n\rightarrow \infty } f(x_n) = f(x^*).
	\end{equation}
	Thus, $f(x^*)=x^*$ so $x^*$ is a fixed point of $f$. 
	To prove uniqueness, suppose that $x^*$ and $y^*$ are fixed points of $f$. Then $d_x(f(x^*),f(y^*)) \leq \gamma d_X(x^*,y^*)$ and $d_x(f(x^*),f(y^*))=d_x(x^*,y^*)$. The only solution to both the equation and inequality is if $d_X(x^*,y^*)=0$ in which case $x^*=y^*$ by MS1.  
\end{proof}

\section{Normed Vector Spaces}
A \emph{vector space}\index{vector space} is an algebraic structure consisting of a set of \emph{vectors}\index{vectors} and a field\footnote{throughout this thesis the only field considered is $\mathbb{R}$.} of \emph{scalars}\index{scalars}.
A vector space is equipped with an associative and commutative addition "+" operation between vectors and a multiplication "$\cdot$" operation between a scalar and a vector that is distributive over vectors. 
In addition, there is an additive identity vector, denoted "$0$", such that $v+0=v$ for every vector $v$. 
Lastly, there is an additive inverse to every vector $v$, denoted $-v$ such that $v+(-v)=0$. This is abbreviated with $v-v$.

A \emph{normed vector space}\index{normed vector space} is a vector space equipped with a function $\Vert \cdot \Vert$ mapping vectors into $[0,\infty)$. This function, called a \emph{norm}\index{norm} must satisfy 
\begin{itemize}
	\item[NS1] $\Vert 0 \Vert =0$,
	\item[NS2] $\Vert v \Vert >0\qquad \forall v\in \mathcal{V}\setminus\{0\}$,
	\item[NS3] $\Vert \alpha v \Vert = |\alpha| \Vert v \Vert \qquad \forall \alpha \in \mathbb{R}, v\in \mathcal{V}$,
	\item[NS4] $\Vert v+w \Vert \leq \Vert v \Vert +\Vert w\Vert\qquad \forall v,w\in \mathcal{V}$.
\end{itemize}
An example of a norm on the vector space $\mathbb{R}^n$ is the $L_p$ norm for $p=1,2,...,\infty$:
\begin{equation}
\Vert v \Vert_p =\left(\sum_{i=1}^n v_i^p\right)^{\frac{1}{p}}
\end{equation}

\begin{lemma}
For every normed vector space $(\mathcal{V},\Vert \cdot \Vert )$, there is an induced metric space $(\mathcal{V},d_{\mathcal{V}})$ with the metric $d_{\mathcal{V}}$ defined by $d_{\mathcal{V}}(v,w)=\Vert v - w \Vert$.
\end{lemma}
\begin{proof}
	We just need to show that MS1-MS4 follow from the definition of the metric and NS1-NS4. Suppose $v=w$. Then $\Vert v-w \Vert$ is the same as writing $\Vert v-v \Vert$ which is equal to $0$ by NS1. Thus, MS1 is satisfied. Suppose $v\neq w$ so that $v-w\neq 0$. Then $\Vert v-w\Vert>0$ by NS2. Thus, MS2 is satisfied. Symmetry of the metric follows from the compatibility of scalar and vector multiplication, $\Vert v-w\Vert=|-1|\Vert (-1)\cdot v-(-1)\cdot w \Vert=\Vert w-v \Vert $. Thus, MS3 is satisfied. Next, the triangle inequality for the metric is just a rearrangement of the triangle inequality for the norm NS4, 
	\begin{equation}\Stretch
	\begin{array}{rcl}
	\Vert u-w \Vert &=& \Vert u + (-w) \Vert \\
	&=& \Vert u-v+v-w \Vert \\
	&\leq& \Vert u-v\Vert + \Vert v-w \Vert .
	\end{array}
	\end{equation}
	Thus, MS4 is satisfied.
\end{proof}

\subsection{Banach Spaces}
A \emph{Banach space}\index{Banach space} is a normed vector space $(\mathcal{V},\Vert \cdot \Vert )$ with the property that the induced metric space is complete. 
The $\mathbb{R}^n$ with any of the $L_p$ norms is a Banach space.

Now consider the set of bounded functions $\mathscr{B}$ mapping a compact normed space $(X,\Vert \cdot \Vert_X)$ into a Banach space $(Y,\Vert \cdot \Vert_Y)$. The following is a norm on $\mathscr{B}$, 
\begin{equation}
	\Vert f \Vert_{\infty}\coloneqq \sup_{x\in X} \Vert f(x)\Vert_Y.
\end{equation}
\begin{lemma}\label{lem:func_banach} 
	The normed space $(\mathscr{B},\Vert \cdot \Vert_\infty)$ is a Banach space.
\end{lemma}
\begin{proof} Let $f_n$ be a Cauchy sequence in $(\mathscr{B}, \Vert \cdot \Vert_\infty)$. Then since $\Vert f_n(x)-f_m(x) \Vert_Y \leq \Vert f_n-f_m \Vert_\infty$ we have that $f_n(x)$ is a Cauchy sequence in $(Y,\Vert \cdot \Vert_Y)$ for every $x\in X$. Since $Y$ is complete, a point-wise limit of $f_n$ exists which we denote $f$. It remains to show that the limit $f$ is bounded. While not all Cauchy sequences converge (if the space is not complete) they are always bounded. There exists an $M$ such that $\Vert f_n \Vert < M$ for all $n$. Therefore, $\Vert f_n(x)\Vert_Y \leq \Vert f_n \Vert_\infty < M$ for all $n\in \mathbb{N}$ and $x\in X$. Thus, $\Vert f(x)\Vert_Y<M $ which implies $f$ is a bounded function. Thus, the Cauchy sequence converges to $f\in \mathscr{B}$. 
\end{proof}
Now let $\mathscr{C}$ be the subset of continuous functions in $\mathscr{B}$. 
\begin{lemma}\label{lem:cont_closed}
	$\mathscr{C}$ is a closed subset of the Banach space $(\mathscr{B},\Vert \cdot \Vert_\infty).$
\end{lemma}
\begin{proof}
	Let $f$ be a limit point of $\mathscr{C}$. 
	Then there exists $f_n\in \mathscr{C}$ which converges to $f\in\mathscr{B}$. 
	For each $\varepsilon>0$ there exists $N>0$ such that $\Vert f_n - f\Vert_\infty<\varepsilon$ for all $n>N$. 
	Since $f_n$ is continuous on a compact set $X$, it is uniformly continuous by Lemma \ref{lem:uniform_cont}.
	Thus, there exists $\delta>0$ such that $\Vert f_n(x_1) - f_n(x_2)\Vert_Y<\varepsilon$ for all $n>N$ and $x_1,x_2\in X$ satisfying $\Vert x_1-x_2 \Vert_X<\delta$.

	Next, we use the triangle inequality to show that $f$ is continuous and is therefore contained in $\mathscr{C}$. 
	\begin{equation}
	\begin{array}{rcl}
		\Vert f(x_1)-f(x_2) \Vert_Y &\leq& \Vert f(x_1)-f_n(x_1)\Vert_Y \\
		&+& \Vert f_n(x_1)-f_n(x_2) \Vert_Y\\ 
		&+& \Vert f_n(x_2)- f(x_2) \Vert_Y.
	\end{array}
	\end{equation}
	Observe that $\Vert f(x_1)-f_n(x_1)\Vert_Y$ and $\Vert f_n(x_2)- f(x_2) \Vert_Y$ are each less than $\Vert f-f_n \Vert_\infty$ which is less than $\varepsilon$. Then since $f_n$ is uniformly continuous we were able to choose $x_1,x_2$ sufficiently close together so that $\Vert f_n(x_1)-f_n(x_2) \Vert_Y$ was also less than $\varepsilon$. Thus,
	\begin{equation}
		\Vert f(x_1)-f(x_2) \Vert_Y \leq 3\varepsilon,
	\end{equation}
	which implies that $f$ is continuous and is therefore contained in $\mathscr{C}$. Since $f$ was an arbitrary limit point of $\mathscr{C}$ we conclude that $\mathscr{C}$ is closed.
\end{proof}

\begin{theorem} The normed space $(\mathscr{C}, \Vert \cdot \Vert _\infty)$ is a Banach space.
\end{theorem}
\begin{proof}
The normed space $(\mathscr{B},\Vert\cdot\Vert_\infty)$ is a Banach space by Lemma \ref{lem:func_banach} and by Lemma \ref{lem:cont_closed}, $\mathscr{C}$ is a closed subset of $\mathscr{B}$. If follows from Lemma \ref{lem:closed_complete} that $(\mathscr{C},\Vert \cdot \Vert_\infty)$ is a Banach space.
\end{proof}

\bibliographystyle{amsalpha}
\bibliography{appb}

\clearpage

\chapter{Analysis of Functions and Integration}\label{app:functional_analysis}
The content of this appendix is based on personal notes compiled from the following sources~\cite{coddington1955theory,feldman1981proof,kolmogorov1961elements,salamon2016measure}.

\section{Measure Spaces}\label{sec:measure}
It is difficult to have a precise discussion about dynamical systems without Legesgue integration and measure spaces which utilize the concept of measure for sets. 
The assignment of measure to subsets of the measure space is the generalization of assigning weight to shapes with volume.
While the concept is very intuitive, one subtlety with measure spaces is that not all subsets can be assigned measure in general. 

A set $X$ together with a collection $\mathscr{X}$ of subsets, and a function $m:\mathcal{S}\rightarrow[0,\infty]$ is called a measure space if
\begin{itemize}
	\item[SA1] $\emptyset \in \mathcal{X}$, $X\in \mathscr{X}$.
	\item[SA2] $X^c \in \mathcal{X}$ for every $X$ in $\mathscr{X}$.
	\item[SA3] $\bigcup_{n\in \mathbb{N}} X_n\in \mathscr{X}$ and $\bigcap_{n\in \mathbb{N}}X_n$ for any sequence $X_n$ contained in $\mathscr{X}$.
	\item[SA4] $m(\emptyset)=0$.
	\item[SA5] $m\left(\bigcup_{n\in\mathbb{N}} X_n \right)=\sum_{n\in\mathbb{N}}m(X_n)$ for all sequences $X_n$ in $\mathscr{X}$ such that $X_i\cap X_j=\emptyset$ when $i\neq j$.
\end{itemize}  
The collection of subsets $\mathscr{X}$ is called a \emph{$\sigma$-algebra}\index{$\sigma$-algebra} and the function $m$ is called a \emph{measure}\index{measure}.
A set $S$ is said to be \emph{measurable}\index{measurable set} if $S$ is an element of $\mathscr{X}$. 
We say that a proposition $p(x)$ is true \emph{almost everywhere}\index{almost everywhere} in the measure space $(S,\mathscr{S},m)$ if the set  
\begin{equation}
N=\{x\in S :\, \neg p(x) \}
\end{equation}
is measurable and has measure zero. 
That is $m(N)=0$. 
Notice that SR1-SR3 are slightly different than the axioms of a topology making measure spaces a similar, but not identical structure.
In particular, one could say that a topology is "almost" a $\sigma$-algebra. 
There is a natural $\sigma$-algebra associated to every topology. 
This is known as the \emph{Borel $\sigma$-algebra}\index{Borel $\sigma$-algebra} composed of \emph{Borel sets}\index{Borel set}. 
A Borel set is any subset of a topological space that can be constructed from countable intersections, countable unions, and complements of elements of the topology.

With the addition of a measure, a Borel $\sigma$-algebra on a topological space becomes a measure space.
Borel measure spaces are of little use without a few regularity properties. 
First, a Borel measure is \emph{locally finite}\index{locally finite measure space} if every point has a neighborhood with finite measure. 
Second, a Borel measure space $(X,\mathscr{X},m_x)$ is \emph{inner regular}\index{inner regular measure spaces} if for every $\varepsilon>0$ and $S \in \mathscr{X}$, there exists a compact $K\subset S$ such that $m_X(S\setminus K)<\varepsilon$. 
A function $f$ from a measure space $(X,\mathscr{X},m_X)$ into a measure space $(Y,\mathscr{Y},m_Y)$ is said to be a \emph{measurable function}\index{measurable function} if the preimage of every measurable set in $Y$ is a measurable set in $X$. That is 
\begin{equation}
f^{-1}(U)\in \mathscr{X},\qquad \forall U\in \mathscr{Y}.
\end{equation}  
Note that this is identical to the definition of continuity in a topological space if we replace the term "measurable" with "open". 
In light of the similarity between measure spaces and topological spaces, one might guess that a measurable function between Borel $\sigma$-algebra is "almost" a continuous function with respect to the topologies generating the algebras.
This is indeed the case and this beautiful connection between measure theory and topology is due to Lusin~\cite{lusin1912proprietes}.
\begin{theorem}[Lusin]
	Let $f$ be a measurable function from a locally finite, inner regular Borel measure space $(X,\mathscr{X},m_X)$ into a Borel measure space $(Y,\mathscr{Y},m_Y)$ generated by a second countable topology. For every $\varepsilon>0$ and $X$ in $\mathscr{X}$, there is a subset of $V$ of $X$ such that 
	\begin{enumerate}
		\item $f$ is continuous on $X\setminus V$, and
		\item $m_X(V)<\varepsilon$.
	\end{enumerate}  
\end{theorem} 
There are many proofs of this theorem. This presentation of the proof is an expanded version of Feldman's~\cite{feldman1981proof} with added detail.
\begin{proof}
	Since $\mathcal{T}_Y$ is second countable, it has a countable base $\{W_i\}_{i\in \mathbb{N}}$. By the inner regularity of $m_X$, there exists open sets $U_i$ in $\mathcal{T}_X$ such that $f^{-1}(W_i)\subset U_i$ and $m_X(U_i\setminus f^{-1}(W_i))<\varepsilon/(2^i)$ for each $i\in \mathbb{N}$. Now define the set 
	\begin{equation}
		V\coloneqq \bigcup_{i\in \mathbb{N}}\, U_i\setminus f^{-1}(W_i).
	\end{equation} 
	It follows from Axiom SA5 of a measure space that $m_X$ is sub-additive so that 
	\begin{equation}
		m_X(V)\leq \sum_{i\in \mathbb{N}} m_X (U_i\setminus f^{-1}(W_i))
	\end{equation}
	Then from the construction of the sets $U$ we have 
	\begin{equation}
		m_X(V)\leq \sum_{i\in \mathbb{N}} \frac{\varepsilon}{2^i} = \varepsilon. 
	\end{equation}
	
	The next step in the proof is to show that $f^{-1}(W_i)\setminus V = U_i\setminus V$. 
	First, it is an immediate consequence of the construction of $U_i$ that 
	\begin{equation}
		f^{-1}(W_i)\setminus V \subset U_i\setminus V.
	\end{equation}
	The reverse inclusion is derived as follows\footnote{It may be helpful to review some the elementary identities of set arithmetic in Appendix \ref{app:sets}.}:
	\begin{equation}\Stretch \label{eq:crazy_set_mess}
	\begin{array}{rcl}
		U_i \setminus V &\subset& U_i \setminus (U_i \setminus f^{-1}(W_i))\\
		&=& U_i \cap (U_i \setminus f^{-1}(W_i) )^c\\
		&=& U_i \cap (U_i^c \cup f^{-1}(W_i))\\
		&=& U_i \cap f^{-1}(W_i)\\
		&=& f^{-1}(W_i).
	\end{array}
	\end{equation}
	Note that the last step follows from the construction of $U_i$. 
	Intersecting the left and right hand sides of \ref{eq:crazy_set_mess} yields the desired reverse inclusion.
	Combining the above two equations, we conclude that 
	\begin{equation}\label{eq:luzin_equality}
		U_i\setminus V= f^{-1}(W_i) \setminus V. 
	\end{equation}
	
	Next, we use \eqref{eq:luzin_equality} to prove that $f$ is continuous on $X\setminus V$. Let $W$ be an open subset of $Y$ such that $f^{-1}(W)\cap V=\emptyset$. 
	The set $W$ can be written as a union of elements of the base $W_i$
	\begin{equation}
		W=\bigcup_{i\in \mathcal{I}} W_i
	\end{equation} 
	Then 
	\begin{equation}\Stretch
	\begin{array}{rcl}
		f^{-1}(W)\setminus V&=& f^{-1}(\bigcup_{i\in \mathcal{I}}W_i)\setminus V\\
		&=& \left( \bigcup_{i \in \mathcal{I}} f^{-1} (W_i) \right) \setminus V  \\
		&=&  \bigcup_{i \in \mathcal{I}} \left( f^{-1} (W_i) \setminus V \right)\\
		&=&  \bigcup_{i \in \mathcal{I}} (U_i\setminus V)\\
		&=&  (\bigcup_{i \in \mathcal{I}} U_i) \setminus V.
	\end{array}
	\end{equation}
	Since $U_i$ are open in $\mathcal{T}_X$, the inverse image $(\bigcup_{i\in\mathcal{I}} U_i)\setminus V$ is open in the subspace topology induced by $X\setminus V$. Therefore, $f$ is continuous on $X\setminus V$ which concludes the proof.
\end{proof}  

\subsection{Lebesgue Integration}
Let $(X,\mathscr{X},m_X)$ be a measure space. 
The \emph{indicator function}\index{indicator function}  $I_S:X\rightarrow \{0,1\}$ on a measurable set $S$ is defined as follows:
\begin{equation}
    I_S(x)= 
    \begin{cases}
    1,&  x\in S,\\
    0,              & \neg (x\in S).
    \end{cases}
\end{equation}
A \emph{simple function}\index{simple function} $s:X\rightarrow [0,\infty)$ has the following form 
\begin{equation}
	s(x)=\sum_{i=1}^n c_i I_{W_i}(x),\qquad c_i \geq 0,\,W_i\in \mathscr{X}.
\end{equation}
%
Now let $\mathcal{S}$ be the set of simple functions which are point-wise less than or equal to a nonnegative measurable function $f:X\rightarrow [0,\infty)$.
The \emph{Lebesgue integral}\index{Lebesgue integration} of $f$ is defined by
\begin{equation}\label{eq:leb_nn}
\int_X f(x)\,m_X(dx) = \sup_{s\in S} \sum_{i=1}^n c_i m_X(W_i).
\end{equation}
The Lebesgue integral of a general real-valued measurable function is defined by
\begin{equation}
\int_X f(x)\,m_X(dx) = \int_X \max\{f(x),0\}\,m_X(dx)-\int_X \max\{-f(x),0\}\,m_X(dx),
\end{equation}
where the two nonnegative functions on the right hand side are defined as in \eqref{eq:leb_nn}.
A measurable function $f$ is said to be Lebesgue integrable if
\begin{equation}
\int_X |f(x)|\,m_X(dx) <\infty.
\end{equation} 
The Lebesgue integral has the following properties,
\begin{enumerate}\label{list:int_props}
	\item If $f,g$ are integrable and $\alpha,\beta$ are real numbers, then 
	\begin{equation}
		\int_X \alpha f(x)+\beta g(x) m_X(dx)=\alpha \int_X f(x) m_X(dx) +\beta \int_X g(x) m_X(dx)
	\end{equation}
	\item If $f,g$ are integrable and $f(x)\leq g(x)$ for all $x$ in $X$, then 
	\begin{equation}
		\int_X f(x)m_X(dx)\leq \int_X g(x)m_X(dx).
	\end{equation}
\end{enumerate}
The set of integrable functions denoted $\mathcal{L}^1 (X,\mathscr{X},m_X)$ is a vector space. 
A semi-norm on this space is given by 
\begin{equation}
\Vert f \Vert _{\mathcal{L}^1}=\int_X |f(x)|\,m_X (dx). 
\end{equation}
The vector space denoted  ${L}^1 (X,\mathscr{X},m_X)$ refers to the quotient space of $\mathcal{L}^1 (X,\mathscr{X},m_X)$ with the Kernel of its semi-norm making the quotient space a proper normed space. 
In fact, it is a Banach space. 

\section{Integral Equations (Dynamical System Models)}\label{sec:ode_models}
Integral equations of the form 
\begin{equation}\label{eq:ode}
	x(\tau) = x_0 + \int_{[t_0,\tau]} f(x(t),t)\,\mu(dt),
\end{equation}
appear frequently in engineering and the sciences when modeling various phenomena.
The function $f:\mathbb{R}^n\times \mathbb{R}$ represents a model of some dynamical phenomenon, and the function $x:[t_0,t_f]\rightarrow \mathbb{R}^n$ describes the evolution of a system parameterized by a vector in $\mathbb{R}^n$. 
For example, the laws of classical mechanics (pretty important for robotics) are stated as equations in the form of \eqref{eq:ode} or the analogous differential form.
It is therefore important to be confident that
\begin{enumerate}
	\item this equation has a solution,
	\item this solution is unique,
	\item and the solution is continuous.
\end{enumerate}
%
%
If there were instances with no solution then the framework would have little scientific value for modeling dynamical phenomena. 
If the second point failed then the model would not be useful for predicting the future state of the system. 
%
%
Lastly, if the solution to \eqref{eq:ode} were not continuous then solutions would permit instantaneous changes in the parameterization of a system. 
Taking these points into consideration highlights the importance and elegance of the fundamental existence-uniqueness theorem stated next.

\begin{theorem}
	Let $f$ be Lipschitz continuous in its first argument on $\mathbb{R}^n$ and measurable in its second argument. 
	Then there exists a unique solution $x:[t_0,t_f]\rightarrow \mathbb{R}^n$ in $(\mathscr{C},\Vert \cdot \Vert_\infty)$.
\end{theorem} 
\begin{proof}
	The proof technique is known as the Picard iteration. 
	Let $L$ denote a Lipschitz constant for $f$. 
	Consider the mapping from one function $x:[t_0,t_0+1/(2L)]\rightarrow \mathbb{R}^n$ to another function $y:[t_0,t_0+1/(2L)]\rightarrow \mathbb{R}^n$ given by 
	\begin{equation}\label{eq:picard_iter}
		y(\tau)=\int_{[t_0,\tau]} f(x(t),t)\,\mu(dt),\qquad \tau\in[t_0,t_0+1/(2L)].
	\end{equation}
	Observe that this mapping is well defined for all $x\in \mathscr{C}$ since the assumptions of the theorem are sufficient for the integral in \eqref{eq:picard_iter} to be finite for each $\tau$. Further, the integral is continuous with respect to $\tau$ so that $y\in \mathscr{C}$ as well\footnote{We have tacitly assumed that $\mu$ is the Lebesgue measure on $\mathbb{R}$.}. 
	
	Denote the mapping from $x$ to $y$  by $y=\mathcal{P}(x)$. Notice that $x$ is a fixed point of $\mathcal{P}$ if and only if it is a solution to \eqref{eq:picard_iter} with $t_f=1/(2L)$.
	
	We will show that the Picard iterate $\mathcal{P}$ is a contraction on  $\mathscr{C}$ to reach the result.
	Choose $x_1,x_2\in \mathscr{C}$ with the domain $[t_0,t_0+1/(2L)]$. 
	Then 
	\begin{equation}\Stretch
	\begin{array}{rcl}
		\Vert [\mathcal{P}(x_1)-\mathcal{P}(x_2)](\tau) \Vert_2 &=& \left\Vert \int_{[0,\tau]} f(x_1(t),t)-f(x_2(t),t)\,\mu(t) \right\Vert_2  \\
		&\leq&  \int_{[0,\tau]} \left\Vert f(x_1(t),t)-f(x_2(t),t)\right\Vert_2\,\mu(t) .
	\end{array}
	\end{equation}
	Then leveraging the Lipschitz continuity of $f$ in its first argument, there exists an $L$ such that 
	\begin{equation}\Stretch
	\begin{array}{rcl}
		\int_{[0,\tau]} \left\Vert f(x_1(t),t)-f(x_2(t),t)\right\Vert_2\,\mu(t) &\leq& \int_{[0,\tau]} L \left\Vert x_1(t)-x_2(t)\right\Vert_2\,\mu(t) \\
		&\leq&   L \left\Vert x_1-x_2\right\Vert_\infty \int_{[0,\tau]} \,\mu(t)\\
		&\leq& L \left\Vert x_1-x_2\right\Vert_\infty (1/(2L))\\
		&\leq& \frac{1}{2} \Vert x_1-x_2 \Vert_\infty. 
		\end{array}
	\end{equation}
	Thus, we have that 
	\begin{equation}
		\Vert [\mathcal{P}(x_1)-\mathcal{P}(x_2)](\tau) \Vert_2 \leq \frac{1}{2} \Vert x_1 - x_2 \Vert_\infty\qquad \forall \tau \in [t_0,t_0+1/(2L)],
	\end{equation}
	which is equivalent to 
	\begin{equation}
		\Vert \mathcal{P}(x_1) - \mathcal{P}(x_2) \Vert_\infty \leq \frac{1}{2} \Vert x_1 - x_2 \Vert_\infty. 
	\end{equation}
	Therefore, the Picard iterate is a contraction on the interval $[t_0,t_0+1/(2L)]$. By Theorem \ref{thm:banach_fixed_point} there exists a unique fixed point to the Picard iterate, and therefore there is a unique solution $x_1^*$ to \eqref{eq:ode} on the interval $[t_0,t_0+1/(2L)]$. This solution necessarily satisfies $x_1^*(t_0)=x_0$. 
	
	Now consider solutions on the time interval $[t_0+1/(2L),t_0+1/L]$ with initial condition $x^*_1(t_0+1/(2L))$, replacing $x_0$ in \eqref{eq:ode}.
	By the same argument, there exists a unique solution $x_2^*$ on this time interval with $x^*_1(t_0+1/(2L))=x^*_2(t_0+1/(2L))$. Combining these observations, there is necessarily a unique continuous solution on the combined interval $[t_0,t_0+1/L]$. A trivial induction argument leads to the conclusion that there is a unique continuous solution on any finite time interval $[t_0,t_f]$. 
\end{proof}
Next we show that if the dynamic model is bounded by $M$ as in Chapter \ref{chap:problem}, then the solutions are not only continuous, but Lipschitz continuous with Lipschitz constant $M$.
\begin{cor}
	If, in addition to the assumptions of Theorem \ref{thm:existence-uniqueness}, the function $f$ is bounded by $M$; then there is a unique Lipschitz continuous solution to \eqref{eq:ode} with Lipschitz constant $M$. 
\end{cor}
\begin{proof}
	Let $x$ be a solution to \eqref{eq:ode} on a finite time interval. Then for two times $t_1<t_2$ in this time interval,
	\begin{equation}\Stretch
	\begin{array}{rcl}
		 \Vert x(t_2)-x(t_1)\Vert_2 &=& \left\Vert \int_{[t_0,t_2]} f(x(t),t)\,\mu(dt) - \int_{[t_0,t_1]} f(x(t),t)\,\mu(dt) \right\Vert_2\\
		 &=& \left\Vert \int_{[t_1,t_2]} f(x(t),t)\,\mu(dt) \right\Vert_2\\
		 &\leq&  \int_{[t_1,t_2]} \left\Vert f(x(t),t) \right\Vert_2\,\mu(dt) \\
		 &\leq& M|t_2-t_1|
	\end{array}  
	\end{equation}
	Thus, $x$ is Lipschitz continuous with constant $M$.
\end{proof}

\begin{lemma}[Gronwall's Inequality]\label{lem:gw_ineq} Suppose
	$z$ is a continuous function from the interval $[0,T]$ into $\mathbb{R}$. If
	\begin{equation}
	z(t)\leq z_{0}+\int_{[0,t]}\lambda z(\tau)\,\mu(d\tau)
	\end{equation}
	for $z_{0},\lambda\in\mathbb{R}$, then 
	\begin{equation}
	z(t)\leq z_{0}e^{\lambda t}.
	\end{equation}
	
\end{lemma}
\begin{proof}
	
	Let $y(t)=\int_{[0,t]}\lambda z(\tau)\,d\mu(\tau)$.
	The assumption that $z$ is continuous implies that $y$ is
	differentiable. 
	Now let $w(t) = z_{0}+y(t)-z(t)$.
	We have $w(t)\geq0$. 
	Rewriting the definition of $y$ in terms
	of $w(t)$ yields, 
	\begin{equation}
	y(t)=\int_{[0,t]}\lambda\left(y(\tau)+z_{0}-w(\tau)\right)\,\mu(d\tau).
	\end{equation}
	Taking the derivative of $y$ yields 
	\begin{equation}
	\frac{d}{dt}y(t)=\lambda y(t)+\lambda z_{0}-\lambda w(t).
	\end{equation}
	This is a linear differential equation with solution 
	\begin{equation}
	y(t) = \int_{[0,t]}\lambda e^{\lambda(t-\tau)}\left(z_{0}-w(\tau)\right)\,\mu(d\tau).
	\end{equation}
	Since $w(t)\geq0$ we have the inequality 
	\begin{equation}
	\begin{array}{rcl}
	y(t) & \leq & z_{0}\int_{[0,t]}\lambda e^{\lambda(t-\tau)}\,\mu(d\tau)\\
	& = & z_{0}\left(e^{\lambda t}-1\right)\\
	& = & z_{0}e^{\lambda t}-z_{0}.
	\end{array}
	\end{equation}
	Returning to the definition of $y(t)$ we have 
	\[
	\begin{array}{rcl}
	z(t) & \leq & z_{0}+y(t)\\
	& \leq & z_{0}e^{\mbox{\ensuremath{\lambda}t }},
	\end{array}
	\]
	which is Gronwall's inequality.	
\end{proof}
Gronwall's inequality is best known for its use a a Lemma in proving the continuity of solutions to \eqref{eq:ode} with respect to the initial condition $x_0$. 
Continuity with respect to the initial conditions adds to the value of the \eqref{eq:ode} as a modeling framework for dynamical systems. Without it, an arbitrarily small change in initial conditions could lead to drastically different outcomes. 
\begin{theorem}
	If the assumptions of Theorem \ref{thm:existence-uniqueness} are met, the the terminal state of the solution to \eqref{eq:ode} is Lipschitz continuous with respect to the initial condition $x_0$. 
\end{theorem}
\begin{proof}
	Suppose $x_1$ and $x_2$ are solutions to two instances of \eqref{eq:ode} on the time interval $[t_0,t_f]$ with different initial conditions. Then 
	\begin{equation}\Stretch
	\begin{array}{rcl}
		\Vert x_1(t_f)- x_2(t_f) \Vert_2 &=& \left\Vert x_1(t_0)-x_2(t_0) + \int_{[t_0,t_f]} f(x_1(t),t)-f(x_2(t),t) \, \mu(dt). \right\Vert_2\\
		&\leq&  \left\Vert x_1(t_0)-x_2(t_0)\right\Vert_2 + \int_{[t_0,t_f]} \left\Vert f(x_1(t),t)-f(x_2(t),t) \right\Vert_2 \, \mu(dt)\\
		&\leq&   \left\Vert x_1(t_0)-x_2(t_0)\right\Vert_2 + \int_{[t_0,t_f]} L \left\Vert x_1(t)-x_2(t) \right\Vert_2 \, \mu(dt)\\
	\end{array}
	\end{equation}
	As a direct consequence of Lemma \ref{lem:gw_ineq} we have that 
	\begin{equation}
		\Vert x_1(t_f)- x_2(t_f) \Vert_2 \leq \left\Vert x_1(t_0)-x_2(t_0) \right\Vert e^{L(t_f-t_0)}. 
	\end{equation}	
\end{proof}
\bibliographystyle{amsalpha}
\bibliography{appc}

\clearpage
\chapter{Graph Search Algorithms}\label{app:graphs} 

This appendix provides an introduction to shortest path problems on graphs and an informed best-first search, also known as the $A^*$ algorithm.
%

A \emph{graph}\index{graph} is a set $V$ equipped with a binary relation $E$ on $V\times V$. An element of $V$ is called a \emph{vertex}\index{vertex} and an element of $E$ is called an \emph{edge}\index{edge (graph)}. 
A \emph{path}\index{path (graph)} $\{p_1,...,p_n\}$ is a finite sequence in $V$ such that $(p_i,p_{i+1}) \in E$ for each sequential pair in the sequence. 
%
%
A weighted graph is a graph equipped with a function $c:E\rightarrow \mathbb{R}$ which assigns a weight or cost to each edge. 
The \emph{cost} or \emph{weight} of a path $\{p_i\}$ with length $n$ in a weighed graph is, with some abuse of notation, given by
\begin{equation}
 c(\{p_i\})=\sum_{i=1}^{n-1} c((p_i,p_{i+1}))
\end{equation} 
A \emph{shortest path}\index{shortest path (graph)} from a source $s\in V$ to a destination set $D\subset V$ is a finite path $\{p_1,...,p_n\}$ with $p_1=s$, $p_n\in D$, and which has the minimum cost among paths originating from $s$ and terminating in $D$.
\begin{lemma}\label{lem:shortest_path_exists}
	 A finite, weighted graph with non-negative edge weights either admits a shortest path from $s\in V$ to $D\subset V$ with cost $c^*$, or has no path from $s$ to $D$.  
\end{lemma}
\begin{proof}
	Suppose there is a path $\{p_i\}$ from $s$ to $D$. Then the set of paths from $s$ to $D$ is nonempty. Note that for every path that visits a vertex twice, there exists a path with equal or lesser cost that omits the cycle. Therefore, if a path from $s$ to $D$ exists, then there is a nonempty set of paths from $s$ to $D$ with no cycles. A path with no cycles is bounded in length by the number of vertices in the graph $N$, and the number of paths of length $N$ is bounded by $N^N$. Thus, there are a finite number of paths with no cycles. Since every real valued function with a finite domain admits a minimum, there is a minimum cost path within this subset which lower bounds all paths from $s$ to $D$.
\end{proof}  

The $A^*$ algorithm is generally the most effective algorithm for single-source shortest path problems.
An essential component of the $A^*$ algorithm is an \emph{admissible heuristic}\index{admissible heuristic}; a function $h:V\rightarrow \mathbb{R}$ which estimates the cost of a shortest path from each vertex to $D$, and never overestimates it.
The algorithm incrementally examines paths originating from $s$ in order of their cost plus the estimated cost to reach $D$. The algorithm uses a function $\texttt{label} : V \rightarrow \mathbb{R}\cup \{\infty\}$ which is recursively redefined at each iteration with the cost of the best known path reaching each vertex; the function initially maps all vertices to $\infty$ representing that no path from $s$ to any particular vertex is known.
A function $\texttt{parent}:V\rightarrow V\cup \{\mathtt{NULL}\}$ is redefined in each iteration to implicitly build a secondary graph $(V\cup\{\mathtt{NULL}\},\tilde{E})$ of shortest paths from $s$ to every other vertex. 
The edge set of this graph is 
\begin{equation}
\tilde{E}=\bigcup_{v\in V} \{(\texttt{parent}(v),v)\}
\end{equation}   
The function $\texttt{parent}$ initially maps all vertices to $\mathtt{NULL}$ which represents that no path to these vertices is known.

A set $Q$ of vertices, associated to paths from $s$ by the $\texttt{parent}$ function, is operated on by the algorithm with standard set operations together with an operation $\texttt{pop}$ which returns a vertex $v \in Q$ with terminal vertex satisfying  
\begin{equation}
	\texttt{pop}(Q) \in \underset{w\in Q}{\rm argmin}\{\texttt{label}(w)+h(w)\}\leq c^*
\end{equation} 
The last operation that is needed to present the algorithm is a function $\texttt{neighborhood}:V\rightarrow 2^E$ which returns the set of edges in the graph whose first component is the argument of the function. That is,
\begin{equation}
\texttt{neighborhood}(w)=\{e\in E: \, e=(w,v) \}
\end{equation}

\begin{algorithm} 
	\begin{algorithmic}[1]
		\State $Q\gets s;$
		\State $\mathtt{label}(s)\gets 0$
		\While {$\neg(Q=\emptyset)$}
		\State $v\gets {\rm pop}(Q)$
		\If{$v\in D$}
		\State \Return $(v,\texttt{parent})$ 
		\EndIf
		\State $S\gets {\rm neighbors}(v)$
		\For $w\in S$
		\If{$\mathtt{label}(v)+\mathtt{cost}(v,w)<\mathtt{label}(w)$}
		\State $\mathtt{label}(w)\gets \mathtt{label}(v)+\mathtt{cost}(v,w)$
		\State $\mathtt{parent}(w)\gets v$
		\State $Q\gets Q\cup \{ w \}$
		\EndIf
		\EndFor  
		\EndWhile
		\State \Return $\mathtt{NO\,SOLUTION} $
	\end{algorithmic} \caption{\label{alg:ucs} The ${\rm A}^*$ algorithm.} 
\end{algorithm}
In each iteration of the algorithm, an element of $Q$ with greatest merit is removed from $Q$ and assigned to $v$ (line 5). The algorithm will terminate if $v\in D$ since the path $(v_1,v_2...,v_n)$ satisfying 
\begin{equation}
v_{i+1}=\texttt{parent}(v_i),\quad v_1=s,\,v_n=v\in D,
\end{equation}
is a path from $s$ to $D$.
If $v$ is not in the destination set, then the algorithm will continue the search by exploring extensions of the subgraph implicitly defined by the $\texttt{parent}$ function (line 7 - line 12). 
This is accomplished by assigning the neighbors of $v$ to the set $S$ (line 7). 
Each of the vertices in $S$ can be reached via the path to $v$ together with the edge connecting $v$ to each vertex in $S$. 
The cost of these paths is the cost to reach $v$, accesed by $\texttt{label}(v)$, plus the cost of the edge from $v$ to the vertex in $S$. 
If this cost is less than the current label of any of the vertices in $S$, the function $\texttt{label}$ is updated with that lower cost (line 10). 
Additionally, the path to this vertex in $S$ is stored by assigning the parent of that vertex to $v$ (line 11).
The last step in each iteration is to add the vertices of $S$ whose label was updated to the set $Q$ for examination in future iterations.
The algorithm terminates when a path from $s$ to $D$ is found as discussed above, or when the set $Q$ is empty and there are no more options for finding better paths than has already been discovered. 

\begin{theorem}
	On a finite weighted graph with nonnegative edge weights, the $A^*$ algorithm terminates in finite time returning a shortest path if one exists, or returns $\mathtt{NO\, SOLUTION}$ if there is no path from $s$ to $D$.
\end{theorem}

\begin{proof}
	(Finite running time) During the loop (line 3 - line 12), a vertex is  inserted into $Q$ if a lower cost path to that vertex than was previously known was discovered. Since the number of paths with cost less than any particular value is finite, the insertion of any vertex into $Q$ can only occur a finite number of times. Together with the finite number of vertices, the total number of vertices inserted into $Q$ during execution is finite. In each iteration, one vertex is removed from $Q$. Therefore, the algorithm must eventually terminate by returning $\texttt{NO\,SOLUTION}$ or a vertex and the parent function.
	
	(Optimality of returned path) Let $\{v_1,...,v_n\}$ be a shortest path from $s$ to $D$ with cost $c(\{v_1,...,v_n\})=c^*$. 
	Suppose, as a point of contradiction, that the algorithm does not return a path with cost $c^*$. 
	Then $v_n$ is never present in $Q$ with $\texttt{label}(v_n)=c(\{v_1,...,v_n\})$; the reason being that, if it were inserted into $Q$ with this label, then 
	\begin{equation}
		c(\{v_1,...,v_n\})+h(v_n)\leq c^*.
	\end{equation}
	Then, prior to termination, $\texttt{pop}(Q)$ would return $v_n\in D$ at which point the algorithm would terminate with a shortest path---contradictory our hypothesis.
	It follows that $v_{n-1}$ never enters $Q$ with $\texttt{label}(v_{n-1})=c(\{v_1,...,v_{n-1}\})$. Similar to the previous step, we would have 
	\begin{equation}
	c(\{v_1,...,v_{n-1}\})+h(v_{n-1})\leq c^*,
	\end{equation}
	so $\texttt{pop}(Q)$ would return $v_{n-1}$ with this label before termination of the algorithm. 
	Since $v_n$ is a neighbor of $v_{n-1}$ this would result in setting $\texttt{label}(v_n)$ to $c(\{v_1,...,v_{n-1}\})+c(\{v_{n-1},v_n\})$ in line 10 which, as was previously established, does not occur. 
	Continuing these deductive steps leads to the conclusion that $\{v_1\}$ must never enter $Q$ with $\texttt{label}(v_1)=c(\{v_1\})=0$. This is a contradiction of line 2 of the algorithm.
	Thus, if a path from $s$ to $D$ exists, the algorithm must terminate, returning a vertex $v_n\in D$ with $\texttt{label}(v_n)=c^*$. A shortest path from $v_1=s$ to $v_n\in D$ can then be recovered with the $\texttt{parent}$ function.
	
	In the event that there is no path from $s$ to $D$, the algorithm will never insert a vertex $v\in D$ into $Q$ since such vertices are always associated to a path to $s$ through the $\texttt{parent}$ function. 
	Since termination of the algorithm was already established, it must terminate with $\mathtt{NO\,SOLUTION}$.
	 
\end{proof}

\bibliographystyle{amsalpha}
\bibliography{appd}

\providecommand{\bysame}{\leavevmode\hbox to3em{\hrulefill}\thinspace}
\providecommand{\MR}{\relax\ifhmode\unskip\space\fi MR }
\providecommand{\MRhref}[2]{%
  \href{http://www.ams.org/mathscinet-getitem?mr=#1}{#2}
}
\providecommand{\href}[2]{#2}
\begin{thebibliography}{Mun00}

\bibitem[Kel75]{kelley1975general}
John~L Kelley, \emph{General topology}, vol.~27, Springer Science \& Business
  Media, 1975.

\bibitem[Mun00]{munkres2000topology}
James~R Munkres, \emph{Topology}, Prentice Hall, 2000.

\bibitem[Rud64]{rudin1964principles}
Walter Rudin, \emph{Principles of mathematical analysis}, vol.~3, McGraw-Hill
  New York, 1964.

\end{thebibliography}


\providecommand{\bysame}{\leavevmode\hbox to3em{\hrulefill}\thinspace}
\providecommand{\MR}{\relax\ifhmode\unskip\space\fi MR }
\providecommand{\MRhref}[2]{%
  \href{http://www.ams.org/mathscinet-getitem?mr=#1}{#2}
}
\providecommand{\href}[2]{#2}
\begin{thebibliography}{Lus12}

\bibitem[CL55]{coddington1955theory}
Earl~A Coddington and Norman Levinson, \emph{{Theory of Ordinary Differential
  Equations}}, {McGraw-Hill Education}, 1955.

\bibitem[Fel81]{feldman1981proof}
Marcus~B Feldman, \emph{A proof of {L}usin's theorem}, The American
  Mathematical Monthly \textbf{88} (1981), no.~3, 191--192.

\bibitem[KF61]{kolmogorov1961elements}
Andrej~Nikolaevi{\v{c}} Kolmogorov and Sergej~Vasil\'evi{\v{c}} Fomin,
  \emph{Elements of the theory of functions and functional analysis. vol. 2,
  measure. the lebesgue integral. hilbert spaces}, Graylock Press, 1961.

\bibitem[Lus12]{lusin1912proprietes}
Nikolai Lusin, \emph{Sur les propri{\'e}t{\'e}s des fonctions mesurables}, CR
  Acad. Sci. Paris \textbf{154} (1912), 1688--1690.

\bibitem[Sal16]{salamon2016measure}
Dietmar Salamon, \emph{Measure and integration}, 2016.

\end{thebibliography}


\newcommand{\etalchar}[1]{$^{#1}$}
\providecommand{\bysame}{\leavevmode\hbox to3em{\hrulefill}\thinspace}
\providecommand{\MR}{\relax\ifhmode\unskip\space\fi MR }
\providecommand{\MRhref}[2]{%
  \href{http://www.ams.org/mathscinet-getitem?mr=#1}{#2}
}
\providecommand{\href}[2]{#2}
\begin{thebibliography}{PVGBM62}

\bibitem[AM78]{abraham1978foundations}
Ralph Abraham and Jerrold~E Marsden, \emph{Foundations of mechanics}, 2nd ed.,
  Addison-Wesley, 1978.

\bibitem[Bet98]{betts1998survey}
John~T Betts, \emph{Survey of numerical methods for trajectory optimization},
  Journal of guidance, control, and dynamics \textbf{21} (1998), no.~2,
  193--207.

\bibitem[BH75]{brysonapplied}
Arthur Bryson and Yu~Chi Ho, \emph{Applied optimal control}, Taylor and Francis
  Publishers, 1975.

\bibitem[BHTR06]{benson2006direct}
David~A Benson, Geoffrey~T Huntington, Tom~P Thorvaldsen, and Anil~V Rao,
  \emph{Direct trajectory optimization and costate estimation via an orthogonal
  collocation method}, Journal of guidance, control, and dynamics \textbf{29}
  (2006), no.~6, 1435--1440.

\bibitem[BKM{\etalchar{+}}96]{barraquand1996random}
J{\'e}r{\^o}me Barraquand, Lydia Kavraki, Rajeev Motwani, Jean-Claude Latombe,
  Tsai-Yen Li, and Prabhakar Raghavan, \emph{A random sampling scheme for path
  planning}, Robotics Research, Springer (1996), 249--264.

\bibitem[BLP85]{brooks1985subdivision}
Rodney~A Brooks and Tomas Lozano-Perez, \emph{A subdivision algorithm in
  configuration space for findpath with rotation}, Transactions on Systems,
  Man, and Cybernetics, IEEE (1985), no.~2, 224--233.

\bibitem[Can87]{canny1987new}
John Canny, \emph{A new algebraic method for robot motion planning and real
  geometry}, Annual Symposium on the Foundations of Computer Science, IEEE
  (1987), 39--48.

\bibitem[HA92]{hwang1992potential}
Yong~K Hwang and Narendra Ahuja, \emph{A potential field approach to path
  planning}, Transactions on Robotics and Automation, IEEE \textbf{8} (1992),
  no.~1, 23--32.

\bibitem[HKLR02]{EST_Journal}
David Hsu, Robert Kindel, Jean-Claude Latombe, and Stephen Rock,
  \emph{Randomized kinodynamic motion planning with moving obstacles}, The
  International Journal of Robotics Research, SAGE Publications \textbf{21}
  (2002), no.~3, 233--255.

\bibitem[HP87]{hargraves1987direct}
Charles~R Hargraves and Stephen~W Paris, \emph{Direct trajectory optimization
  using nonlinear programming and collocation}, Journal of guidance, control,
  and dynamics \textbf{10} (1987), no.~4, 338--342.

\bibitem[KF11]{karaman2011sampling}
Sertac Karaman and Emilio Frazzoli, \emph{{Sampling-Based Algorithms for
  Optimal Motion Planning}}, The International Journal of Robotics Research,
  SAGE Publications \textbf{30} (2011), no.~7, 846--894.

\bibitem[Kha86]{khatib1986real}
Oussama Khatib, \emph{Real-time obstacle avoidance for manipulators and mobile
  robots}, Autonomous robot vehicles, Springer (1986), 396--404.

\bibitem[KSLO96]{kavraki1996probabilistic}
Lydia~E Kavraki, Petr Svestka, Jean-Claude Latombe, and Mark~H Overmars,
  \emph{Probabilistic roadmaps for path planning in high-dimensional
  configuration spaces}, Transactions on Robotics and Automation, IEEE
  \textbf{12} (1996), no.~4, 566--580.

\bibitem[LK01]{RRT_Journal}
Steven~M LaValle and James~J Kuffner, \emph{Randomized kinodynamic planning},
  The International Journal of Robotics Research, SAGE Publications \textbf{20}
  (2001), no.~5, 378--400.

\bibitem[LLB15]{li2015sparse}
Yanbo Li, Zakary Littlefield, and Kostas~E Bekris, \emph{{Sparse Methods for
  Efficient Asymptotically Optimal Kinodynamic Planning}}, {Algorithmic
  Foundations of Robotics XI} (2015), 263--282.

\bibitem[PF16]{paden2016generalized}
Brian Paden and Emilio Frazzoli, \emph{A generalized label correcting method
  for optimal kinodynamic motion planning}, Algorithmic Foundations of Robotics
  XII (WAFR) (2016), Currently available at:
  \url{http://www.wafr.org/program.html\#detailedprogram}.

\bibitem[PPK{\etalchar{+}}12]{perez2012lqr}
Alejandro Perez, Robert Platt, George Konidaris, Leslie Kaelbling, and Tomas
  Lozano-Perez, \emph{{LQR-RRT*}: Optimal sampling-based motion planning with
  automatically derived extension heuristics}, International Conference on
  Robotics and Automation (ICRA), IEEE (2012), 2537--2542.

\bibitem[PVGBM62]{pontryagin}
L.~S. Pontryagin, R.~V.~Gambkrelidze V.~G.~Boltyanskii, and E.~F. Mishchenko,
  \emph{The mathematical theory of optimal processes}, vol.~4, Gordon and
  Breach Science Publishers, 1962, English Translation.

\bibitem[PvY{\etalchar{+}}16]{padenSurvey}
Brian Paden, Michal \v{C}\'ap, Sze~Zheng Yong, Dmitry Yershov, and Emilio
  Frazzoli, \emph{A survey of motion planning and control techniques for
  self-driving urban vehicles, ieee}, Transactions on Intelligent Vehicles
  (2016).

\bibitem[Rei79]{reif1979complexity}
John~H Reif, \emph{Complexity of the mover's problem and generalizations},
  421--427.

\bibitem[RK92]{rimon1992exact}
Elon Rimon and Daniel~E Koditschek, \emph{Exact robot navigation using
  artificial potential functions}, Transactions on robotics and automation,
  IEEE \textbf{8} (1992), no.~5, 501--518.

\bibitem[SS83]{schwartz1983piano}
Jacob~T Schwartz and Micha Sharir, \emph{On the "piano movers" problem. {II.}
  general techniques for computing topological properties of real algebraic
  manifolds}, Advances in applied Mathematics \textbf{4} (1983), no.~3,
  298--351.

\bibitem[STZ{\etalchar{+}}14]{spieser2014toward}
Kevin Spieser, Kyle Treleaven, Rick Zhang, Emilio Frazzoli, Daniel Morton, and
  Marco Pavone, \emph{Toward a systematic approach to the design and evaluation
  of automated mobility-on-demand systems: A case study in {S}ingapore}, Road
  Vehicle Automation, Springer (2014), 229--245.

\bibitem[VPO13]{world2013global}
World Health~Organization. Violence, Injury Prevention, and World~Health
  Organization, \emph{Global status report on road safety 2013: supporting a
  decade of action}, World Health Organization, 2013.

\bibitem[WvdB13]{webb2013kinodynamic}
Dustin~J Webb and Jur van~den Berg, \emph{Kinodynamic {RRT*}: Asymptotically
  optimal motion planning for robots with linear dynamics}, International
  Conference on Robotics and Automation, IEEE (2013), 5054--5061.

\end{thebibliography}


\providecommand{\bysame}{\leavevmode\hbox to3em{\hrulefill}\thinspace}
\providecommand{\MR}{\relax\ifhmode\unskip\space\fi MR }
\providecommand{\MRhref}[2]{%
  \href{http://www.ams.org/mathscinet-getitem?mr=#1}{#2}
}
\providecommand{\href}[2]{#2}
\begin{thebibliography}{HLM97}

\bibitem[HLM97]{hsu1997path}
David Hsu, J-C Latombe, and Rajeev Motwani, \emph{Path planning in expansive
  configuration spaces}, International Conference on Robotics and Automation,
  vol.~3, IEEE, 1997, pp.~2719--2726.

\bibitem[KF10]{karaman2010optimal}
Sertac Karaman and Emilio Frazzoli, \emph{Optimal kinodynamic motion planning
  using incremental sampling-based methods}, Conference on Decision and
  Control, IEEE, 2010, pp.~7681--7687.

\bibitem[KF11]{karaman2011sampling}
Sertac Karaman and Emilio Frazzoli, \emph{Sampling-based algorithms for optimal
  motion planning}, The international journal of robotics research, Sage
  Publications \textbf{30} (2011), no.~7, 846--894.

\bibitem[KKL98]{kavraki1998analysis}
Lydia~E Kavraki, Mihail~N Kolountzakis, and J-C Latombe, \emph{Analysis of
  probabilistic roadmaps for path planning}, Transactions on Robotics and
  Automation, IEEE \textbf{14} (1998), no.~1, 166--171.

\bibitem[LK01]{RRT_Journal}
Steven~M LaValle and James~J Kuffner, \emph{Randomized kinodynamic planning},
  The International Journal of Robotics Research, SAGE Publications \textbf{20}
  (2001), no.~5, 378--400.

\bibitem[LLB15]{Li2016Asymptotically-}
Yanbo Li, Zakary Littlefield, and Kostas~E Bekris, \emph{Sparse methods for
  efficient asymptotically optimal kinodynamic planning}, {Algorithmic
  Foundations of Robotics XI}, Springer, 2015, pp.~263--282.

\end{thebibliography}


\providecommand{\bysame}{\leavevmode\hbox to3em{\hrulefill}\thinspace}
\providecommand{\MR}{\relax\ifhmode\unskip\space\fi MR }
\providecommand{\MRhref}[2]{%
  \href{http://www.ams.org/mathscinet-getitem?mr=#1}{#2}
}
\providecommand{\href}[2]{#2}
\begin{thebibliography}{YL11}

\bibitem[YL11]{yershov2011sufficient}
Dmitry~S Yershov and Steven~M LaValle, \emph{Sufficient conditions for the
  existence of resolution complete planning algorithms}, Algorithmic
  Foundations of Robotics IX, Springer (2011), 303--320.

\end{thebibliography}


\providecommand{\bysame}{\leavevmode\hbox to3em{\hrulefill}\thinspace}
\providecommand{\MR}{\relax\ifhmode\unskip\space\fi MR }
\providecommand{\MRhref}[2]{%
  \href{http://www.ams.org/mathscinet-getitem?mr=#1}{#2}
}
\providecommand{\href}[2]{#2}
\begin{thebibliography}{Lus12}

\bibitem[Fel81]{feldman1981proof}
Marcus~B Feldman, \emph{A proof of {L}usin's theorem}, The American
  Mathematical Monthly \textbf{88} (1981), no.~3, 191--192.

\bibitem[Lus12]{lusin1912proprietes}
Nikolai Lusin, \emph{Sur les propri{\'e}t{\'e}s des fonctions mesurables}, CR
  Acad. Sci. Paris \textbf{154} (1912), 1688--1690.

\end{thebibliography}


\providecommand{\bysame}{\leavevmode\hbox to3em{\hrulefill}\thinspace}
\providecommand{\MR}{\relax\ifhmode\unskip\space\fi MR }
\providecommand{\MRhref}[2]{%
  \href{http://www.ams.org/mathscinet-getitem?mr=#1}{#2}
}
\providecommand{\href}[2]{#2}
\begin{thebibliography}{HNR68}

\bibitem[Bel56]{bellman1956dynamic}
Richard Bellman, \emph{Dynamic programming and {L}agrange multipliers},
  Proceedings of the National Academy of Sciences \textbf{42} (1956), no.~10,
  767--769.

\bibitem[Ber95]{bertsekas1995dynamic}
Dimitri~P Bertsekas, \emph{Dynamic programming and optimal control}, vol.~1,
  Athena Scientific Belmont, MA, 1995.

\bibitem[HNR68]{hart1968formal}
Peter~E Hart, Nils~J Nilsson, and Bertram Raphael, \emph{A formal basis for the
  heuristic determination of minimum cost paths}, Systems Science and
  Cybernetics \textbf{4} (1968), no.~2, 100--107.

\bibitem[Kar]{rrt_implementation}
Sertac Karaman, \emph{{RRT* Library}},
  \url{http://karaman.mit.edu/software.html}, Accessed: Jan. 2016.

\bibitem[LLB]{BBekris2015}
Yanbo Li, Zakary Littlefield, and Kostas~E Bekris, \emph{Sparse {RRT} package},
  \url{https://bitbucket.org/pracsys/sparse_rrt/}, Accessed: Jan. 2016.

\bibitem[PF16]{paden2016generalized}
Brian Paden and Emilio Frazzoli, \emph{A generalized label correcting method
  for optimal kinodynamic motion planning}, Algorithmic Foundations of Robotics
  XII (WAFR), 2016, Available at:
  \url{http://www.wafr.org/program.html\#detailedprogram}.

\bibitem[PF17]{paden2017}
Brian Paden and Emilio Frazzoli, \emph{Selection of input primitives for the
  generalized label correcting method}, American Control Conference (ACC),
  2017.

\bibitem[Spo95]{spong1995swing}
Mark~W Spong, \emph{The swing up control problem for the acrobot}, Control
  Systems, IEEE \textbf{15} (1995), no.~1, 49--55.

\end{thebibliography}


\newcommand{\etalchar}[1]{$^{#1}$}
\providecommand{\bysame}{\leavevmode\hbox to3em{\hrulefill}\thinspace}
\providecommand{\MR}{\relax\ifhmode\unskip\space\fi MR }
\providecommand{\MRhref}[2]{%
  \href{http://www.ams.org/mathscinet-getitem?mr=#1}{#2}
}
\providecommand{\href}[2]{#2}
\begin{thebibliography}{DTMD10}

\bibitem[AM14]{ahmadi2014dsos}
Amir~Ali Ahmadi and Anirudha Majumdar, \emph{{DSOS} and {SDSOS} optimization:
  {LP} and {SOCP}-based alternatives to sum-of-squares optimization},
  Conference on Information Sciences and Systems, IEEE (2014), 1--5.

\bibitem[CL83]{crandall1983viscosity}
Michael~G Crandall and Pierre-Louis Lions, \emph{Viscosity solutions of
  {H}amilton-{J}acobi equations}, Transactions of the American Mathematical
  Society \textbf{277} (1983), no.~1, 1--42.

\bibitem[DTMD10]{dolgov2010path}
Dmitri Dolgov, Sebastian Thrun, Michael Montemerlo, and James Diebel,
  \emph{Path planning for autonomous vehicles in unknown semi-structured
  environments}, The International Journal of Robotics Research (IJRR)
  \textbf{29} (2010), no.~5, 485--501.

\bibitem[GAO05]{garau2005path}
Bartolome Garau, Alberto Alvarez, and Gabriel Oliver, \emph{Path planning of
  autonomous underwater vehicles in current fields with complex spatial
  variability: an {A*} approach}, Proceedings of the international conference
  on robotics and automation (ICRA), IEEE, 2005, pp.~194--198.

\bibitem[GBA{\etalchar{+}}14]{garau2014path}
Bartolom{\'e} Garau, Matias Bonet, Alberto Alvarez, Sim{\'o}n Ruiz, and Ananda
  Pascual, \emph{Path planning for autonomous underwater vehicles in realistic
  oceanic current fields: Application to gliders in the western {M}editerranean
  {S}ea}, Journal of Maritime Research \textbf{6} (2014), no.~2, 5--22.

\bibitem[GSB14]{informedRRT}
J.~D. Gammell, S.~S. Srinivasa, and T.~D. Barfoot, \emph{Informed {RRT*}:
  Optimal incremental path planning focused through an admissible ellipsoidal
  heuristic}, International Conference on Intelligent Robots and Systems, IEEE
  (2014), 2997--3004.

\bibitem[GSB15]{batchInformedRRT}
Jonathan~D Gammell, Siddhartha~S Srinivasa, and Timothy~D Barfoot, \emph{Batch
  informed trees ({BIT}*): Sampling-based optimal planning via the
  heuristically guided search of implicit random geometric graphs},
  International Conference on Robotics and Automation, IEEE (2015), 3067--3074.

\bibitem[HNR68]{hart1968formal}
Peter~E Hart, Nils~J Nilsson, and Bertram Raphael, \emph{A formal basis for the
  heuristic determination of minimum cost paths}, Transactions on Systems
  Science and Cybernetics, IEEE \textbf{4} (1968), no.~2, 100--107.

\bibitem[HZ15]{hauser2015asymptotically}
Kris Hauser and Yilun Zhou, \emph{Asymptotically optimal planning by feasible
  kinodynamic planning in state-cost space}, arXiv e-print (2015), Available
  at: \url{http://arxiv.org/abs/1505.04098}.

\bibitem[Lof04]{yalmip}
Johan Lofberg, \emph{{YALMIP}: A toolbox for modeling and optimization in
  {MATLAB}}, International Symposium on Computer Aided Control Systems Design,
  IEEE, 2004, pp.~284--289.

\bibitem[Par00]{parrilo2004sum}
Pablo~A Parrilo, \emph{Structured semidefinite programs and semialgebraic
  geometry methods in robustness and optimization}, Ph.D. thesis, California
  Institute of Technology, 2000.

\bibitem[PF16]{paden2016generalized}
Brian Paden and Emilio Frazzoli, \emph{A generalized label correcting method
  for optimal kinodynamic motion planning}, Algorithmic Foundations of Robotics
  XII (2016), Available at:
  \url{http://www.wafr.org/program.html\#detailedprogram}.

\bibitem[PPP02]{sostools}
Stephen Prajna, Antonis Papachristodoulou, and Pablo~A Parrilo,
  \emph{Introducing {SOSTOOLS}: A general purpose sum-of-squares programming
  solver}, Conference on Decision and Control (CDC), vol.~1, IEEE, 2002,
  pp.~741--746.

\bibitem[PVF16]{SOS_heuristics}
Brian Paden, Valerio Varricchio, and Emilio Frazzoli, \emph{Sum-of-squares
  heuristic synthesis for kinodynamic motion planning}, Available at:
  \url{https://github.com/bapaden/Sum_of_Squares_Admissible_Heuristics/releases}.

\bibitem[PVF17]{paden2017verification}
Brian Paden, Valerio Varricchio, and Emilio Frazzoli, \emph{Verification and
  synthesis of admissible heuristics for kinodynamic motion planning}, IEEE
  Robotics and Automation Letters \textbf{2} (2017), no.~2, 648--655.

\bibitem[Stu99]{sedumi}
Jos~F Sturm, \emph{Using {SeDuMi} 1.02, a {MATLAB} toolbox for optimization
  over symmetric cones}, Optimization methods and software \textbf{11} (1999),
  no.~1-4, 625--653.

\bibitem[TTT99]{sdpt3}
Kim-Chuan Toh, Michael~J Todd, and Reha~H Tutuncu, \emph{{SDPT3}-a {MATLAB}
  software package for semidefinite programming, version 1.3}, Optimization
  methods and software \textbf{11} (1999), no.~1-4, 545--581.

\end{thebibliography}

\clearpage
\newpage

\chapter{Extended Derivations and Proofs}\label{app:pfs}
This appendix contains the tedious and less interesting proofs that are necessary for the completeness of this thesis.
%

\section{Proof that $d_\U$ is a Metric on \U }\label{app:du_pf}
Recall the the proposed metric $d_{\U}$ for the input signal space \U in Chapter \ref{chap:topo},
\begin{equation}
d_{\mathcal{U}}(u_{1},u_{2})\coloneqq\int_{[0,\min(\tau(u_{1}),\tau(u_{2})]}\Vert u_{1}(t)-u_{2}(t)\Vert_{2} \, \mu(dt)+u_{max}|\tau(u_{1})-\tau(u_{2})|.
\end{equation}
The following derivation shows that $d_\U$ satisfies axioms MS1-MS4 for a metric:
Let $u_{1}$, $u_{2}$, and $u_{3}$ denote elements of $\mathcal{U}$. 
(MS1) By the nonnegativity of the norms on $\mathbb{R}$ and $\mathbb{R}^n$, the integrand $\Vert u_1(t)-u_2(t)\Vert_2$  and $|\tau(u_1)-\tau(u_2)|$ are nonnegative for all times $t$. Thus,
\begin{equation}
\begin{array}{rcl}
d_{\mathcal{U}}(u_{1},u_{2}) & = & \int_{[0,\min\{\tau(u_{1}),\tau(u_{2})\}]}\left\Vert u_{1}(t)-u_{2}(t)\right\Vert _{2}\mu(dt)+u_{max}|\tau(u_{1})-\tau(u_{2})|\\
& \leq & \int_{[0,\min\{\tau(u_{1}),\tau(u_{2})\}]}(0)\,\mu(dt)+u_{max}(0)\\
& = & 0,
\end{array}
\end{equation}
so that MS1 is satisfied. 
(MS2) Consider two input signals $u_1$, $u_2$ which are not equal.
Then at least one of the following is true: (i) their time domains differ so that $|\tau(u_1)-\tau(u_2)|>0$, (ii) there is a subset of $S\subset[0,\min\{\tau(u_1),\tau(u_2) \}]$ on which $u_1(t)\neq u_2(t)$ in which case $\Vert u_1(t)-u_2(t) \Vert_2>0$.
For the latter case, recall that the $L_1$-spaces are quotient spaces whose elements are the equivalence classes of signals that differ only on sets of measure zero.
Thus, $S$ has positive measure.
Then,
\begin{equation}
\begin{array}{rcl}
d_{\mathcal{U}}(u_{1},u_{2}) & = & \int_{[0,\min\{\tau(u_{1}),\tau(u_{2})\}]}\left\Vert u_{1}(t)-u_{2}(t)\right\Vert _{2}\,\mu(dt)+u_{max}|\tau(u_{1})-\tau(u_{2})|\\
& = &  \int_{S}\left\Vert u_{1}(t)-u_{2}(t)\right\Vert _{2}\mu(dt)+u_{max}|\tau(u_{1})-\tau(u_{2})|\\
& > & 0.
\end{array}
\end{equation}
In the case that $u_1=u_2$, we have $u_1(t)=u_2(t)$ for every $t$ in $[0,\tau(u_1)]$. 
Additionally, $\tau(u_1)=\tau(u_2)$ so the distance between $u_1$ and $u_2$ is zero when they are equal.
Thus, MS2 is satisfied.
(MS3) Next, symmetry of the metric follows from the symmetry of the Euclidean metric from which it is constructed,
\begin{equation}
\begin{array}{rcl}
d_{\mathcal{U}}(u_{1},u_{2}) & = & \int_{[0,\min\{\tau(u_{1}),\tau(u_{2})\}]}\left\Vert u_{1}(t)-u_{2}(t)\right\Vert _{2}\mu(dt)+u_{max}|\tau(u_{1})-\tau(u_{2})|\\
& = & \int_{[0,\min\{\tau(u_{2}),\tau(u_{1})\}]}\left\Vert u_{2}(t)-u_{1}(t)\right\Vert _{2}\mu(dt)+u_{max}|\tau(u_{2})-\tau(u_{1})|\\
& = & d_{\mathcal{U}}(u_{2},u_{1}),
\end{array}
\end{equation}
so that MS3 is satisfied.
(MS4) To simplify the notation of the proof, let $\tau(u_{1})\leq\tau(u_{2})\leq\tau(u_{3})$.
To show that triangle inequality holds together with this assumption, for any three points in $\mathcal{U}$ it will need to be shown for all permutations of $u_{1}$, $u_{2}$, $u_{3}$.
From the triangle inequality for the Euclidean metric we have,
\begin{equation}
\begin{array}{rcl}
d_{\mathcal{U}}(u_{1},u_{3}) & = & \int_{[0,\tau(u_{1})]}\left\Vert u_{1}(t)-u_{3}(t)\right\Vert _{2}\, \mu(dt)+u_{max}\left|\tau(u_{1})-\tau(u_{3})\right|\\
& \leq & \int_{[0,\tau(u_{1})]} \Vert u_1(t)-u_2(t) \Vert_2 + \Vert u_2(t)-u_3(t) \Vert_2 \,\mu(dt) \\
&  & +u_{max}|\tau_{2}-\tau_{1}|+u_{max}|\tau_{3}-\tau_{2}|.\\
& = & \int_{[0,\tau(u_{1})]}\left\Vert u_1(t)-u_{2}(t)\right\Vert _{2}\,\mu(dt)+u_{max}|\tau(u_{2})-\tau(u_{1})|\\
&  & +\int_{[0,\tau(u_{1})]}\left\Vert u_{2}(t)-u_{3}(t)\right\Vert _{2}\,\mu(dt) +u_{max}|\tau(u_{3})-\tau(u_{2})|.\\
& \leq & \int_{[0,\tau(u_{1})]}\left\Vert u_1(t)-u_{2}(t)\right\Vert _{2}\,\mu(dt)+u_{max}|\tau(u_{2})-\tau(u_{1})|\\
&  & +\int_{[0,\tau(u_{2})]}\left\Vert u_2(t)-u_3(t)\right\Vert _{2}\,\mu(dt) +u_{max}|\tau(u_{3})-\tau(u_{2})|.\\
& = & d_{\mathcal{U}}(u_{1},u_{2})+d_{\mathcal{U}}(u_{2},u_{3})
\end{array}
\end{equation}
Note that in the second step above, the integration in the second term
is carried out over a larger domain to arrive at the desired result.
The next permutation is nearly identical to the first, 
\begin{equation}
\begin{array}{rcl}
d_{\mathcal{U}}(u_1,u_2) & = & \int_{[0,\tau(u_1)]}\left\Vert u_1(t)-u_2(t)\right\Vert _{2}\mu(dt)+u_{max}\left|\tau(u_1)-\tau(u_2)\right|\\
& \leq & \int_{[0,\tau(u_1)])}\left\Vert u_1(t)-u_3(t)\right\Vert _2 + \left\Vert u_2(t)-u_3(t)\right\Vert _{2}\, \mu(dt)\\
&  & + u_{max}|\tau(u_1)-\tau(u_3)|+u_{max}|\tau(u_2)-\tau(u_3)|.\\
& = & \int_{[0,\tau(u_{1})]}\left\Vert u_1(t)-u_{3}(t)\right\Vert _{2}\,\mu(dt)+u_{max}|\tau(u_{1})-\tau(u_{3})|\\
&  & +\int_{[0,\tau(u_{1})]}\left\Vert u_{2}(t)-u_{3}(t)\right\Vert _{2}\,\mu(dt) +u_{max}|\tau(u_{2})-\tau(u_{3})|.\\
& \leq & \int_{[0,\tau(u_{1})]}\left\Vert u_1(t)-u_{3}(t)\right\Vert _{2}\,\mu(dt)+u_{max}|\tau(u_{1})-\tau(u_{3})|\\
&  & +\int_{[0,\tau(u_{2})]}\left\Vert u_{2}(t)-u_{3}(t)\right\Vert _{2}\,\mu(dt) +u_{max}|\tau(u_{2})-\tau(u_{3})|.\\
& = & d_{\mathcal{U}}(u_{1},u_{3})+d_{\mathcal{U}}(u_{2},u_{3})
\end{array}
\end{equation}
Now the third permutation,
\begin{equation}
\begin{array}{rcl}
d_{\mathcal{U}}(u_{2},u_{3}) & = & \int_{[0,\tau(u_{2})]}\left\Vert u_{2}(t)-u_{3}(t)\right\Vert _{2}\mu(dt)+u_{max}\left|\tau(u_{2})-\tau(u_{3})\right|\\
& \leq & \int_{[0,\tau(u_{1})]}\left\Vert u_{2}(t)-u_{3}(t)\right\Vert _{2}\mu(dt)+u_{max}\left|\tau(u_{2})-\tau(u_{3})\right|\\
& \leq & \int_{[0,\tau(u_{1})]}\left\Vert u_{1}(t)-u_{2}(t)\right\Vert _{2}\mu(dt)+u_{max}\left|\tau(u_{1})-\tau(u_{2})\right|\\
&  & +\int_{[0,\tau(u_{1})]}\left\Vert u_{1}(t)-u_{3}(t)\right\Vert _{2}\mu(dt)+u_{max}\left|\tau(u_{1})-\tau(u_{3})\right|\\
& = & d_{\mathcal{U}}(u_{1},u_{2})+d_{\mathcal{U}}(u_{1},u_{3}).
\end{array}
\end{equation}
The remaining three cases follow trivially from the symmetry property already established.
Thus, the function $d_\U$ is a metric on the set of input signals \U.

\section{Proof that ${d_\mathcal{X}}$ is a metric on $\mathcal{X}$}\label{app:dx_pf}

Recall the the proposed metric $d_{\X}$ for the input signal space \X in Chapter \ref{chap:topo},
\begin{equation}
d_{\mathcal{X}}(x_{1},x_{2})\coloneqq\max_{t\in\left[0,\min\{\tau(x_{1}),\tau(x_{2})\}\right]}\left\{ \left\Vert x_{1}(t)-x_{2}(t)\right\Vert_2 \right\} +M|\tau(x_{1})-\tau(x_{2})|.
\end{equation}
The following derivation shows that $d_\X$ satisfies axioms MS1-MS4 for a metric:
Let $x_{1}$, $x_{2}$, and $x_{3}$ be elements
of $\mathcal{X}$. 

(MS1) For each $t\in\left[0,\min\{\tau(x_{1}),\tau(x_{2})\}\right]$ the term
$+M|\tau(x_{1})-\tau(x_{2})|$ is nonnegative as is $\left\Vert x_{1}(t)-x_{2}(t)\right\Vert _{2}$.
The maximum over a set of nonnegative numbers is nonnegative so $d_\X$ is nonnegative. 

(MS2) Suppose $d_{\mathcal{X}}(x_{1},x_{2})=0$. 
Then 
\begin{equation}
\max_{t\in\left[0,\min\{\tau(x_{1}),\tau(x_{2})\}\right]}\left\{ \left\Vert x_{1}(t)-x_{2}(t)\right\Vert _{2}\right\} +M|\tau(x_{1})-\tau(x_{2})|=0.
\end{equation}
Since the expression is a sum of two nonnegative terms adding to
zero, each term is necessarily zero. 
Therefore, $\tau(x_{1})=\tau(x_{2})$, and
$x_{1}(t)=x_{2}(t)$ for all $t\in\left[0,\tau(x_{1})\right]$ and $t\in\left[0,\tau(x_{2})\right]$.
Equivalently, $x_{1}=x_{2}$.
Now suppose $x_1=x_2$.
Then $\tau(x_1)=\tau(x_1)$ and $x_1(t)=x_2(t)$ for every $t$ in $[0,\tau(x_1)$ and $[0,\tau(x_2)]$.
This implies $|\tau(x_1)-\tau(x_1)|=0$ and $\Vert x_1(t)-x_2(t) \Vert_2=0$.
Therefore, $d_\X(x_1,x_2)=0$.

(MS3) Symmetry of the distance function is inherited from the symmetry of the metrics from which it is constructed, 
\begin{equation}
\begin{array}{rcl}
d_{\mathcal{X}}(x_{1},x_{2}) & = & \underset{{t\in\left[0,\min\{\tau(x_{1}),\tau(x_{2})\}\right]}}{\max}\left\{ \left\Vert x_{1}(t)-x_{2}(t)\right\Vert \right\} +M|\tau(x_{1})-\tau(x_{2})|\\
& = & \underset{{t\in\left[0,\min\{\tau(x_{2}),\tau(x_{1})\}\right]}}{\max}\left\{ \left\Vert x_{2}(t)-x_{1}(t)\right\Vert \right\} +M|\tau(x_{2})-\tau(x_{1})|\\
& = & d_{\mathcal{X}}(x_{2},x_{1})
\end{array}
\end{equation}

(MS4) To simplify the notation, assume $\tau(x_{1})\leq\tau(x_{2})\leq\tau(x_{3})$.
To show that triangle inequality holds together with this assumption, the inequality it will have to be shown for all permutations of $x_{1}$, $x_{2}$, $x_{3}$: 
First, 
\begin{equation}
\begin{array}{rcl}
d_{\mathcal{X}}(x_{1},x_{3}) & = & \max_{t\in\left[0,\tau_{1}\right]}\left\{ \left\Vert x_{1}(t)-x_{3}(t)\right\Vert \right\} +M|\tau_{3}-\tau_{1}|\\
& \leq & \max_{t\in\left[0,\tau_{1}\right]}\left\{ \left\Vert x_{1}(t)-x_{2}(t)\right\Vert \right\} +\max_{t\in\left[0,\tau_{1}\right]}\left\{ \left\Vert x_{2}(t)-x_{3}(t)\right\Vert \right\} \\
&  & +M|\tau_{2}-\tau_{1}|+M|\tau_{3}-\tau_{2}|.\\
& \leq & \max_{t\in\left[0,\tau_{1}\right]}\left\{ \left\Vert x_{1}(t)-x_{2}(t)\right\Vert \right\} +\max_{t\in\left[0,\tau_{2}\right]}\left\{ \left\Vert x_{2}(t)-x_{3}(t)\right\Vert \right\} \\
&  & +M|\tau_{2}-\tau_{1}|+M|\tau_{3}-\tau_{2}|.\\
& = & d_{\mathcal{X}}(x_{1},x_{2})+d_{\mathcal{X}}(x_{2},x_{3}).
\end{array}
\end{equation}
Next,
\begin{equation}
\begin{array}{rcl}
d_{\mathcal{X}}(x_{1},x_{2}) & = & \max_{t\in\left[0,\tau_{1}\right]}\left\{ \left\Vert x_{1}(t)-x_{2}(t)\right\Vert \right\} +M|\tau_{3}-\tau_{1}|\\
& \leq & \max_{t\in\left[0,\tau_{1}\right]}\left\{ \left\Vert x_{1}(t)-x_{3}(t)\right\Vert \right\} +\max_{t\in\left[0,\tau_{1}\right]}\left\{ \left\Vert x_{2}(t)-x_{3}(t)\right\Vert \right\} \\
&  & +M|\tau_{1}-\tau_{3}|+M|\tau_{2}-\tau_{3}|.\\
& \leq & \max_{t\in\left[0,\tau_{1}\right]}\left\{ \left\Vert x_{1}(t)-x_{2}(t)\right\Vert \right\} +\max_{t\in\left[0,\tau_{2}\right]}\left\{ \left\Vert x_{2}(t)-x_{3}(t)\right\Vert \right\} \\
&  & +M|\tau_{1}-\tau_{3}|+M|\tau_{2}-\tau_{3}|.\\
& = & d_{\mathcal{X}}(x_{1},x_{3})+d_{\mathcal{X}}(x_{2},x_{3})
\end{array}
\end{equation}
Then,
\begin{equation}
\begin{array}{rcl}
d_{\mathcal{X}}(x_{2},x_{3}) & = & \max_{t\in\left[0,\tau_{2}\right]}\left\{ \left\Vert x_{2}(t)-x_{3}(t)\right\Vert \right\} +M|\tau_{3}-\tau_{2}|\\
& \leq & \max_{t\in\left[0,\tau_{1}\right]}\left\{ \left\Vert x_{1}(t)-x_{2}(t)\right\Vert \right\} +\max_{t\in\left[0,\tau_{1}\right]}\left\{ \left\Vert x_{1}(t)-x_{3}(t)\right\Vert \right\} \\
&  & +M|\tau_{1}-\tau_{2}|+M|\tau_{1}-\tau_{3}|\\
& = & d_{\mathcal{X}}(x_{1},x_{2})+d_{\mathcal{X}}(x_{1},x_{3}).
\end{array}
\end{equation}
The remaining three cases follow trivially from the symmetry property already
established.

\section{Derivation of Inequality \eqref{eq:inequality}}\label{app:inequality_derivation}
This section derives the inequality
\begin{equation}
\sum_{i=1}^{h(R)}\frac{\sqrt{n}}{\eta(R)}e^{\frac{L_{f}(h(R)-i)}{R}} \leq\frac{R\sqrt{n}}{L_{f}\eta(R)}\left(e^{L_{f}h(R)/R}-1\right).
\end{equation}
First, we factor out the term $\sqrt{n}/\eta(R)$,
\begin{equation}
\sum_{i=1}^{h(R)}\frac{\sqrt{n}}{\eta(R)}e^{\frac{L_{f}(h(R)-i)}{R}}=\frac{\sqrt{n}}{\eta(R)}\sum_{i=1}^{h(R)}e^{\frac{L_{f}(h(R)-i)}{R}},
\end{equation}
followed by reversing the order of the summation,
\begin{equation}\label{eq:some_step}
\frac{\sqrt{n}}{\eta(R)}\sum_{i=1}^{h(R)}e^{\frac{L_{f}(h(R)-i)}{R}}=\frac{\sqrt{n}}{\eta(R)}\sum_{i=0}^{h(R)-1}e^{\frac{L_{f}i}{R}}.
\end{equation}
Note that for a general nondecreasing function $f:\mathbb{R}\rightarrow \mathbb{R}$ we have the inequality
\begin{equation}
	\sum_{i\in \{j,j+1,...,k\}} f(i) \leq \int_{j}^{k}f(\rho) d\rho,
\end{equation}
since the sum is the left-Riemann sum if the integral. 
This can be applied to \eqref{eq:some_step} to obtain
\begin{equation}
\begin{array}{rcl}
\frac{\sqrt{n}}{\eta(R)}\sum_{i=0}^{h(R)-1}e^{\frac{L_{f}i}{R}} & \leq & \frac{\sqrt{n}}{\eta(R)}\int_{0}^{h(R)}e^{\frac{L_{f}}{R}\rho}d\rho\\
& = & \frac{R\sqrt{n}}{L_{f}\eta(R)}\left(e^{\frac{L_{f}h(R)}{R}}-1\right),
\end{array}
\end{equation}
which is the desired inequality.

\clearpage

\clearpage
\addcontentsline{toc}{chapter}{Index}
\printindex

\end{document}